\documentclass{article}
\usepackage[numbers]{natbib}

\usepackage{authblk} 
\usepackage{fullpage}

\usepackage[utf8]{inputenc} 
\usepackage[T1]{fontenc}    
\usepackage{hyperref}       
\usepackage{url}            
\usepackage{booktabs}       
\usepackage{bm}       
\usepackage{nicefrac}       
\usepackage{microtype}      
\usepackage{amsthm}
\usepackage{float}
\usepackage[english]{babel}
\usepackage{mathtools}
\usepackage{verbatim}
\usepackage{hyperref}
\usepackage{tabu}
\usepackage{amsmath}
\usepackage{mathrsfs}
\usepackage{enumitem}
\usepackage{amsfonts}
\usepackage{graphicx}
\usepackage{comment}
\usepackage{epstopdf}
\usepackage{parskip}
\usepackage{color}
\usepackage{algorithmic}
\usepackage[linesnumbered,ruled]{algorithm2e}%
\usepackage{amssymb}
\newtheorem{theorem}{Theorem}
\newtheorem{lemma}{Lemma}

\newtheorem{proposition}[lemma]{Proposition}
\newtheorem{fact}{Fact}

\renewenvironment{proof}[1][Proof]{\begin{trivlist}
\item[\hskip \labelsep {\bfseries #1}]}{\qed\end{trivlist}}

\newcommand*\del[0]{\partial}
\newcommand*\R[0]{\mathbb{R}}

\renewcommand*\div[0]{\nabla \cdot}
\newcommand*\ddt[0]{\frac{d}{d t}}

\newcommand*\KL[2]{\mathrm{KL}\left(#1\|#2\right)}
\newcommand*\lin[1]{\left\langle #1\right\rangle}
\newcommand*\E[1]{\mathbb{E}\left[#1\right]}
\newcommand*\Ep[2]{\mathbb{E}_{#1}\left[#2\right]}

\newcommand*\F[0]{\mathcal{F}}

\newcommand*\G[0]{\mathcal{G}}

\newcommand*\lrb[1]{\left[#1\right]}
\newcommand*\lrw[1]{\left\langle#1\right\rangle}

\newcommand*\lrp[1]{\left(#1\right)}
\newcommand*\lrn[1]{\left\|#1\right\|}

\renewcommand*{\qed}{\hfill\ensuremath{\blacksquare}}
\renewcommand*{\Re}{\mathbb{R}}

\newcommand*{\p}{\mathbf{p}}
\newcommand*{\q}{\mathbf{q}}
\newcommand*{\pmu}{\boldsymbol{\mu}}
\newcommand*{\pnu}{\boldsymbol{\nu}}
\newcommand*{\step}{(\tau-kh)}

\def\rT{\mathrm{T}}
\def\rd{\mathrm{d}}
\def\mI{\mathrm{I}}
\def\ball{\mathbb{B}}
\def\ppi{\boldsymbol{\pi}}
\def\pgamma{\boldsymbol{\gamma}}

\def\DiscError{\Ep{x_{kh}\sim\p(x_{kh})}{\big(\nabla U(\theta_\tau) - \nabla U(\theta_{kh})\big) \p(x_\tau | x_{kh})}}

\title{Is There an Analog of Nesterov Acceleration for MCMC?}

\author[a]{Yi-An Ma\thanks{yianma@berkeley.edu}}
\author[b]{Niladri S. Chatterji\thanks{chatterji@berkeley.edu}}
\author[a]{Xiang Cheng\thanks{x.cheng@berkeley.edu}}
\author[a]{Nicolas Flammarion\thanks{flammarion@berkeley.edu}}
\author[a, c]{Peter L. Bartlett\thanks{peter@berkeley.edu}}
\author[a, c]{Michael I. Jordan\thanks{jordan@cs.berkeley.edu}}

\affil[a]{Department of Electrical Engineering and Computer Sciences}
\affil[b]{Department of Physics}
\affil[c]{Department of Statistics, University of California, Berkeley, CA 94720}

\begin{document}
\maketitle

\begin{abstract}
We formulate gradient-based Markov chain Monte Carlo (MCMC) sampling as optimization on the space of probability measures, with Kullback-Leibler (KL) divergence as the objective functional.  We show that an underdamped form of the Langevin algorithm performs accelerated gradient descent in this metric. To characterize the convergence of the algorithm, we construct a Lyapunov functional and exploit hypocoercivity of the underdamped Langevin algorithm.  As an application, we show that accelerated rates can be obtained for a class of nonconvex functions with the Langevin algorithm.
\end{abstract}

\section{Introduction}
While optimization methodology has provided much of the underlying algorithmic machinery that has driven the theory and practice of machine learning in recent years, sampling-based methodology, in particular Markov chain Monte Carlo (MCMC), remains of critical importance, given its role in linking algorithms to statistical inference and, in particular, its ability to provide notions of confidence that are lacking in optimization-based methodology. However, the classical theory of MCMC is largely asymptotic and the theory has not developed as rapidly in recent years as the theory of optimization.

Recently, however, a literature has emerged that derives nonasymptotic rates for MCMC algorithms~\citep[see, e.g.,][]{Dalalyan_JRSSB,Moulines_ULA, Dalalyan_user_friendly,Xiang_underdamped,Xiang_overdamped,dwivedi2018log, Mangoubi1,Mangoubi2,Eberle_HMC,Variance_Reduction_theory}.  This work has explicitly aimed at making use of ideas from optimization; in particular, whereas the classical literature on MCMC focused on reversible Markov chains, the recent literature has focused on non-reversible stochastic processes that are built on gradients~\citep[see, e.g.,][]{completesample,completeframework,Bouncy_particle,ZigZag2}. In particular, the gradient-based Langevin algorithm~\citep{Langevin_origin,MALA,durmus2017} has been shown to be a form of gradient descent on the space of probabilities~\citep[see, e.g.,][]{JKO,Sampling_as_optimization}.

What has not yet emerged is an analog of acceleration.  Recall that the notion of acceleration has played a key role in gradient-based optimization methods~\citep{Nesterov_intro}. In particular, Nesterov's accelerated gradient descent (AGD) method, an instance of the general family of ``momentum methods,'' provably achieves a faster convergence rate than gradient descent (GD) in a variety of settings~\citep{Nesterov}. Moreover, it achieves the optimal convergence rate under an oracle model of optimization complexity in the convex setting~\citep{Nemirovskii}.

This motivates us to ask: Is there an analog of Nesterov acceleration for gradient-based MCMC algorithms?
And does it provably accelerate the convergence rate of these algorithms?

This paper answers these questions in the affirmative by showing that an underdamped form of the Langevin algorithm performs accelerated gradient descent.  Critically, our work is based on the use of Kullback-Leibler (KL) divergence as the metric.
We build on previous work that has studied the underdamped Langevin algorithm and has used coupling methods to establish convergence of the algorithm in the Wasserstein distance~\citep[see, e.g.,][]{Xiang_underdamped,Xiang_Nonconvex,Dalalyan_underdamped}.
Our work establishes a direct linkage between the underdamped Langevin algorithm and Nesterov acceleration by working directly in the objective functional, the KL divergence.
Combining ideas from optimization theory and diffusion processes, we construct a Lyapunov functional that couples the convergence in the momentum and the original variables.
We then prove the overall convergence rate by leveraging the hypocoercivity structure of the underdamped Langevin algorithm~\citep{Villani}.
For target distributions satisfying a log-Sobolev inequality,
we find that the underdamped Langevin algorithm accelerates the convergence rate of the classical Langevin algorithm from $d/\epsilon$ to $\sqrt{d/\epsilon}$ in terms of KL divergence (See Theorem~\ref{theorem:main} for formal statement).

\section{Preliminaries}
We start by laying out the problem setting, including our assumptions on the target distribution that we sample from, properties of the KL divergence with respect to other measure of differences between probability distributions, and the notion of gradient on the space of probabilities.

\subsection{Problem setting}
Assume that we wish to sample from a target (posterior) probability density, $\p^*(\theta)$, where $\theta\in\R^d$.
Consider the KL divergence to this target:
\begin{align*}
\KL{\p}{\p^*} = \int \p(\theta)\ln\left(\frac{\p(\theta)}{\p^*(\theta)}\right) \rd \theta.
\end{align*}
We use this KL divergence as an objective functional in an optimization-theoretic formulation of convergence to $\p^*(\theta)$.

We assume that $\p^*$ satisfies the following conditions.
\begin{enumerate}[label=\textbf{A{\arabic*}}]
\item \label{A1}
The target density $\p^*$ satisfies a log-Sobolev inequality with constant $\rho$~\citep{Gross,Villani_Talagrand}. That is, for any smooth function $g : \R^d \rightarrow \R$, we have
\begin{align*}
\int g(\theta) \ln g(\theta) \cdot \p^*(\theta) \rd \theta -\! \int g(\theta) \ \p^*(\theta) \rd \theta \cdot \ln\left(\int g(\theta) \ \p^*(\theta) \rd \theta\right)
\leq \frac{1}{2\rho} \int \frac{\left|\left| \nabla g(\theta)\right|\right|^2}{g(\theta)} \p^*(\theta) \rd \theta.
\end{align*}
\item \label{A2}
For $\p^* \propto e^{- U}$, the potential function $U$ is $L_G$-gradient Lipschitz and is $L_H$-Hessian Lipschitz; that is, for $U\in C^2(\R^d)$ and for all $\theta, \vartheta\in\R^d$:\footnote{It is worth noting that this definition of Hessian Lipschitzness with respect to the Frobenius norm is stronger than that with respect to the spectral norm. We postulate here that the requirement of a Hessian Lipschitz condition is an artifact of our particular choice of Lyapunov functional $\mathcal{L}$ and can possibly be removed in future work.}
\begin{align*}
&\lrn{\nabla U(\theta) - \nabla U(\vartheta)} \leq L_G \lrn{\theta-\vartheta}; \\
&\lrn{\nabla^2 U(\theta) - \nabla^2 U(\vartheta)}_F \leq L_H \lrn{\theta-\vartheta}.
\end{align*}

\item \label{A3}
Without loss of generality, for $\p^*(\theta) \propto e^{- U(\theta)}$, let $\nabla U(0)=0$ and $U(0)=0$ (which can be achieved by shifting the potential function $U$).
Further assume that the normalization constant for $e^{-U(\theta)}$ is bounded and scales at most exponentially with dimension $d$:
$\ln \left( \displaystyle\int \exp(-U(\theta)) \rd\theta \right) \leq C_N \cdot d + C_M$.
\end{enumerate}

As a concrete example, these assumptions are satisfied in the ``locally nonconvex'' case studied by~\citep{MCMC_nonconvex}, with
nonconvex region of radius $R$ and strong convexity $m$; see also Assumption~\ref{B2}--\ref{B3} in Appendix~\ref{assumptions_loc}.  Note that~\cite{MCMC_nonconvex} instantiates both the log-Sobolev constant $\rho$ and the normalization constants $C_N$ in terms of the smoothness and conditioning of $U$, showing that $\rho \geq \frac{m}{2} e^{-16 L_G R^2}$.  Here we additionally establish (see Fact~\ref{fact:normalization}) that $C_N \leq \frac{1}{2}\ln\frac{4\ppi}{m}$,
and $C_M \leq 32\frac{L_G^2}{m^2}L_G R^2$.

\subsection{KL divergence and relation to other metrics}
Our convergence result is expressed in terms of the KL Divergence.
In this section, we recall that $\KL{\p}{\p^*}$ upper bounds a number of other metrics of interest.
\begin{enumerate}
    \item By Pinsker's inequality, we can upper bound the total variation distance by the KL divergence:
    \[\mathrm{TV}\left(\p,\p^*\right) \leq \sqrt{2 \KL{\p}{\p^*}}.\]

    \item Since $\p^*$ satisfies the log-Sobolev inequality (\ref{A1}) with constant $\rho$ and has a Lipschitz smoothness property, by the Talagrand inequality (Theorem 1 of \cite{Villani_Talagrand}), we can upper bound the Wasserstein-$2$ distance (defined in Eq.~\eqref{eq:Wasserstein}) by the KL divergence:
    \begin{align}
        W_2(\p,\p^*) \le \sqrt{\frac{2\KL{\p}{\p^*}}{\rho}}.
        \label{eq:Talangrand_ineq}
    \end{align}
\end{enumerate}

\subsection{Gradients on the space of probabilities}
\label{sec:gradient_notion}
Given an iterative algorithm that generates a random vector
$\theta^{(k)}$ at each step $k$, we are interested in the convergence
of the law of $(\theta^{(k)}$, $\ppi^{(k)})$ to the measure $\ppi^*$
associated with the target density $\p^*$.
In this paper, we consider the space of probability measures that are absolutely continuous with respect to the Lebesgue measure (have density functions) and have finite second moments, $\mathcal{P}_2(\R^d)$.
It will become clear later in the paper (in Theorem~\ref{theorem:main}) that when the target density $\p^*$ satisfies Assumptions~\ref{A1}--\ref{A3}, the measure $\ppi^{(k)}$ belongs to $\mathcal{P}_2$, for any $k>0$.
For this reason, we can always analyze behaviors of the distributions in terms of their density functions.

In order to define a notion of ``gradient'' for accelerated gradient descent on the space of probabilities, $\mathcal{P}_2(\R^d)$, we first need to equip $\mathcal{P}_2(\R^d)$ with a metric.
To this end, we use the
Wasserstein-$2$ distance, defined in terms of couplings as
follows~\cite{Villani_optimal_transport}. For a pair of distributions
$\p$ and $\q$ on $\R^d$, a coupling $\pgamma$ is a joint measure
over the product space $\R^d \times \R^d$ that has $\p$ and $\q$
as its two marginal densities.
We let $\Gamma(\p, \q)$ denote the
space of all possible couplings of $\p$ and $\q$.  With this notation,
the Wasserstein-$2$ distance is given by
\begin{align}
  \label{eq:Wasserstein}
{W}_2^2 (\p, \q) := \frac{1}{2} \inf_{\pgamma\in\Gamma(\p,
  \q)} \int_{\mathbb{R}^d\times\mathbb{R}^d} \lrn{\theta-\vartheta}_2^2 \rd
\pgamma(\theta,\vartheta),
\end{align}
where the set of $\pgamma$ that attains the infimum above is denoted $\Gamma_\mathrm{opt}$.

On the space of $\mathcal{P}_2(\R^d)$ with Wasserstein-$2$ metric, there is also an optimal transport picture of the coupling.
Namely, for the measures $\pmu$ and $\pnu$ corresponding to the densities $\p$ and $\q$, there exists a transport map $\mathrm{t}:\R^d\rightarrow\R^d$, so that $(\mathrm{t}\times\mathrm{id})_{\#} \pnu \in \Gamma_\mathrm{opt}(\p,\q)$, where the push-forward operator $\#$ is defined as $\mathrm{t}_{\#}\pnu(\theta) = \pnu(\mathrm{t}(\theta))$.
With this notion, we can make use of the underlying $L^2$ Hilbert space to define strong subdifferentials.
Letting $\mathcal{L}:\mathcal{P}_2\rightarrow\R$ be a proper functional, define $\xi \in \partial\mathcal{L}$ as the strong subdifferential of $\mathcal{L}$
(taken at density $\p$ associated with measure $\pmu$)
if, for any transport map $\mathrm{t}$, we have:
\[
\mathcal{L} (\mathrm{t}_{\#}\pmu) - \mathcal{L}(\pmu)
\geq \int_{\R^d} \lrw{\xi(\theta),\mathrm{t}(\theta)-\theta} \rd \pmu(\theta)
+ o\left( \int_{\R^d} \lrn{\mathrm{t(\theta)-\theta}}_2 \rd \pmu(\theta) \right).
\]
See~\citep[Definition 10.1.1]{gradient_flow} for more details.
This strong subdifferential provides us the proper notion of ``gradient.''
In particular, for functionals with enough regularity, the strong subdifferential of $\mathcal{L}$ taken at $\p$ can be expressed as $\nabla_\theta \frac{\delta \mathcal{L}}{\delta \p}$, where $\frac{\delta}{\delta \p}$ is the functional derivative taken at $\p$ and $\nabla_\theta$ is the ordinary gradient operator in the space of $\theta$~\citep[Lemma 10.4.1]{gradient_flow}.

\section{Underdamped Langevin Algorithm as Accelerated Gradient Descent}

A recent trend in optimization theory involves casting the analysis of  algorithms into a continuous dynamical systems framework~\citep{Weijie,Ashia,Jingzhao,BinShi}.  This approach involves two steps: (1) a continuous-time system is specified and a convergence rate is obtained for the continuous dynamics; (2) the continuous dynamics is discretized, yielding a discrete-time algorithm, and the discretization error is analyzed, yielding an overall convergence rate. Our work follows in this vein.  We first study a continuous-time stochastic dynamical system that can be interpreted as an accelerated gradient flow with respect to the KL divergence $\KL{\p_t}{\p^*}$.  We then derive the underdamped Langevin algorithm as a discretization of the accelerated gradient flow. We show that this discretization is precisely accelerated gradient descent with respect to $\KL{\p_t}{\p^*}$.

\subsection{Gradient descent dynamics with respect to KL divergence}
We start by defining the dynamics of gradient descent via a consideration of the gradient flow associated with the KL divergence $\KL{\p_t}{\p^*}$.
We first formulate the ``vector flow'' associated with the following stochastic differential equation with Lipschitz continuous drift $b: \R^d \rightarrow \R^d$:
\begin{equation}
\rd \theta_t = b(\theta_t)\rd t + \sqrt{2} \rd B_t, \label{general_dynamics}
\end{equation}
where $B_t$ is a standard Brownian motion. The evolution of the probability density function $\p_t$ of the random variable $\theta_t$ follows the transport of probability mass along a vector flow $v_t$ in the state space:
\begin{align}
    \frac{\partial}{\partial t}\p_t(\theta) + \nabla^\rT\left(\p_t(\theta) v_t(\theta)\right) = 0,
    \label{eq:v_t_FPE}
\end{align}
where the vector flow can be calculated as: $v_t(\theta) = b(\theta) - \nabla\ln\p_t(\theta)$.
This can be compared with the following Liouville equation:
\[
\frac{\partial}{\partial t}\bar{\p}_t(\theta) + \nabla^\rT\left(\bar{\p}_t(\theta) b(\theta)\right) = 0,
\]
which describes the evolution of the probability along a deterministic vector field, $\frac{\rd}{\rd t} \bar{\theta}_t = b(\bar{\theta}_t)$.

On the other hand, we formulate the ``gradient'' of the KL divergence corresponding to the vector flow point of view.
For the objective functional $\F[\p_t]$, its time change when $\theta_t$ follows Eq.~\eqref{general_dynamics} is:
\[
\frac{\rd}{\rd t} \F[\p_t] = \Ep{\theta\sim\p_t}{\lrw{\nabla\frac{\delta \F[\p_t]}{\delta \p_t}(\theta), b(\theta) - \nabla \ln\p_t}},
\]
where 
$\nabla\frac{\delta \F[\p_t]}{\delta \p_t}(\theta)$ 
is the strong subdifferential of $\F[\p_t]$ associated with the \mbox{$2$-Wasserstein} metric (See Sec.~\ref{sec:gradient_notion}).
Therefore, we can consider the gradient-descent dynamics with respect to the functional $\F[\p_t]$ as taking the vector flow $v_t$ in Eq.~\eqref{eq:v_t_FPE} as $v_t(\theta) = -\nabla\frac{\delta \F[\p_t]}{\delta \p_t}(\theta)$.
%
%
%
%
When the functional is the KL divergence, $\F[\p_t] = \KL{\p_t}{\p^*}$, the gradient descent flow $v_t^{GD}$ involves taking
\[
v_t^{GD}(\theta) = -\nabla\frac{\delta \KL{\p_t}{\p^*}}{\delta \p_t}(\theta) = - \nabla\ln\frac{\p_t(\theta)}{\p^*(\theta)},
\]
or, equivalently, $b^{GD}(\theta) = -\nabla U(\theta)$ in Eq.~\eqref{general_dynamics}.

Along this gradient descent flow, $v_t^{GD}$, the time evolution of the KL divergence is
\[
\frac{\rd}{\rd t} \KL{\p_t}{\p^*}
= - \Ep{\theta\sim\p_t}{\lrn{\nabla\frac{\delta \KL{\p_t}{\p^*}}{\delta \p_t}(\theta)}^2}
= - \Ep{\theta\sim\p_t}{\lrn{\nabla\ln\frac{\p_t(\theta)}{\p^*(\theta)}}^2}.
\]

If $\p^*(\theta)$ satisfies Assumption~\ref{A1}
then taking $g=\frac{\p_t}{\p^*}$ in the log-Sobolev inequality yields:
\begin{align}
\label{eq:logsob}
\Ep{\theta\sim\p_t}{\ln\lrp{\frac{\p_t(\theta)}{\p^*(\theta)}}} \leq \frac{1}{2\rho} \Ep{\theta\sim\p_t}{ \left|\left| \nabla \ln\lrp{\frac{\p_t(\theta)}{\p^*(\theta)}} \right|\right|^2 }.
\end{align}
Note the resemblance of this bound to the Polyak-\L{}ojasiewicz condition~\citep{Polyak}
used in optimization theory for studying the convergence of gradient methods---in both cases the difference in objective value from the current iterate
to the optimum is upper bounded by the squared norm of the gradient of the objective.
With the log-Sobolev inequality, we obtain that
\[
\frac{\rd}{\rd t} \KL{\p_t}{\p^*}
= - \Ep{\theta\sim\p_t}{\lrn{\nabla\ln\frac{\p_t(\theta)}{\p^*(\theta)}}^2}
\leq - 2\rho \KL{\p_t}{\p^*},
\]
which implies the linear convergence of $\KL{\p_t}{\p^*}$ along the gradient descent flow.

\subsection{Accelerated gradient descent in KL divergence: A continuous perspective}
\label{sec:underdamped_cont_cvg}
We now introduce an accelerated dynamics in the space of probabilities via the incorporation of a momentum variable $r\in\R^d$.
Denote $x=(\theta, r)$ and let the joint target distribution be $\p^*(x) = \p^*(\theta) \p^*(r) = \exp\left( - U(\theta) - \frac{\xi}{2} ||r||_2^2 \right)$.\footnote{We will use $\p^*(\theta)$ and $\p_t(\theta)$ to denote marginal distributions of $\p^*(\theta, r)$ and $\p_t(\theta, r)$, respectively, after integration over $r$.}
To design the accelerated gradient descent dynamics with respect to the KL divergence, we leverage the acceleration phenomenon in optimization, which uses the gradient of the expanded objective function to guide the algorithm (see the discussion in Sec.~\ref{sec:AGD}).
We expand the KL divergence (in both the $\theta$ and $r$ coordinates) to obtain:
\begin{align*}
\KL{\p_t(\theta, r)}{\p^*(\theta) \p^*(r)} &= \int \int \p_t(\theta, r) \ln\frac{\p_t(\theta, r)}{ \p^*(\theta) \p^*(r) } \rd \theta \rd r \nonumber\\
&= \KL{\p_t(\theta)}{\p^*(\theta)}
+ \Ep{\theta\sim\p_t(\theta)}{\KL{\p_t(r|\theta)}{\p^*(r)}},
\end{align*}
and form the vector field:
\begin{align}
v_t^{AGD}(x)
&=
- \left(
\begin{array}{cc}
0 & -\mI \\
\mI & \gamma \mI
\end{array}
\right)
\left(
\begin{array}{c}
\nabla_\theta\frac{\delta \KL{\p_t}{\p^*}}{\delta \p_t} \vspace{2pt}\\
\nabla_r\frac{\delta \KL{\p_t}{\p^*}}{\delta \p_t}
\end{array}
\right) \label{eq:simple_irr_dyn}\\
&= \left(
\begin{array}{l}
\nabla_r\ln{\p_t(\theta, r)} + \xi r\\
- \nabla_\theta\ln{\p_t(\theta, r)} - \nabla U(\theta) -\gamma\nabla_r\ln\frac{\p_t(\theta, r)}{ \p^*(r) }
\end{array}
\right). \label{eq:simple_irr}
\end{align}
The corresponding continuity equation defined by this vector field is
\begin{align*}
0&= \frac{\partial}{\partial t}\p_t(\theta,r) + \nabla^\rT\left(\p_t(\theta,r) v_t^{AGD}(\theta,r)\right)
\nonumber\\
&= \frac{\partial}{\partial t}\p_t(\theta,r) + \left(\nabla_\theta^\rT, \nabla_r^\rT\right)
\left[ \p_t(\theta, r) \left(
\begin{array}{l}
\xi r\\
- \nabla U(\theta) -\gamma\nabla_r\ln\frac{\p_t(\theta, r)}{ \p^*(r) }
\end{array}\right) \right].
\end{align*}
This implies that the vector field can be implemented via the following stochastic differential equation
\begin{equation}
\left\{
\begin{array}{l}
d \theta_t = \xi r_t \rd t\\
d r_t = - \nabla U(\theta_t) \rd t -\gamma \xi r_t \rd t + \sqrt{2 \gamma} \rd B_t,
\end{array}
\right.  \label{e:underdamp_diff}
\end{equation}
which is the underdamped Langevin dynamics~\cite{Langevin}.

\subsubsection{Convergence of the accelerated gradient-descent dynamics}
If we consider the time derivative of the KL divergence, we have: $\KL{\p_t}{\p^*}$,
\begin{align}
\frac{d}{\rd t} \KL{\p_t}{\p^*}
&= \int \left< \nabla_x \frac{\delta \KL{\p_t}{\p^*}}{\delta \p_t}, v_t^{AGD}(\theta,r) \right> \p_t \ \rd x
\nonumber\\
&= \int \left< \nabla_x \frac{\delta \KL{\p_t}{\p^*}}{\delta \p_t},
-\left(
\begin{array}{cc}
0 & -\mI \\
\mI & \gamma \mI
\end{array}
\right)
\nabla_x\ln\frac{\p_t}{\p^*} \right> \p_t \ \rd x
\nonumber\\ \label{eq:dF}
&= -\gamma \Ep{\p_t}{ \lrn{\nabla_r \ln\frac{\p_t}{\p^*}}^2 }.
\end{align}
This only demonstrates the contractive property in the $r$ coordinates (note that the gradient is only in $r$ in Line~\eqref{eq:dF}) and does not directly provide a linear convergence rate over time.
To quantify the convergence rate for this accelerated gradient descent dynamics with respect to the KL divergence objective, we need to couple the convergence in $\theta$ coordinates to that in $r$.
To this end, we follow recent work in the optimization literature~\citep{Ashia} and design a Lyapunov functional which makes use of a quadratic form of the gradient of the distance $\mathcal{D}$ between the current iteration $\p_t$ and the stationary solution $\p^*$:
\begin{align}
\mathcal{L}[\p_t] &= \KL{\p_t}{\p^*} + \Ep{\p_t}{\left< \nabla_x \frac{\delta \mathcal{D}[\p_t, \p^*]}{\delta \p_t},
S \nabla_x \frac{\delta \mathcal{D}[\p_t, \p^*]}{\delta \p_t} \right>} \nonumber\\
&= \Ep{\p_t}{\ln\frac{\p_t}{\p^*} + \left< \nabla_x \ln\frac{\p_t}{\p^*}, S \nabla_x \ln\frac{\p_t}{\p^*} \right>}, \label{eq:Lyapunov_functional}
\end{align}
where we take the distance measure between $\p_t$ and $\p^*$ as the KL divergence itself: $\mathcal{D}[\p_t, \p^*] = \KL{\p_t}{\p^*} = \Ep{\p_t}{\ln\frac{\p_t}{\p^*}}$.
Here we set the positive definite matrix in the quadratic form to be
\begin{align}
S=\frac{1}{L_G} \left(
\begin{array}{cc}
1/4 \ \mI_{d \times d} & 1/2 \ \mI_{d \times d} \\
1/2 \ \mI_{d \times d} & 2 \ \mI_{d \times d}
\end{array}
\right).
\end{align}
Interestingly, similar forms appear in the analyses of both accelerated gradient descent dynamics~\citep{Nesterov,Ashia} and hypocoercive diffusion operators~\citep{Villani,Hypo_revised}.

We then make use of this Lyapunov functional to obtain a linear convergence rate for the accelerated gradient descent dynamics with respect to the KL divergence.
\begin{proposition}
\label{proposition:cont_evolution}
Under Assumptions~\ref{A1}--\ref{A3},
the time evolution of the Lyapunov functional $\mathcal{L}$ with respect to the continuous time vector flow $v_t^{AGD}$ in Eq.~\eqref{eq:simple_irr} with $\gamma=2$ and $\xi=2L_G$ is upper bounded as:
\begin{align*}
\frac{d}{\rd t} \mathcal{L}[\p_t]
\leq - \frac{\rho}{10} \mathcal{L}[\p_t].
\end{align*}
\end{proposition}
This establishes linear convergence of the continuous process with a rate of $\frac{\rho}{10}$.

\subsubsection{Accelerated gradient descent dynamics for optimization} 
\label{sec:AGD}
It is worth noting that the derivation in the previous subsection has a close correspondence to recent analyses of the accelerated gradient descent dynamics in convex optimization~\citep{Weijie,Ashia}.
Indeed, when optimizing a strongly convex function $U(\theta)$ on a Euclidean space with the accelerated gradient descent dynamics, the continuous limit of the algorithm is expressed as an ordinary differential equation~\citep{Ashia}:
\[
\frac{\rd^2 \theta_t}{\rd t^2} + \gamma \xi \frac{\rd \theta_t}{\rd t} + \xi \nabla U(\theta_t) = 0.
\]
We can expand the space of interest via introducing a ``momentum'' variable, $r_t = \frac{1}{\xi} \frac{\rd \theta_t}{\rd t}$, to obtain a vector field point of view on the joint space of $x_t=(\theta_t,r_t)$:
\begin{equation*}
\left\{
\begin{array}{l}
\frac{\rd \theta_t}{\rd t} = \xi r_t \\
\frac{\rd r_t}{\rd t} = -\nabla U(\theta_t) - \gamma \xi r_t.
\end{array}
\right.
\end{equation*}
We also extend the original objective function $U(\theta)$ to $H(x) = U(\theta) + \frac{\xi}{2} \|r\|_2^2$ to capture the overall dynamical behavior in the space of $x$.
With the definition of this extended objective function $H$, we can simplify the expression of the dynamics:
\begin{equation}
\frac{\rd x}{\rd t} = - \left(
\begin{array}{cc}
0 & -\mI \\
\mI & \gamma \mI
\end{array}
\right)
\left(
\begin{array}{c}
\nabla_\theta H(x) \\
\nabla_r H(x)
\end{array}
\right) . \label{e:agd}
\end{equation}
To quantify convergence for the strongly convex objective $U$, \citep{Ashia} considers a Lyapunov function of the form $l(x) = H(\theta) + \left< \nabla_x^T D_h(x), S \nabla_x D_h(x) \right>$, where $D_h(x) = \frac12 \lrn{\theta-\theta^*}^2 + \frac12 \lrn{r}^2$ is the squared distance from $(\theta, r)$ to the optimum of $H$, $(\theta^*, 0)$. 

Comparing the dynamics of Eq.~\eqref{e:agd} versus Eq.~\eqref{eq:simple_irr_dyn} and the convergence analyses for them, we observe that the underdamped Langevin diffusion defined in Eq.~\eqref{e:underdamp_diff} is precisely accelerated gradient descent with respect to the KL divergence.

\subsection{Underdamped Langevin via second-order discretization}
While the continuous-time perspective yields insight into the convergence rates achievable by acceleration, for these insights to apply to discrete-time algorithms it is necessary to understand the effects of discretization. In optimization, an emerging literature has begun to show how to design discretization procedures that retain accelerated rates from continuous time~\citep{Ashia,Jingzhao,BinShi}. The literature in MCMC has not yet formalized lower bounds on convergence rates that allow characterizations of acceleration, in either continuous time or discrete time, but there are results that exhibit the importance of discretization for convergence.  In particular, higher order (and more accurate) discretization schemes are found to accelerate convergence~\citep{NealHMC,thermostat,Xiang_underdamped,Dalalyan_underdamped,Mangoubi1,Mangoubi2}.

In this section we show how to design a discretization for the an underdamped Langevin algorithm that yields accelerated rates.
Following~\citep{Xiang_underdamped}, we discretize the time dimension underlying Eq.~\eqref{e:underdamp_diff} into intervals of equal length $h$ (at the end of the $k$-th iteration, we have $t = kh$). Then in the $(k+1)$-th step, we define a continuous dynamics in the interval of $\tau\in[kh,(k+1)h]$ by conditioning on the initial value of $x_{kh}$:
\begin{equation}
\left\{
\begin{array}{l}
\rd \theta_\tau = \xi r_\tau \rd \tau\\
\rd r_\tau = -\gamma \xi r_\tau \rd \tau - \nabla U(\theta_{kh}) \rd \tau + \sqrt{2 \gamma} \rd B_\tau.
\end{array}
\right. \label{eq:underdamp_diff_disc}
\end{equation}
\begin{algorithm}
\caption{Underdamped Langevin Algorithm}\label{alg:main}
\begin{algorithmic}
\STATE Let $x_0 = (\theta_0, r_0)$, where $\theta_0, r_0 \sim \mathcal{N} \left(0, \frac{1}{L_G} \mI \right)$.  
\FOR{$k=0,\cdots,K-1$} 
\STATE Sample ${x}_{(k+1)h} \sim
\mathcal{N} \left( \mu\left(x_{kh}\right), \Sigma \right)$, where $\mu\left(x_{kh}\right)$ and $\Sigma$ are defined in
Eq.~\eqref{eq:mu_k} and \eqref{eq:Sigma}.  
\ENDFOR
\end{algorithmic}
\end{algorithm}
In Appendix~\ref{Append:iteration} we derive explicit formulas for $x_\tau$ given $x_{kh}$.  These are used to generate the $(k+1)$-th iterate.  In particular, define the hyperparameters $\gamma = 2$, $\xi = 2 L_G$, and set the step size as follows:
\begin{align}
h = \frac{1}{56} \frac{1}{\sqrt{L_G}} 
\min\left\{ \frac{1}{24} \frac{\rho}{L_G}, \frac{ \sqrt{L_G} \rho }{L_H} \right\}
\cdot \min\left\{ \left(\widetilde{C_N} + 2\right)^{-1/2} \sqrt{\frac{\epsilon}{d}}, \sqrt{\frac{\epsilon}{C_M}} \right\},
\label{eq:h_def}
\end{align}
where $\widetilde{C_N} = C_N + \frac{1}{2}\ln\frac{L_G}{2\pi}$.
The discretized vector field is
\begin{align}
\hat{v}_\tau^{AGD}
= \left(
\begin{array}{l}
\xi r_\tau\\
\!- \!\nabla U(\theta_{kh})\! -\!\gamma\nabla_r\ln\frac{\p (\theta_\tau, r_\tau)}{ \p^* (r_\tau) }\!
\end{array}
\right)
= \left(
\begin{array}{l}
\xi r_\tau\\
\!-\! \nabla U(\theta_{kh}) \!-\! \gamma\xi r_\tau \!-\! \gamma\nabla_r\ln \p (\theta_\tau, r_\tau)
\end{array}
\right). \label{eq:simple_irr_discrete}
\end{align}
This leads to a high-order discretization scheme that is defined explicitly in Appendix~\ref{Append:iteration} and summarized in Algorithm~\ref{alg:main}. 

By way of comparison, the Euler-Maruyama discretization scheme corresponds to:
\begin{align}
\hat{v}_\tau^{E-M}
= \left(
\begin{array}{l}
\xi r_{kh}\\
- \nabla U(\theta_{kh}) - \gamma\xi r_{kh} - \gamma\nabla_r\ln \p (\theta_\tau, r_\tau)
\end{array}
\right). \nonumber
\end{align}
After integration, we obtain that for $\tau\in[kh,(k+1)h]$:
\begin{align}
\left\{
\begin{array}{l}
\theta_\tau = \theta_{kh} + (\tau - kh) \xi r_{kh} \\
r_\tau = \left( 1 - (\tau - kh) \gamma\xi\right) r_{kh} - (\tau - kh) \nabla U(\theta_{kh}) + \sqrt{2\gamma} B_{\tau - kh},
\end{array}
\right. \nonumber
\end{align}
where the Brownian motion is defined as $B_{\tau - kh} \sim \mathcal{N} \left(0, (\tau - kh) \mI_{d \times d} \right)$.
This low-order integration scheme does not grant accelerated convergence guarantees.

There are other higher-order discretization schemes that can be considered in addition to our scheme in Eq.~\eqref{eq:simple_irr_discrete}.  In particular, note that $v_t^{AGD}$ decomposes into two parts:
\begin{align}
v_t^{AGD}
= \left(
\begin{array}{l}
\xi r_t\\
- \nabla U(\theta_t)
\end{array}
\right)
+
\left(
\begin{array}{l}
0\\
-\gamma\nabla_r\ln\frac{\p (\theta_t, r_t)}{ \p^* (r_t) }
\end{array}
\right), \nonumber
\end{align}
where each part preserves $p^*$ as the invariant distribution.
This inspires a splitting scheme for integrating $v_t^{AGD}$.
The first part is a Hamiltonian vector flow, which can be integrated via symplectic integration schemes such as the leapfrog method. The second part can be explicitly integrated to yield $r_{\tau - kh} \sim \mathcal{N}\left( e^{-\gamma\xi(\tau - kh)} r_{kh} , \frac{1}{\xi}\left(1-e^{-2\gamma\xi(\tau - kh)}\right) \mI \right)$.

Taking $(\tau - kh)\rightarrow\infty$, $r$ is resampled as: $r \sim \mathcal{N}\left( 0 , \frac{1}{\xi} \mI \right)$ according to the stationary distribution $\p^*(r)$.
This recovers the Hamiltonian Monte Carlo (HMC) method~\citep{NealHMC}.
Relating to concepts in optimization, this ``momentum resampling'' step corresponds to a ``momentum restart'' method in optimization: one periodically restarts the momentum from the stationary point~\citep{Momentum_Restart}.
In optimization this has a theoretical justification in terms of increasing convergence rate; for HMC it has been observed empirically that \emph{not} taking $(\tau - kh)\rightarrow\infty$ at every step increases mixing~\citep{SOL_HMC}.

\begin{figure}
\centering
\includegraphics[scale=0.43]{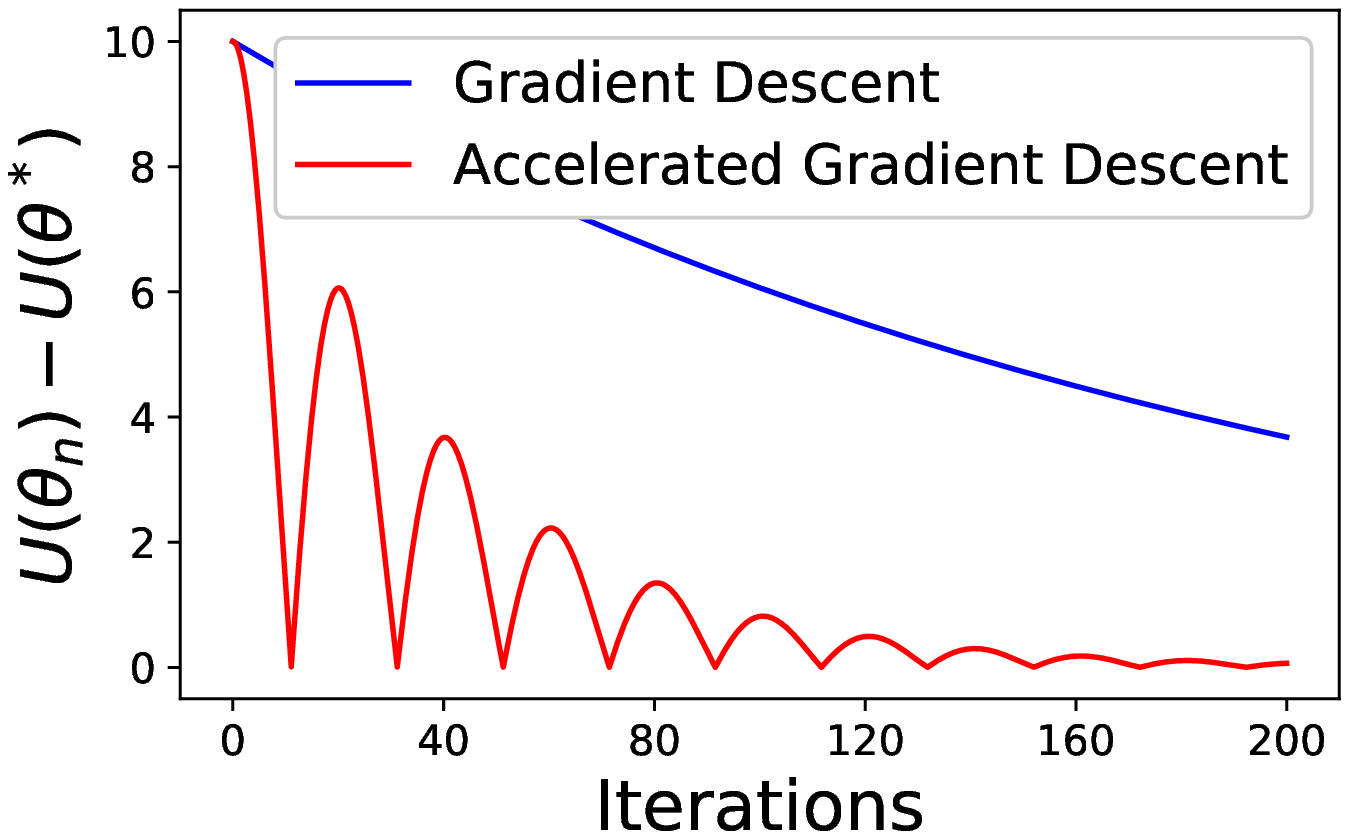}
\includegraphics[scale=0.43]{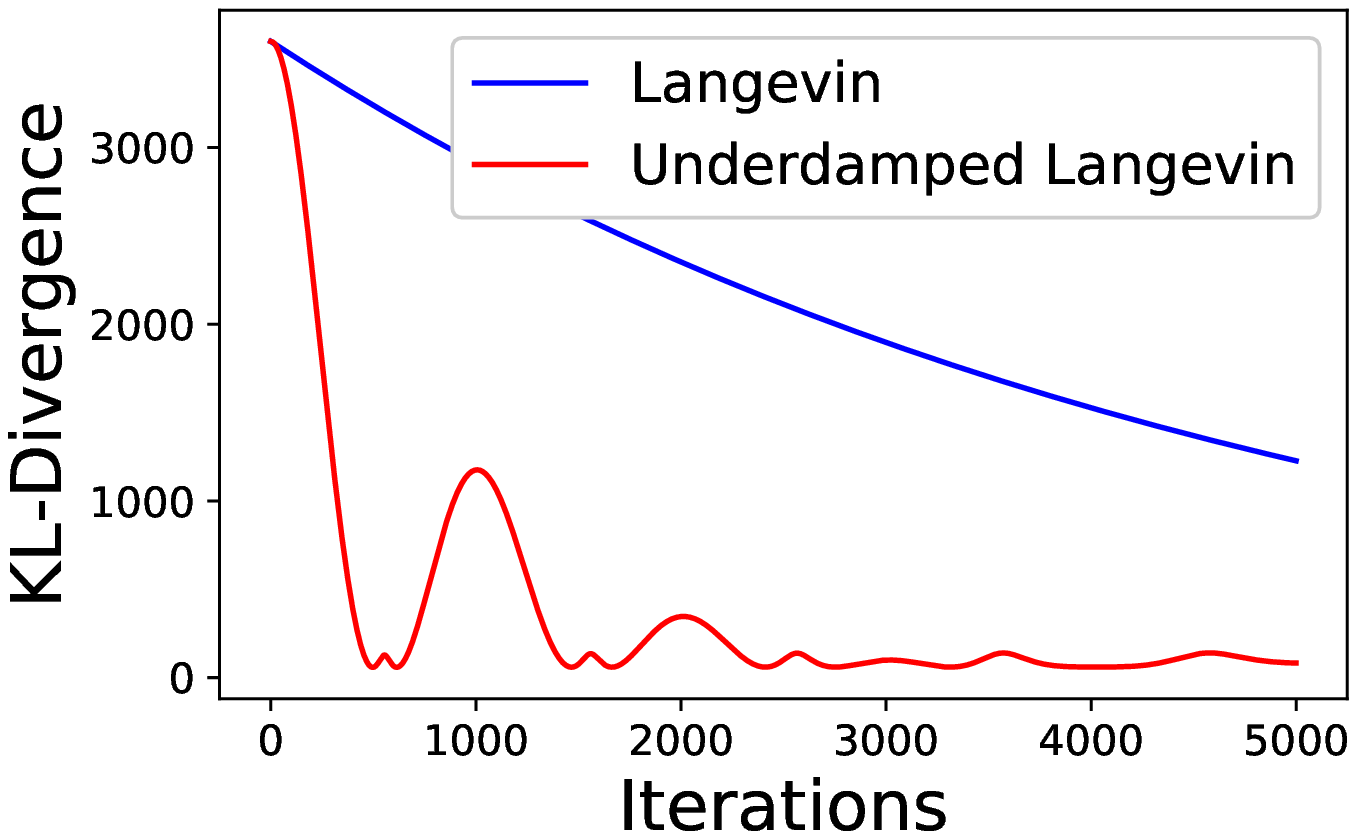}
\caption{Acceleration phenomenon in optimization and sampling. Left: The (accelerated) gradient descent algorithms minimize the objective function value $\left|U(\theta_t) - U(\theta^*)\right|$. Right: The (underdamped) Langevin algorithms minimize the KL divergence $\KL{\p_t(\theta)}{\p^*(\theta)}$, where $\p^*(\theta)\propto e^{-U(\theta)}$. In both cases, $U$ is a quadratic function in $100$ dimensions with condition number $L/m=100$. }
\label{fig:Acceleration}
\end{figure}

\section{Convergence of the Underdamped Langevin Algorithm}

\label{sec:converge_rate}
From Fig.~\ref{fig:Acceleration}, we see that the underdamped Langevin algorithm, Eq.~\eqref{eq:sim_updates}, seems to have a similar profile to accelerated gradient descent; it uses oscillatory behavior to increase the convergence rate.
In this section, we rigorously establish acceleration, by proving that the convergence of the underdamped Langevin algorithm is of order $\mathcal{O}\left(\sqrt{{d/\epsilon}}\right)$ in terms of KL divergence.

Let the KL divergence from $\p_t(\theta)$ to $\p^*(\theta)$ be the target functional to minimize:
\begin{align}
\KL{\p_t(\theta)}{\p^*(\theta)}
\leq \KL{\p_t(\theta, r)}{\p^*(\theta) \p^*(r)}. \nonumber
\end{align}
We have the following theorem.

\begin{theorem}
Assume $\p^*(\theta) \propto e^{- U(\theta)}$ satisfies Assumptions~\ref{A1}--\ref{A3}.
We use $\rho$ to denote the minimum of the log-Sobolev constant and $1$.
Then if we iterate the underdamped Langevin algorithm \eqref{eq:sim_updates} with initial condition $\theta_0\sim\mathcal{N}\left(0,\frac{1}{L_G}\mI\right)$ for
\begin{align*}
k \geq \mathcal{O}\left(\sqrt{\frac{d}{\epsilon}} \ln\left(\frac{d}{\epsilon}\right)\right)
\end{align*}
steps,
we have
$\KL{\p_{kh}(\theta)}{\p^*(\theta)}
\leq \KL{\p_{kh}(\theta, r)}{\p^*(\theta) \p^*(r)}<\epsilon$, $\forall \epsilon\leq 2 d$.

If we further assume that the function $U$ is locally nonconvex
with radius $R$ and has global strong convexity $m$ (Assumption~\ref{B2}--\ref{B3}),
we obtain an explicit dependence of the convergence time $K$ on other constants:
\begin{align*}
K = {\mathcal{O}}\left( \max\left\{ \frac{L_G^{3/2}}{\rho^2}, \frac{L_H}{\rho^2} \right\} \sqrt{\frac{d}{\epsilon}} \ln \frac{d}{\epsilon} \right),
\end{align*}
where $\rho = \min \left\{ \frac{m}{2} e^{-16 L_G R^2}, 1 \right\}$.
\label{theorem:main}
\end{theorem}

We devote the remainder of Section~\ref{sec:converge_rate} to the proof of Theorem~\ref{theorem:main}. As advertised, the proof decomposes into a continuous-time analysis and a discretization analysis. We first establish the convergence rate of the continuous underdamped Langevin dynamics in Proposition~\ref{proposition:cont_evolution} to quantify the instantaneous contraction provided by the dynamics.
We then study the discretization error of the underdamped Langevin algorithm in each step.
Combining these two results and integrating over the time steps leads us to the final conclusion.

We begin by formulating the instantaneous change of the probability density $\p(x_\tau)$ within each step of the underdamped Langevin algorithm.
The time evolution of $\p(x_\tau | x_{kh})$ following the discretized vector flow $\hat{v}_\tau^{AGD}$ for $\tau\in[{kh},(k+1)h]$ is as follows:
\begin{align}
\frac{\partial \p(x_\tau | x_{kh})}{\partial \tau}
&= - \nabla_x^\rT \big( \p(x_\tau | x_{kh}) \cdot \hat{v}_\tau^{AGD} \big)
\nonumber\\
&= - \nabla_x^\rT \big( \p(x_\tau | x_{kh}) \cdot v_\tau^{AGD} \big)
- \nabla_x^\rT \left( \p(x_\tau | x_{kh}) \cdot (\hat{v}_\tau^{AGD} - v_\tau^{AGD}) \right). \nonumber
\end{align}
Therefore, for the unconditioned probability density $\p(x_\tau) = \Ep{x_{kh}\sim\p(x_{kh})}{\p(x_\tau | x_{kh})}$,
\begin{align}
\frac{\partial \p(x_\tau)}{\partial \tau}
= - \nabla_x^\rT \big( \p(x_\tau) \cdot v_\tau^{AGD} \big)
- \Ep{x_{kh}\sim\p(x_{kh})}{\nabla_x^\rT \big( (\hat{v}_\tau^{AGD} - v_\tau^{AGD})\p(x_\tau | x_{kh}) \big)}. \label{eq:discrete_FPE}
\end{align}
We have thus separated the time evolution of $\p(x_\tau)$ into two parts: the continuous component and the discretization error component.

Recall the Lyapunov functional, $\mathcal{L}(\p_t) = \Ep{\p_t}{\ln\frac{\p_t}{\p^*} + \left< \nabla_x \ln\frac{\p_t}{\p^*}, S \nabla_x \ln\frac{\p_t}{\p^*} \right>}$, that we defined in Sec.~\ref{sec:underdamped_cont_cvg}).
We use this Lyapunov functional to analyze the convergence of the underdamped Langevin algorithm.  Note that the instantaneous change of the Lyapunov functional $\mathcal{L}$ follows the overall vector flow $\hat{v}_t^{AGD}$, and derives from the continuous vector flow $v_t^{AGD}$ and the discretization error $\hat{v}_t^{AGD} - v_t^{AGD}$:
\begin{subequations}
\begin{align}
\frac{d}{dt} \mathcal{L}[\p(x_\tau)]
&= \int \frac{\delta \mathcal{L}}{\delta \p(x_\tau)} \frac{\partial \p(x_\tau)}{\partial t} \ \rd x_\tau
= \int \left< \nabla_x\frac{\delta \mathcal{L}}{\delta \p(x_\tau)}, \hat{v}_\tau^{AGD} \right> \p(x_\tau)\ \rd x_\tau \nonumber\\
&= \int \left< \nabla_x\frac{\delta \mathcal{L}}{\delta \p(x_\tau)}, v_\tau^{AGD} \right> \p(x_\tau)\ \rd x_\tau \label{eq:Lyap_cont}\\
&+ \int \left< \nabla_x \frac{\delta \mathcal{L}}{\delta \p(x_\tau)}, \Ep{x_{kh}\sim\p(x_{kh})}{\left(\hat{v}_\tau^{AGD}-v_\tau^{AGD}\right)\p(x_\tau | x_{kh})} \right> \ \rd x_\tau. \label{eq:Lyap_disc}
\end{align}
\end{subequations}
We now analyze term~\eqref{eq:Lyap_cont} and term~\eqref{eq:Lyap_disc} separately, returning later to combine the analyses and obtain the overall convergence rate.

We use Lemma~\ref{lemma:cont_flow_dL} in the Appendix to 
expand term~\eqref{eq:Lyap_cont} and quantify the convergence of $\mathcal{L}$ with respect to the continuous vector flow $v_\tau^{AGD}$:
\begin{align}
\int \left< \nabla_x \frac{\delta \mathcal{L}}{\delta \p_t}(x), v_\tau^{AGD}(x) \right> \p_\tau(x) \ \rd x
&= -4 \Ep{\p_t}{ \left< \nabla_x \nabla_r \ln \left(\frac{\p_t}{\p^*}\right), S \nabla_x \nabla_r \ln \left(\frac{\p_t}{\p^*}\right) \right>_F } \nonumber\\
&- \Ep{\p_t}{ \left< \nabla_x \ln \left(\frac{\p_t}{\p^*}\right), M_C \nabla_x \ln \left(\frac{\p_t}{\p^*}\right) \right> }, \label{eq:cont_cvg}
\end{align}
where $M_C$ is defined in Eq.~\eqref{eq:M_C_def}.
The two terms on the right-hand side of Eq.~\eqref{eq:cont_cvg} are both less than or equal to zero.
We will use the first term to cancel similar terms in the discretization error and use the second term to drive the convergence of the process (by way of the log-Sobolev inequality).

\subsection{Discretization error}
For term~\eqref{eq:Lyap_disc} capturing the discretization error, we provide an upper bound in the following proposition.
\begin{proposition}
\label{proposition:disc_error}
Under Assumption~\ref{A2}, when $\tau-kh\leq\frac{1}{8L_G}$, $\gamma=2$, and $\xi=2L_G$, term~\eqref{eq:Lyap_disc} is upper bounded as:
\begin{align}
\lefteqn{ \int \left< \nabla_x \frac{\delta \mathcal{L}}{\delta \p(x_\tau)}, \Ep{x_{kh}\sim\p(x_{kh})}{\left(\hat{v}_\tau^{AGD}-v_\tau^{AGD}\right)\p(x_\tau | x_{kh})} \right> \ \rd x_\tau } \nonumber\\
&\leq 4 \Ep{\p_\tau(x_\tau)}{\left< \nabla_x \nabla_r\ln\frac{\p_\tau(x_\tau)}{\p^*(x_\tau)}, S \nabla_x \nabla_r\ln\frac{\p_\tau(x_\tau)}{\p^*(x_\tau)} \right>_F} \nonumber\\
&+ \frac{1}{32} \Ep{\p_\tau}{\lrn{\nabla_\theta\ln\frac{\p_\tau(x_\tau)}{\p^*(x_\tau)}}^2}
+ \frac{9}{16} \Ep{\p_\tau}{\lrn{\nabla_r\ln\frac{\p_\tau(x_\tau)}{\p^*(x_\tau)}}^2} \nonumber\\
&+ \left( 68 L_G^2 + \frac{1}{8} \frac{L_H^2}{L_G}\right) \Ep{\p(x_{kh}, x_\tau)}{\lrn{\theta_\tau-\theta_{kh}}^2}
+ 18 e L_G d \max\left\{ L_G^4 (\tau-kh)^4, L_G^2 (\tau-kh)^2 \right\}. \nonumber
\end{align}
\end{proposition}
Roughly speaking, Proposition~\ref{proposition:disc_error} upper bounds the instantaneous contribution of the discretization error by the terms appearing in Eq.~\eqref{eq:cont_cvg} (the contraction of the continuous process), the variance of $\theta_\tau-\theta_{kh}$ (the progress of $\theta$ within one step), and constant terms that depend on the step size.
After combining Proposition~\ref{proposition:disc_error} with Proposition~\ref{proposition:cont_evolution}, the only nonnegative terms that remain are the variance of $\theta_\tau-\theta_{kh}$ and other constant terms.

We devote the rest of this subsection to the proof of Proposition~\ref{proposition:disc_error}.
We first 
expand term~\eqref{eq:Lyap_disc} using the definitions of the functional $\mathcal{L}$ as well as the discrete and continuous vector flows $\hat{v}_\tau^{AGD}$ and $v_\tau^{AGD}$.
\begin{lemma}
\label{lemma:disc_evolution}
For $\tau-kh\leq\frac{1}{8L_G}$, the time evolution of the Lyapunov functional $\mathcal{L}$ with respect to the discretization error $\hat{v}_\tau^{AGD}-v_\tau^{AGD}$ is:
\begin{subequations}
\begin{align}
\lefteqn{ \int \left< \nabla_x \frac{\delta \mathcal{L}}{\delta \p(x_\tau)}, \Ep{x_{kh}\sim\p(x_{kh})}{\left(\hat{v}_\tau^{AGD}-v_\tau^{AGD}\right)\p(x_\tau | x_{kh})} \right> \ \rd x_\tau } \nonumber\\
&= 2 \int \left< \nabla_\theta\ln\frac{\p_\tau(x_\tau)}{\p^*(x_\tau)},  \DiscError \right>_F \ \rd x_\tau \label{eq:disc_error_easy_1}\\
&+ 9 \int \left< \nabla_r\ln\frac{\p_\tau(x_\tau)}{\p^*(x_\tau)}, \DiscError \right> \ \rd x_\tau \label{eq:disc_error_easy_2}\\
&+ 2 \int \left< \nabla_x \nabla_r\ln\frac{\p_\tau(x_\tau)}{\p^*(x_\tau)}, S \nabla_{x_\tau} \Ep{x_{kh}\sim\p(x_{kh}|x_\tau)}{\nabla U(\theta_\tau)-\nabla U(\theta_{kh})} \right>_F \p_\tau(x_\tau)\ \rd x_\tau. \label{eq:disc_error_diff}
\end{align}
\end{subequations}
\end{lemma}
It can be observed that of the three terms~\eqref{eq:disc_error_easy_1}--\eqref{eq:disc_error_diff} in Lemma~\ref{lemma:disc_evolution},
there are two types of term: Terms~\eqref{eq:disc_error_easy_1} and~\eqref{eq:disc_error_easy_2} only involve first-order derivatives, $\nabla_{\#}\ln\frac{\p_\tau(x_\tau)}{\p^*(x_\tau)}$ (for $\#$ labeling $\theta$ or $r$); while term~\eqref{eq:disc_error_diff} involves a second-order derivative, $\nabla_x \nabla_r\ln\frac{\p_\tau(x_\tau)}{\p^*(x_\tau)}$.

For terms~\eqref{eq:disc_error_easy_1} and~\eqref{eq:disc_error_easy_2},
we make use of Young's inequality to obtain upper bounds:
\begin{subequations}
\begin{align}
\lefteqn{ \int \left< \nabla_{\theta}\ln\frac{\p_\tau(x_\tau)}{\p^*(x_\tau)}, \DiscError \right> \ \rd x_\tau } \nonumber\\
&\leq \frac{1}{64} \int \lrn{\nabla_{\theta}\ln\frac{\p_\tau(x_\tau)}{\p^*(x_\tau)}}^2 \p_\tau(x_\tau) \ \rd x_\tau + 16 L_G^2 \Ep{\p(x_\tau, x_{kh})}{\lrn{\theta_\tau-\theta_{kh}}^2}. \label{eq:disc_error_easy_bound_1}
\end{align}
\begin{align}
\lefteqn{ \int \left< \nabla_{r}\ln\frac{\p_\tau(x_\tau)}{\p^*(x_\tau)}, \DiscError \right> \ \rd x_\tau } \nonumber\\
&\leq \frac{1}{16} \int \lrn{\nabla_{r}\ln\frac{\p_\tau(x_\tau)}{\p^*(x_\tau)}}^2 \p_\tau(x_\tau) \ \rd x_\tau 
+ 4 L_G^2 \Ep{\p(x_\tau, x_{kh})}{\lrn{\theta_\tau-\theta_{kh}}^2}. \label{eq:disc_error_easy_bound_2}
\end{align}
\end{subequations}
%
The main difficulty is in bounding term~\eqref{eq:disc_error_diff},
which is the object of the following lemma.
%
\begin{lemma}
Under Assumption~\ref{A2}, we provide an explicit bound for term~\eqref{eq:disc_error_diff}. 
When $\tau-kh\leq\frac{1}{8L_G}$, $\gamma=2$, and $\xi=2L_G$,
\begin{align*}
\lefteqn{ \int \left< \nabla_x \nabla_r\ln\frac{\p_\tau(x_\tau)}{\p^*(x_\tau)}, S \nabla_{x_\tau} \Ep{x_{kh}\sim\p(x_{kh}|x_\tau)}{\nabla U(\theta_\tau)-\nabla U(\theta_{kh})} \right>_F \p_\tau(x_\tau)\ \rd x_\tau } \\
&\leq
2 \Ep{\p_\tau(x_\tau)}{\left< \nabla_x \nabla_r\ln\frac{\p_\tau(x_\tau)}{\p^*(x_\tau)}, S \nabla_x \nabla_r\ln\frac{\p_\tau(x_\tau)}{\p^*(x_\tau)} \right>_F} \\
&+ 9 e L_G d \max\left\{ L_G^4 (\tau-kh)^4, L_G^2 (\tau-kh)^2 \right\}
+ \frac{1}{16} \frac{L_H^2}{L_G} \Ep{\p(x_{kh}|x_\tau)\p_\tau(x_\tau)}{\lrn{\theta_\tau-\theta_{kh}}^2}.
\end{align*}
\label{lemma:disc_error_bound}
\end{lemma}
In the proof of Lemma~\ref{lemma:disc_error_bound}, we first upper bound the Frobenius inner product in term~\eqref{eq:disc_error_diff} by the (weighted) Frobenius norms of $\nabla_x \nabla_r\ln\frac{\p_\tau(x_\tau)}{\p^*(x_\tau)}$ and $\nabla_{x_\tau} \Ep{x_{kh}\sim\p(x_{kh}|x_\tau)}{\nabla U(\theta_\tau)-\nabla U(\theta_{kh})}$.
We then use a synchronous coupling technique to calculate $\nabla_{x_\tau} \Ep{x_{kh}\sim\p(x_{kh}|x_\tau)}{\nabla U(\theta_\tau)-\nabla U(\theta_{kh})}$ and provide an upper bound of its Frobenius norm.
We defer the complete proof to Appendix~\ref{subsection:discrete}.

Applying Eq.~\eqref{eq:disc_error_easy_bound_1}--\eqref{eq:disc_error_easy_bound_2} and Lemma~\ref{lemma:disc_error_bound} to Eq.~\eqref{eq:disc_error_easy_1}--\eqref{eq:disc_error_diff}, we bound the overall discretization error and finish the proof of Proposition~\ref{proposition:disc_error} as follows:
\begin{align}
\lefteqn{ \int \left< \nabla_x \frac{\delta \mathcal{L}}{\delta \p(x_\tau)}, \Ep{x_{kh}\sim\p(x_{kh})}{\left(\hat{v}_\tau^{AGD}-v_\tau^{AGD}\right)\p(x_\tau | x_{kh})} \right> \ \rd x_\tau } \nonumber\\
& \leq 4 \Ep{\p_\tau(x_\tau)}{\left< \nabla_x \nabla_r\ln\frac{\p_\tau(x_\tau)}{\p^*(x_\tau)}, S \nabla_x \nabla_r\ln\frac{\p_\tau(x_\tau)}{\p^*(x_\tau)} \right>_F} \nonumber\\
&+ \frac{1}{32} \Ep{\p_\tau}{\lrn{\nabla_\theta\ln\frac{\p_\tau(x_\tau)}{\p^*(x_\tau)}}^2}
+ \frac{9}{16} \Ep{\p_\tau}{\lrn{\nabla_r\ln\frac{\p_\tau(x_\tau)}{\p^*(x_\tau)}}^2} \nonumber\\
&+ \left( 68 L_G^2 + \frac{1}{8} \frac{L_H^2}{L_G}\right) \Ep{\p(x_{kh}, x_\tau)}{\lrn{\theta_\tau-\theta_{kh}}^2}
+ 18 e L_G d \max\left\{ L_G^4 (\tau-kh)^4, L_G^2 (\tau-kh)^2 \right\}. \nonumber
\end{align}

\subsection{Convergence of the underdamped Langevin algorithm}
Combining Propositions~\ref{proposition:cont_evolution} and~\ref{proposition:disc_error}, which establish the convergence rates of the continuous underdamped Langevin dynamics and the discretization error, we find that the overall time evolution of the Lyapunov functional $\mathcal{L}$ within each step of the underdamped Langevin algorithm can be upper bounded as follows:
\begin{subequations}
\begin{align}
\frac{\rd \mathcal{L}(\p_t)}{\rd t}
&= \int \left< \nabla_x \frac{\delta \mathcal{L}}{\delta \p_t}, v_t^{AGD} \right> \p_t \ \rd x \nonumber\\
&+ \int \left< \nabla_x \frac{\delta \mathcal{L}}{\delta \p_t}, \Ep{x_{kh}\sim\p(x_{kh})}{\left(\hat{v}_\tau^{AGD}-v_\tau^{AGD}\right)\p(x_\tau | x_{kh})}  \right> \p_t \ \rd x \nonumber\\
&\leq - \Ep{\p_t}{ \left< \nabla_x \ln \left(\frac{\p_t}{\p^*}\right), M \nabla_x \ln \left(\frac{\p_t}{\p^*}\right) \right>_F }
\label{eq:dL_init} \\
&+\left( 68 L_G^2 + \frac{1}{8} \frac{L_H^2}{L_G}\right) \Ep{\p(x_{kh}, x_\tau)}{\lrn{\theta_\tau-\theta_{kh}}^2} \label{eq:dL_error_1}\\
&+ 18 e L_G d \max\left\{ L_G^4 (\tau-kh)^4, L_G^2 (\tau-kh)^2 \right\},
\label{eq:dL_error_2}
\end{align}
\end{subequations}
where
\begin{align}
M = \left( \begin{array}{cc}
\frac{31}{32} \mI_{d\times d} 
& 4 \cdot \mI_{d\times d} - \frac{1}{8} \frac{\nabla^2 U(\theta)}{L_G} \vspace{4pt} \\
4 \cdot \mI_{d\times d} - \frac{1}{8} \frac{\nabla^2 U(\theta)}{L_G}
& \frac{279}{16} \mI_{d\times d} - \frac{1}{2} \frac{\nabla^2 U(\theta)}{L_G}
\end{array}
\right). \nonumber
\end{align}
In this section, we will further analyze terms~\eqref{eq:dL_init}--\eqref{eq:dL_error_2} to obtain the overall convergence rate of the underdamped Langevin algorithm.
We will need to quantify the convergence contributed by term~\eqref{eq:dL_init} and upper bound the extra discretization error in terms~\eqref{eq:dL_error_1}--\eqref{eq:dL_error_2} as the algorithm progresses.
After these two steps, choosing a suitable step size will finish the proof of Theorem~\ref{theorem:main}.

We begin by using the log-Sobolev inequality to relate term~\eqref{eq:dL_init} to the Lyapunov functional $\mathcal{L}(\p_t)$.
A key step is lower bounding matrix $M$ which is done in the following Lemma~\ref{lemma:eigenvalues} (the proof of which is deferred to Appendix~\ref{Append:overall_cvg}).
\begin{lemma}
Under Assumption~\ref{A2},
for any $L_G \geq 2\rho$,
$M \succeq \frac{\rho}{30} \left(S + \frac{1}{2 \rho}\mI_{2d\times 2d}\right)$.
\label{lemma:eigenvalues}
\end{lemma}
We can thus upper bound term~\eqref{eq:dL_init} using this lower bound on $M$ in conjunction with the log-Sobolev inequality, Eq.~\eqref{eq:logsob}:
\begin{align}
- \lefteqn{ \Ep{\p_t}{ \left< \nabla_x \ln \left(\frac{\p_t}{\p^*}\right), M \nabla_x \ln \left(\frac{\p_t}{\p^*}\right) \right>_F } } \nonumber\\
&\leq 
- \frac{\rho}{30} \left(\Ep{\p_t}{\lin{ \nabla_x \ln\lrp{\frac{\p_t(x)}{\p^*(x)}}, S \nabla_x \ln\lrp{\frac{\p_t(x)}{\p^*(x)}} }} + \frac{1}{2\rho} \Ep{\p_t}{\lrn{  \nabla_x \ln\lrp{\frac{\p_t(x)}{\p^*(x)}} }^2}\right) \nonumber\\
&\leq 
- \frac{\rho}{30} \left(\Ep{\p_t}{\lin{ \nabla_x \ln\lrp{\frac{\p_t(x)}{\p^*(x)}}, S \nabla_x \ln\lrp{\frac{\p_t(x)}{\p^*(x)}} }} + \Ep{\p_t}{\ln\lrp{\frac{\p_t(x)}{\p^*(x)}}}\right) \nonumber\\
&\leq
- \frac{\rho}{30} \cdot \mathcal{L}[\p_t].
\label{eq:log_Sobolev_Expanded}
\end{align}
Consequently, Eq.~\eqref{eq:dL_init}--\eqref{eq:dL_error_2} simplify to:
\begin{subequations}
\begin{align}
\frac{\rd \mathcal{L}(\p_t)}{\rd t}
&\leq - \frac{\rho}{30} \mathcal{L}(\p_t)
\label{eq:dL_overall_cont} \\
&+ \left( 68 L_G^2 + \frac{1}{8} \frac{L_H^2}{L_G}\right) \Ep{\p(x_{kh}, x_\tau)}{\lrn{\theta_\tau-\theta_{kh}}^2} 
\label{eq:dL_overall_error_1} \\
&+ 18 e L_G d \max\left\{ L_G^4 (\tau-kh)^4, L_G^2 (\tau-kh)^2 \right\}.
\label{eq:dL_overall_error_2}
\end{align}
\end{subequations} 
This implies that without the extra discretization error of terms~\eqref{eq:dL_overall_error_1}--\eqref{eq:dL_overall_error_2}, the Markov process converges exponentially (similarly as for the continuous dynamics) with a rate of $\rho/30$, proportional to the log-Sobolev constant.

We now focus on the second task of upper bounding terms~\eqref{eq:dL_overall_error_1}--\eqref{eq:dL_overall_error_2}.
The crux of the argument is to upper bound the variance of $\theta_\tau-\theta_{kh}$ as the algorithm progresses.
In the following lemma we show that for a suitable choice of step size, $\Ep{\p(x_{kh}, x_\tau)}{\lrn{\theta_\tau-\theta_{kh}}^2}$ is uniformly upper bounded by a term that scales as $\mathcal{O}(h^2 d)$.
\begin{lemma}
Assume that function $U$ satisfies Assumption~\ref{A1}--\ref{A3}, where $\rho$ denotes the minimum of the log-Sobolev constant and $1$.
Assume that we take $\gamma=2$, $\xi=2L_G$, and
\[
h = \frac{1}{56} \frac{1}{\sqrt{L_G}} 
\min\left\{ \frac{1}{24} \frac{\rho}{L_G}, \frac{ \sqrt{L_G} \rho }{L_H} \right\}
\cdot \min\left\{ \left(\widetilde{C_N} + 2\right)^{-1/2} \sqrt{\frac{\epsilon}{d}}, \sqrt{\frac{\epsilon}{C_M}} \right\},
\]
where $\epsilon \leq 2 d$.
Then for $\theta_\tau$ following Eq.~\eqref{eq:underdamp_diff_disc}, $\forall n \in \mathbb{N}^+$ and $\forall \tau \in [kh,(k+1)h]$,
\[
\Ep{\p(x_{kh}, x_\tau)}{\lrn{\theta_\tau-\theta_{kh}}^2}
\leq \left( \left( 24 \widetilde{C_N} + 26 \right) \frac{L_G}{\rho} \cdot d + 24 C_M \frac{L_G}{\rho} \right) h^2
= \mathcal{O} \left( \frac{L_G}{\rho} d \cdot h^2 \right).
\]
\label{lemma:bounded_variance_main}
\end{lemma}
To establish this uniform upper bound, we use an inductive argument---we prove that if the above bound holds for $t\leq kh$, then, given the effect of contraction and the discretization error in $[kh,\tau]$, the bound will still hold for any $\tau \in [kh,(k+1)h]$.
We defer the complete proof of Lemma~\ref{lemma:bounded_variance_main} to Appendix~\ref{Append:overall_cvg}.

Given this uniform bound for $\Ep{\p(x_{kh}, x_\tau)}{\lrn{\theta_\tau-\theta_{kh}}^2}$ across the entire interval, we can upper bound term~\eqref{eq:dL_overall_error_1} using our choice of the parameters $\gamma=2$, $\xi=2L_G$, and the step size $h$:
\begin{subequations}
\begin{align}
\lefteqn{\left( 68 L_G^2 + \frac{1}{8} \frac{L_H^2}{L_G}\right) \Ep{\p(x_{kh}, x_\tau)}{\lrn{\theta_\tau-\theta_{kh}}^2}} \nonumber\\
&= \left( 68 L_G^2 + \frac{1}{8} \frac{L_H^2}{L_G}\right) \left( \left( 24 \widetilde{C_N} + 26 \right) \frac{L_G}{\rho} \cdot d + 24 C_M \frac{L_G}{\rho} \right) h^2 \nonumber\\
&\leq \rho \cdot L_G \max\left\{136 \frac{L_G}{\rho}, \frac{1}{4} \frac{L_H^2}{L_G^2 \rho}\right\}
\cdot \max\left\{\left( 48 \widetilde{C_N} + 52 \right) \frac{L_G}{\rho} d, 48 C_M \frac{L_G}{\rho} \right\} h^2 \nonumber\\
&\leq \frac{49}{4} \rho \cdot L_G \max\left\{24^2 \frac{L_G^2}{\rho^2}, \frac{L_H^2}{L_G \rho^2}\right\}
\cdot \max\left\{\left( \widetilde{C_N} + 2 \right) d, C_M \right\} h^2 \nonumber\\
&\leq 
\frac{\rho}{30} \cdot \frac{\epsilon}{4}.
\label{eq:disc_final_1}
\end{align}
For term~\eqref{eq:dL_overall_error_2}, we obtain that
\begin{align}
18 e L_G d \max\left\{ L_G^4 (\tau-kh)^4, L_G^2 (\tau-kh)^2 \right\}
&\leq \frac{\rho}{30} \cdot 540 e \frac{L_G}{\rho} d\max\{L_G^4 h^4, L_G^2 h^2\} \nonumber\\
&\leq \frac{\rho}{30} \cdot \frac{\epsilon}{4}.
\label{eq:disc_final_2}
\end{align}
\end{subequations}

Plugging Eqs.~\eqref{eq:disc_final_1}--\eqref{eq:disc_final_2} into Eqs.~\eqref{eq:dL_overall_error_1}--\eqref{eq:dL_overall_error_2},
we obtain the following upper bound for $\frac{\rd \mathcal{L}(\p_t)}{\rd t}$:
\begin{align*}
\frac{\rd \mathcal{L}(\p_t)}{\rd t}
\leq - \frac{\rho}{30} \cdot \left( \mathcal{L}(\p_t) - \frac{\epsilon}{2} \right).
\end{align*}
Applying Gr\"onwall's lemma, we arrive at a bound for the Lyapunov functional at every step:
\[
\mathcal{L}[\p_{kh}] - \frac{\epsilon}{2}
\leq e^{-\frac{\rho}{30} h} \left(\mathcal{L}[\p_{(k-1)h}] - \frac{\epsilon}{2} \right)
\leq e^{-\frac{\rho}{30} h k} \left(\mathcal{L}[\p_0] - \frac{\epsilon}{2} \right)
< e^{-\frac{\rho}{30} h k} \mathcal{L}[\p_0].
\]
Therefore, for any $k \geq K = \frac{30}{\rho h}\ln\left(\frac{2\mathcal{L}[\p_0]}{\epsilon}\right)$, we have
$\KL{\p_{kh}}{\p^*} \leq \mathcal{L}[\p_{kh}] \leq \epsilon$.

We now use the definition of the step size $h$ and the upper bound on the initial value $\mathcal{L}[\p_0]$ from Lemma~\ref{lemma:initial_dist} to obtain the number of iterations for Algorithm~\ref{alg:main} to converge to within $\epsilon$ of the target distribution $\p^*$:
%
\begin{align*}
K
&= 1680 \max\left\{24 \frac{L_G^{3/2}}{\rho^2}, \frac{L_H}{\rho^2} \right\}
\cdot \max\left\{ \sqrt{\widetilde{C_N} + 2} \sqrt{\frac{d}{\epsilon}}, \sqrt{\frac{C_M}{\epsilon}} \right\}
\cdot \ln\left( 4 \max\left\{ \left(\widetilde{C_N} + 1\right)\frac{d}{\epsilon}, \frac{C_M}{\epsilon} \right\} \right) \\
&= {\mathcal{O}}\left(\sqrt{\frac{d}{\epsilon}} \ln \frac{d}{\epsilon} \right).
\end{align*}
%

If the function $U$ further satisfies assumptions~\ref{A1}---\ref{A3} (that $U$ is nonconvex inside a region of radius $R$ and $m$-strongly convex outside of it), we can instantiate the constants $\rho \geq \frac{m}{2} e^{-16 L_G R^2}$, $\widetilde{C_N} = C_N + \frac{1}{2}\ln\frac{L_G}{2\pi} \leq \frac{1}{2}\ln\frac{2L_G}{m}$, and $C_M \leq 32\frac{L_G^2}{m^2} L_G R^2$, and study the computational complexity in more detail.
The number of iterations required becomes:
\begin{align*}
K &= 4800 e^{32 L_G R^2}
\max\left\{24 \frac{L_G^{3/2}}{m^2}, \frac{L_H}{m^2} \right\}
\cdot \max\left\{ \sqrt{\ln\frac{L_G}{m} + 5} \sqrt{\frac{d}{\epsilon}}, 8R\frac{L_G}{m}\sqrt{\frac{L_G}{\epsilon}} \right\} \\
&\quad \cdot \ln\left( 2 \max\left\{ \left(\ln\frac{L_G}{m} + 4 \right)\frac{d}{\epsilon}, 64R^2 \frac{L_G^2}{m^2} \frac{L_G}{\epsilon} \right\} \right).
\end{align*}
Emphasizing the dimension dependency, we have:
\begin{align*}
K = {\mathcal{O}}\left( \max\left\{ \frac{L_G^{3/2}}{\rho^2}, \frac{L_H}{\rho^2} \right\} \sqrt{\frac{d}{\epsilon}} \ln \frac{d}{\epsilon} \right).
\end{align*}

\section{Discussion}
We have shown that there is an analog of Nesterov accelerated gradient method for MCMC---it is the underdamped Langevin algorithm. We demonstrated this by adopting a view of sampling algorithms as optimizing over the space of probability measures, with KL divergence as the objective functional. By constructing an appropriate Lyapunov functional, we were able to prove that the underdamped Langevin algorithm has an accelerated convergence rate compared to the classical overdamped Langevin algorithm. 

A line of recent results leverage richer stochastic dynamics to obtain better pre-conditioning and employ higher-order discretization schemes~\citep{RiemannianMALA,geod,thermostat}.
They observe that in practice such dynamics increase stability and in turn results in faster convergence of the algorithm.

Our particular approach involves multiplying the strong sub-differential of the KL divergence by a symplectic matrix and a positive semidefinite matrix. An interesting direction for future research would be to consider other, more general choices. Indeed, a general construction of underdamped stochastic processes would involve taking a vector field $v_t$ to have the following form:
\begin{align}
v_t = -(D(x) + Q(x))\nabla \ln\lrp{\frac{\p_t(x)}{\p^*(x)}}, \label{e:vector_flow}
\end{align}
where $D(x)$ is a positive semidefinite \emph{diffusion} matrix, and $Q(x)$ is a skew-symmetric \emph{curl} matrix.
This has the form of a generic dynamics for smooth optimization.
It can be checked that when $\p_t(x)=\p^*(x)$, $v_t=0$.
Therefore, $\p^*$ is a stationary distribution when $\p_t$ follows the vector flow $v_t$:
\begin{align}
\frac{\partial \p_t(x)}{\partial t}
=& - \div\lrp{\p_t(x) \cdot v_t} \nonumber\\
=& \div\lrp{\p_t(x) \cdot (D(x) + Q(x))\nabla \ln\lrp{\frac{\p_t(x)}{\p^*(x)}}}. \label{e:FPE}
\end{align}
It has been previously proved~\citep{completesample,completeframework} that any continuous Markov process with the stationary distribution $\p^*$ which satisfies an integrability condition can be represented in the form of Eq.~\eqref{e:FPE}.

To simulate the dynamics of $v_t$ on the state space of $x$, we can realize it as a stochastic process with an It\^o diffusion:
\begin{align}
\frac{\partial \p_t(x)}{\partial t}
=& \div\lrp{\p_t(x) \big(D(x) + Q(x)\big)\nabla \ln\lrp{\frac{\p_t(x)}{\p^*(x)}}} \nonumber\\
=& \sum_i \sum_j \frac{\del^2}{\del x_i \del x_j} \big( \p_t(x) D_{i,j}(x) \big) - \div\bigg(\p_t(x) \Big(\big(D(x) + Q(x)\big)\nabla \ln\p^*(x) + \Gamma(x)\Big) \bigg), \label{e:Ito_diff}
\end{align}
where $\Gamma_i(x) = \sum_j \frac{\del}{\del x_j}\lrb{D(x)+Q(x)}_{i,j}$.
Eq.~\eqref{e:Ito_diff} corresponds to the probability density of $x_t$ following a stochastic differential equation:
\begin{equation}
\label{e:SDE}
\rd x_t =  \lrp{(D(x)+Q(x))\nabla \ln\lrp{\p^*(x)} + \Gamma(x)} \rd t + \sqrt{2D(x)} \rd B_t.
\end{equation}

To study convergence of this process, denote the first variation of a functional $\G[\p_t]$ as $\frac{\delta \G}{\delta \p_t} [\p_t]: \Re^d \to \Re$.
If the vector flow $v_t$ satisfies the continuity equation for $\p_t$, Eq.~\eqref{e:FPE},
then
\begin{align}
\ddt \G[\p_t]
& = \int \frac{\delta \G}{\delta \p_t} [\p_t](x) \ddt \p_t(x) \rd x \nonumber\\
& = \int \frac{\delta \G}{\delta \p_t} [\p_t](x) (- \div(\p_t(x) v_t(x))) \rd x \nonumber\\
& = \int \lin{\nabla \frac{\delta \G}{\delta \p_t} [\p_t](x), v_t(x) } \p_t(x)\rd x \nonumber\\
& = \Ep{\p_t}{\lin{\nabla \frac{\delta \G}{\delta \p_t} [\p_t](x), v_t(x) }} \nonumber\\
& = - \Ep{\p_t}{\lin{\nabla \frac{\delta \G}{\delta \p_t} [\p_t](x), (D(x) + Q(x))\nabla \ln\lrp{\frac{\p_t(x)}{\p^*(x)}} }}. \label{eq:dF_general}
\end{align}
Using notation from statistical mechanics, we can represent Eq.~\eqref{eq:dF_general} in a more compact form using a (Ginzburg-Landau) dissipative bracket and a generalized Poisson bracket to generate the stochastic process $\ddt \p_t(x)$ with $\nabla \frac{\delta \F}{\delta \p_t}$.
Define the dissipative bracket $\{\cdot, \cdot\}$ as
\begin{align}
\{\G[\p_t], \F[\p_t]\} = \Ep{\p_t(x)}{\lin{\nabla \frac{\delta \G[\p_t](x)}{\delta \p_t(y)}, D(x) \nabla \frac{\delta \F[\p_t](x)}{\delta \p_t(y)}}};
\end{align}
and the generalized Poisson bracket $[\cdot,\cdot]$ as
\begin{align}
[\G[\p_t], \F[\p_t]] = \Ep{\p_t(x)}{\lin{\nabla \frac{\delta \G[\p_t](x)}{\delta \p_t(y)}, Q(x) \nabla \frac{\delta \F[\p_t](x)}{\delta \p_t(y)}}}.
\end{align}
Then
\begin{align}
\ddt \G[\p_t] 
& = - \Ep{\p_t}{\lin{\nabla \frac{\delta \G}{\delta \p_t} [\p_t](x), (D(x) + Q(x))\nabla \ln\lrp{\frac{\p_t(x)}{\p^*(x)}} }} \nonumber\\
& = - \{\G[\p_t], \F[\p_t]\} - [\G[\p_t], \F[\p_t]].
\end{align}

By taking $\G=\F$ as the KL-divergence, we can calculate its time derivative as:
\begin{align}
\ddt \KL{\p_t}{\p^*} &= \Ep{\p_t}{\lin{\nabla \ln\lrp{\frac{\p_t(x)}{\p^*(x)}}, - (D(x)+Q(x))\nabla \ln\lrp{\frac{\p_t(x)}{\p^*(x)}} }} \nonumber\\
&= - \Ep{\p_t}{\lin{\nabla \ln\lrp{\frac{\p_t(x)}{\p^*(x)}}, D(x)\nabla \ln\lrp{\frac{\p_t(x)}{\p^*(x)}} }}
\leq 0,
\end{align}
where we know from the positive semidefiniteness of $D(x)$ that $\KL{\p_t}{\p^*}$ is monotonically non-increasing.
If $D(x)$ were to be positive definite, we can directly obtain a linear convergence rate for the \emph{continuous process} using the log-Sobolev inequality.
However if $D(x)$ is just positive semidefinite (as is the case for the diffusion matrix that we encountered while analyzing the underdamped Langevin algorithm) we need to choose a well-designed Lyapunov functional to prove convergence (if the process indeed converges).

Some attempts have been made in this direction in the stochastic optimization literature for a class of constant $D$ and $Q$ matrices~\cite{Mert}. For the generic case, \cite{Lester} explores an approach based on Stein factors; this seems like a particularly promising avenue to explore further.

\section{Acknowledgements}
We would like to thank Jianfeng Lu, Chi Jin, and Nilesh Tripuraneni for many helpful discussions and insights. This work was partially supported by Army Research Office grant W911NF-17-1-0304, and National Science Foundation Grant NSF-IIS-1740855, NSF-IIS-1909365, and NSF-IIS-1619362.

\bibliographystyle{plain}
\bibliography{CompleteSampling}

\newpage
\appendix
\section{Local Nonconvexity Assumption}
\label{assumptions_loc}
For $\p^*(\theta) \propto e^{- U(\theta)}$, we call a function $U:\R^d \rightarrow \R$ \emph{locally nonconvex} with radius $R$ and global strong convexity $m$ if it satisfies the following assumptions:
\begin{enumerate}[label=(\alph*)]
\item \label{B2}
$U(\theta)$ is $m$-strongly convex for $\lrn{\theta}>R$.

That is: $V(\theta) = U(\theta)-\dfrac{m}{2}\lrn{\theta}_2^2$ is convex on $\Omega=\R^d\setminus\ball(0,R)$\footnote{Here we let $\ball(0,R)$ denote the closed ball of radius $R$ centered at $0$.}.
We then follow the definition of convexity on nonconvex domains~\citep{Econ_convexity,MinYan} to require that $\forall \theta\in\Omega$, any convex combination of $\theta=\lambda_1 \theta_1 + \cdots + \lambda_k \theta_{kh}$ with $\theta_1,\cdots, \theta_{kh}\in\Omega$ satisfies:
\[
V(\theta) \leq \lambda_1 V(\theta_1) + \cdots + \lambda_k V(\theta_{kh}).
\]

\item \label{B1}
$U(\theta)$ is $L_G$-Lipschitz smooth and Hessian $L_H$-Lipschitz.

That is: $U\in C^2(\R^d)$; $\forall \theta, \vartheta\in\R^d$, $\lrn{\nabla U(\theta) - \nabla U(\vartheta)} \leq L_G \lrn{\theta-\vartheta}$ and $\lrn{\nabla^2 U(\theta) - \nabla^2 U(\vartheta)}_F \allowbreak \leq L_H \lrn{\theta-\vartheta}$.

\item \label{B3}
For convenience, let $\nabla U(0)=0$ (i.e., zero is a local extremum).
\end{enumerate}

From~\cite{MCMC_nonconvex}, we know that $\rho \geq \dfrac{m}{2} e^{-16 L_G R^2}$.
We prove that the constants in Assumption~\ref{A3} are also upper bounded by functions of $m$, $L_G$, and $R$.
\begin{fact}
\label{fact:normalization}
If $\p^*(\theta) \propto e^{- U(\theta)}$ satisfy Assumptions~\ref{B2}--\ref{B3}, then the normalization constant $\displaystyle\int \exp(-U(\theta)) \rd\theta$ is upper bounded as follows:
\begin{align*}
\ln{\int \exp\left(-U(\theta)\right) \rd \theta}
= \dfrac{d}{2}\ln\dfrac{4\pi}{m} + 32 \dfrac{L_G^2}{m^2} L_G R^2.
\end{align*}
In other words, constants in Assumption~\ref{A3} are bounded as: $C_N \leq \dfrac{1}{2}\ln\dfrac{4\pi}{m}$,
and $C_M \leq 32\dfrac{L_G^2}{m^2}L_G R^2$.
\end{fact}

\section{Explicit Iteration Rule for Algorithm~\ref{alg:main}}
\label{Append:iteration}
We provide an explicit iteration formula for $x_\tau$ given $x_{kh}$ in Eq.~\eqref{eq:underdamp_diff_disc}.
Given $x_{kh}$ at the previous iteration, $x_\tau$ can be calculated as:
\begin{align}
\left\{
\begin{array}{l}
\theta_\tau = \theta_{kh} + \dfrac{1-e^{-\gamma\xi\step}}{\gamma} r_{kh} - \dfrac{1}{\gamma}\left(\step-\dfrac{1-e^{-\gamma\xi\step}}{\gamma\xi}\right)\nabla U(\theta_{kh}) + W_\theta \\
r_\tau = e^{-\gamma\xi\step} r_{kh} - \dfrac{1-e^{-\gamma\xi\step}}{\gamma\xi}\nabla U(\theta_{kh}) + W_r
\end{array}
\right., \label{eq:sim_updates}
\end{align}
where
\begin{align}
\left(
\begin{array}{l}
W_\theta \\
W_r
\end{array}
\right)
\sim
\mathcal{N}
\left(0,
\Sigma_\tau
\right). \nonumber
\end{align}
The covariance matrix $\Sigma\in\R^{2d\times2d}$ is
\[
\Sigma_\tau =
\left(\begin{array}{cc}
\Sigma_{1,1}(\tau) \ \mI_{d\times d}& \Sigma_{1,2}(\tau) \ \mI_{d\times d} \\
\Sigma_{1,2}(\tau) \ \mI_{d\times d}& \Sigma_{2,2}(\tau) \ \mI_{d\times d}
\end{array}\right),
\]
where 
\begin{eqnarray*}
\Sigma_{1,1}(\tau) &=& \dfrac{1}{\gamma}\left(2\step-\dfrac{3}{\gamma\xi}+\dfrac{4}{\gamma\xi}e^{-\gamma\xi\step}-\dfrac{1}{\gamma\xi}e^{-2\gamma\xi\step}\right); \\
\Sigma_{1,2}(\tau) &=& \dfrac{1+e^{-2\gamma\xi\step}-2e^{-\gamma\xi\step}}{\gamma\xi}; \\
\Sigma_{2,2}(\tau) &=& \dfrac{1-e^{-2\gamma\xi\step}}{\xi}.
\end{eqnarray*}

Therefore, the update rule in Algorithm~\ref{alg:main} can be expressed as:
\[
x_{(k+1)h} \sim \mathcal{N} \left( \mu\left(x_{kh}\right), \Sigma \right),
\]
where
\begin{align}
\mu\left(x_{kh}\right)
= \left(
\begin{array}{l}
\theta_{kh} + \dfrac{1-e^{-\gamma\xi h}}{\gamma} r_{kh} - \dfrac{1}{\gamma}\left(h-\dfrac{1 - e^{ - \gamma \xi h } }{\gamma\xi}\right)\nabla U(\theta_{kh}) \\
e^{-\gamma\xi h} r_{kh} - \dfrac{1-e^{-\gamma\xi h}}{\gamma\xi}\nabla U(\theta_{kh})
\end{array}
\right), \label{eq:mu_k}
\end{align}
and
\begin{align}
\Sigma =
\left(\begin{array}{cc}
\dfrac{1}{\gamma}\left(2 h-\dfrac{3}{\gamma\xi}+\dfrac{4}{\gamma\xi}e^{-\gamma\xi h}-\dfrac{1}{\gamma\xi}e^{-2\gamma\xi h}\right)\mI_{d\times d}& \dfrac{1+e^{-2\gamma\xi h}-2e^{-\gamma\xi h}}{\gamma\xi}\mI_{d\times d} \\
\dfrac{1+e^{-2\gamma\xi h}-2e^{-\gamma\xi h}}{\gamma\xi}\mI_{d\times d}& \dfrac{1-e^{-2\gamma\xi h}}{\xi}\mI_{d\times d}
\end{array}\right). \label{eq:Sigma}
\end{align}
In Algorithm~\ref{alg:main}, the hyperparameters are set to be: $\gamma = 2$, $\xi = 2 L_G$, and 
\begin{align}
h = \dfrac{1}{56} \dfrac{1}{\sqrt{L_G}} 
\min\left\{ \dfrac{1}{24} \dfrac{\rho}{L_G}, \dfrac{ \sqrt{L_G} \rho }{L_H} \right\}
\cdot \min\left\{ \left(\widetilde{C_N} + 2\right)^{-1/2} \sqrt{\dfrac{\epsilon}{d}}, \sqrt{\dfrac{\epsilon}{C_M}} \right\},
\label{eq:h_def}
\end{align}
where $\widetilde{C_N} = C_N + \dfrac{1}{2}\ln\dfrac{L_G}{2\pi}$.

\section{Convergence of the Continuous Process}
To simplify the notations in the proofs, we let $a=\dfrac{1}{L_G}$, $b=\dfrac{1}{4 L_G}$, and $c=\dfrac{2}{L_G}$, so that
\[
S=\dfrac{1}{L_G}
\left( \begin{array}{cc}
1/4 \ \mI_{d \times d} & 1/2 \ \mI_{d \times d} \\
1/2 \ \mI_{d \times d} & 2 \ \mI_{d \times d}
\end{array} \right)
=\left( \begin{array}{cc}
b \ \mI_{d \times d} & a/2 \ \mI_{d \times d} \\
a/2 \ \mI_{d \times d} & c \ \mI_{d \times d}
\end{array} \right).
\]
%

\begin{proof}[Proof of Proposition~\ref{proposition:cont_evolution}]
We first compute the time evolution of the Lyapunov function $\mathcal{L}$ with respect to the continuous time vector flow $v_t^{AGD}$ in Eq.~\eqref{eq:simple_irr}.
\begin{lemma}
\label{lemma:cont_flow_dL}
The time derivative of the Lyapunov functional $\mathcal{L}$ with respect to the continuous time vector flow $v_t^{AGD}$ in Eq.~\eqref{eq:simple_irr} with $\gamma=2$ and $\xi=2L_G$ is:
\begin{align*}
\dfrac{d}{\rd t} \mathcal{L}[\p_t]
&= \int \left< \nabla_x \dfrac{\delta \mathcal{L}}{\delta \p_t}, v_t^{AGD} \right> \p_t \ \rd x \nonumber\\
&= - \Ep{\p_t}{ \left< \nabla_x \ln \left(\dfrac{\p_t}{\p^*}\right), M_C \nabla_x \ln \left(\dfrac{\p_t}{\p^*}\right) \right>_F } \nonumber\\
& - 4 \Ep{\p_t}{ \left< \nabla_x \nabla_r \ln \left(\dfrac{\p_t}{\p^*}\right), S \nabla_x \nabla_r \ln \left(\dfrac{\p_t}{\p^*}\right) \right>_F },
\end{align*}
where
\begin{align}
M_C
&= \left( \begin{array}{ll}
\dfrac{a}{2} \xi \cdot \mI
& \dfrac{c+a\gamma}{2} \xi \cdot \mI - \dfrac{b}{2} \nabla^2 U(\theta) \vspace{4pt} \\
\dfrac{c+a\gamma}{2} \xi \cdot \mI - \dfrac{b}{2} \nabla^2 U(\theta)
& \gamma \left( 2c \xi + 1 \right)\mI - \dfrac{a}{2} \nabla^2 U(\theta)
\end{array}
\right) \nonumber\\
&= \left( \begin{array}{cc}
\mI_{d\times d}
& 4 \cdot \mI_{d\times d} - \dfrac{1}{8} \dfrac{\nabla^2 U(\theta)}{L_G} \vspace{4pt} \\
4 \cdot \mI_{d\times d} - \dfrac{1}{8} \dfrac{\nabla^2 U(\theta)}{L_G}
& 18 \cdot \mI_{d\times d} - \dfrac{1}{2} \dfrac{\nabla^2 U(\theta)}{L_G}
\end{array} \right).
\label{eq:M_C_appnd}
\end{align}
\end{lemma}

We then upper bound the time derivative of $\mathcal{L}$ by a negative factor times itself to obtain linear convergence rate.
\begin{lemma}
\label{MC_eig}
For $L_G$-Lipschitz smooth $U$, matrix $M_C$ defined in Eq.~\eqref{eq:M_C_appnd} satisfy:
\begin{align}
M_C \succeq
\dfrac{\rho}{10} \left(S + \dfrac{1}{2 \rho}\mI_{2d\times 2d}\right).
\label{eq:M_C_def}
\end{align}
\end{lemma}
Since the matrix $S$ is positive definite, we can directly bound the evolution of the Lyapunov functional  $\mathcal{L}$ as
\begin{align*}
\dfrac{d}{\rd t} \mathcal{L}[\p_t]
&\leq - \Ep{\p_t}{ \left< \nabla_x \ln \left(\dfrac{\p_t}{\p^*}\right), M_C \nabla_x \ln \left(\dfrac{\p_t}{\p^*}\right) \right>_F }
\\
&\leq - \dfrac{\rho}{10} \left(
\Ep{\p_t}{ \left< \nabla_x \ln \left(\dfrac{\p_t}{\p^*}\right), S \nabla_x \ln \left(\dfrac{\p_t}{\p^*}\right) \right>_F }
+ \dfrac{1}{2 \rho} \Ep{\p_t}{ \lrn{\nabla_x \ln\dfrac{\p_t}{\p^*}}^2 }
\right).
\end{align*}
Using the log-Sobolev inequality in Assumption~\ref{A1}, we directly obtain:
\begin{align}
\dfrac{d}{\rd t} \mathcal{L}[\p_t]
&\leq - \dfrac{\rho}{10} \left(
\Ep{\p_t}{ \left< \nabla_x \ln \left(\dfrac{\p_t}{\p^*}\right), S \nabla_x \ln \left(\dfrac{\p_t}{\p^*}\right) \right>_F }
+ \dfrac{1}{2 \rho} \Ep{\p_t}{ \lrn{\nabla_x \ln\dfrac{\p_t}{\p^*}}^2 }
\right) \nonumber\\
&\leq - \dfrac{\rho}{10} \left(
\Ep{\p_t}{ \left< \nabla_x \ln \left(\dfrac{\p_t}{\p^*}\right), S \nabla_x \ln \left(\dfrac{\p_t}{\p^*}\right) \right>_F }
+ \Ep{\p_t}{\ln\dfrac{\p_t}{\p^*}}
\right) \nonumber\\
&= - \dfrac{\rho}{10} \mathcal{L}[\p_t], \nonumber
\end{align}
which implies the linear convergence of the continuous process with a rate of $\dfrac{\rho}{10}$.
\end{proof}

\begin{proof}[Proof of Lemma~\ref{lemma:cont_flow_dL}]
Denote $h(\p_t) = \sqrt{\dfrac{\p_t}{\p^*}}$.
Then
\[
\mathcal{L}[\p_t] = \Ep{\p_t}{2\ln h + 4 \left< \nabla_x \ln h, S \nabla_x \ln h \right>}
=  2\Ep{\p_t}{\ln h} + 4 \Ep{\p^*}{\left< \nabla_x h, S \nabla_x h \right>}.
\]
The variational derivative of $\mathcal{L}[\p_t]$ can be thus calculated as:
\[
\dfrac{\delta \mathcal{L}[\p_t]}{\delta \p_t} = 2 \ln h + 1 + \dfrac{4}{h} (\nabla_x)^* S \nabla_x h,
\]
where the adjoint operator of $\nabla_x$ is with respect to the inner product: $\Ep{\p^*}{\left< \cdot, \cdot \right>}$.
Since:
\[
\Ep{p^*}{ \left< \nabla_x f, \overrightarrow{v} \right> }
=
\Ep{p^*}{ \left(- \nabla_x^\rT \overrightarrow{v} - \nabla_x^\rT \ln \p^*(x) \overrightarrow{v}\right) f}
\footnote{Here we define the $\nabla_x^\rT$ operator over a vector field $\overrightarrow{v}(x)$ as its divergence:
$\nabla_x^\rT \overrightarrow{v}(x) = \sum_i \dfrac{\partial v_i(x)}{\partial x_i}$.
},
\]
the adjoint operator can be expressed as:
\[
(\nabla_x)^* = - \nabla_x^\rT - \nabla_x^\rT \ln \p^*(x)
= \left(- \nabla_\theta^\rT + \nabla^\rT U(\theta), - \nabla_r^\rT + \xi r^\rT \right).
\]

The vector flow $v_t$ can also be expressed in terms of $h(\p_t)$ as:
\[
v_t = - 2 (D(x) + Q(x))\nabla_x \ln h
= - \dfrac{2}{h} (D(x) + Q(x)) \nabla_x h.
\]
Therefore,
\begin{align}
\lefteqn{ \Ep{p_t}{ \left< \nabla_x \dfrac{\delta \mathcal{L}}{\delta \p_t}, v_t \right> } } \nonumber\\
&= -4 \Ep{p^*}{ \left< \nabla_x h, (D(x) + Q(x)) \nabla_x h \right> } \label{eq:dt_1}\\
&\ \ \ - 8 \Ep{p^*}{ \left< \nabla_x (\nabla_x)^* S \nabla_x h, (D(x) + Q(x)) \nabla_x h \right> } \label{eq:dt_2}\\
&\ \ \ + 8 \Ep{p^*}{ \left< \nabla_x h, (D(x) + Q(x)) \nabla_x h \right> \dfrac{(\nabla_x)^* S \nabla_x h}{h} } \label{eq:dt_3}.
\end{align}
For Line~\eqref{eq:dt_1},
\[
-4\Ep{p^*}{ \left< \nabla_x h, (D(x) + Q(x)) \nabla_x h \right> } = -4\gamma \Ep{p^*}{ \|\nabla_r h\|^2 } = -\gamma \Ep{\p_t}{ \lrn{\nabla_r \ln\dfrac{\p_t}{\p^*}}^2 },
\]
same as in Eq.~\eqref{eq:dF}.

For Line~\eqref{eq:dt_3},
\begin{align}
\lefteqn{ 8 \Ep{p^*}{ \left< \nabla_x h, (D(x) + Q(x)) \nabla_x h \right> \dfrac{(\nabla_x)^* S \nabla_x h}{h} } } \nonumber\\
&= 8\gamma \Ep{p^*}{ \dfrac{1}{h} \left< \nabla_r h, \nabla_r h \right> (\nabla_x)^* S \nabla_x h } \nonumber\\
&= 8\gamma \Ep{p^*}{ \left< \dfrac{1}{h} \nabla_x \lrn{\nabla_r h}^2 - \dfrac{1}{h^2} \lrn{\nabla_r h}^2 \nabla_x h , S \nabla_x h \right> } \nonumber\\
&= 16\gamma \Ep{p^*}{ \left< \dfrac{\nabla_x h}{h} \nabla_r^\rT h, S \nabla_x \nabla_r^\rT h \right>_F }
- 8\gamma \Ep{p^*}{ \left< \dfrac{\nabla_x h}{h} \nabla_r^\rT h, S \dfrac{\nabla_x h}{h} \nabla_r^\rT h \right>_F } \label{eq:dt_3_simplify}.
\end{align}

Next we focus on Line~\eqref{eq:dt_2}.
\begin{lemma}
\begin{align}
\lefteqn{ - 8 \Ep{p^*}{ \left< \nabla_x (\nabla_x)^* S \nabla_x h, (D(x) + Q(x)) \nabla_x h \right> } } \nonumber\\
&= - 8 \gamma \Ep{p^*}{ \left< \nabla_x \nabla_r h, S \nabla_x \nabla_r h \right>_F } \label{eq:sum_up_cont_high_order}\\
&\ \ \ - 4a\xi \Ep{p^*}{ || \nabla_\theta h ||^2 } \nonumber\\
&\ \ \ - 4\Ep{p^*}{ \left< \nabla_r h, \left( 2 c \gamma \xi\mI - a\nabla^2 U(\theta) \right) \nabla_r h \right> } \nonumber\\
&\ \ \ - 4\Ep{p^*}{ \left< \nabla_\theta h, \left( (c \xi - a \gamma \xi )\mI - 2b\nabla^2 U(\theta) \right) \nabla_r h \right> }.
\label{eq:sum_up_cont}
\end{align}
\label{lemma:sum_up_cont}
\end{lemma}
Then Line~\eqref{eq:sum_up_cont_high_order} combines with Eq.~\eqref{eq:dt_3_simplify}:
\begin{align*}
\lefteqn{ - 8 \gamma \Ep{p^*}{ \left< \nabla_x \nabla_r h, S \nabla_x \nabla_r h \right>_F } } \\
&+ 16 \gamma \Ep{p^*}{ \left< \dfrac{\nabla_x h}{h} \nabla_r^\rT h, S \nabla_x \nabla_r h \right>_F } \\
&- 8 \gamma \Ep{p^*}{ \left< \dfrac{\nabla_x h}{h} \nabla_r^\rT h, S \dfrac{\nabla_x h}{h} \nabla_r^\rT h \right>_F } \\
&= -8 \gamma \Ep{p^*}{ \left< \left( \nabla_x \nabla_r^\rT h - \dfrac{\nabla_x h}{h} \nabla_r h \right), S \left( \nabla_x \nabla_r^\rT h - \dfrac{\nabla_x h}{h} \nabla_r h \right) \right>_F }
\end{align*}
Therefore, Lines~\eqref{eq:dt_1}--\eqref{eq:dt_3} sum up to be:
\begin{align}
\lefteqn{ \Ep{p_t}{ \left< \nabla_x \dfrac{\delta L}{\delta \p_t}, v_t \right> } } \nonumber\\
&= -8 \gamma \Ep{p^*}{ \left< \left( \nabla_x \nabla_r h - \dfrac{\nabla_x h}{h} \nabla_r^\rT h \right), S \left( \nabla_x \nabla_r h - \dfrac{\nabla_x h}{h} \nabla_r^\rT h \right) \right>_F } \nonumber\\
&- 4 \mathbb{E}_{p^*} \bigg<
\left( \begin{array}{l}
\nabla_\theta h\\
\nabla_r h
\end{array} \right) ,
M_C
\left( \begin{array}{l}
\nabla_\theta h\\
\nabla_r h
\end{array} \right)
\bigg>  \nonumber\\
&= -8 \gamma \Ep{\p_t}{ \left< \nabla_x \nabla_r \ln h, S \nabla_x \nabla_r \ln h \right>_F }
-4 \Ep{\p_t}{ \left< \nabla_x \ln h, M_C \nabla_x \ln h \right>_F } \nonumber\\
&= -2 \gamma \Ep{\p_t}{ \left< \nabla_x \nabla_r \ln \left(\dfrac{\p_t}{\p^*}\right), S \nabla_x \nabla_r \ln \left(\dfrac{\p_t}{\p^*}\right) \right>_F } \nonumber\\
&- \Ep{\p_t}{ \left< \nabla_x \ln \left(\dfrac{\p_t}{\p^*}\right), M_C \nabla_x \ln \left(\dfrac{\p_t}{\p^*}\right) \right>_F }, \nonumber
\end{align}
where
\begin{align}
M_C &= \left( \begin{array}{ll}
\dfrac{a}{2} \xi \cdot \mI
& \dfrac{c+a\gamma}{2} \xi \cdot \mI - \dfrac{b}{2} \nabla^2 U(\theta) \vspace{4pt} \\
\dfrac{c+a\gamma}{2} \xi \cdot \mI - \dfrac{b}{2} \nabla^2 U(\theta)
& \gamma \left( 2c \xi + 1 \right)\mI - \dfrac{a}{2} \nabla^2 U(\theta)
\end{array}
\right) \nonumber\\
&= \left( \begin{array}{cc}
\mI
& 4 \cdot \mI - \dfrac{1}{8} \dfrac{\nabla^2 U(\theta)}{L_G} \vspace{4pt} \\
4 \cdot \mI - \dfrac{1}{8} \dfrac{\nabla^2 U(\theta)}{L_G}
& 18 \cdot \mI - \dfrac{1}{2} \dfrac{\nabla^2 U(\theta)}{L_G}
\end{array}
\right). \label{eq:M_C_1}
\end{align}
\end{proof}

\begin{proof}[Proof of Lemma~\ref{MC_eig}]
We aim to prove that
\begin{align}
M_C &= \left( \begin{array}{ll}
\dfrac{a}{2} \xi \cdot \mI
& \dfrac{c+a\gamma}{2} \xi \cdot \mI - \dfrac{b}{2} \nabla^2 U(\theta) \vspace{4pt} \\
\dfrac{c+a\gamma}{2} \xi \cdot \mI - \dfrac{b}{2} \nabla^2 U(\theta)
& \gamma \left( 2c \xi + 1 \right)\mI - \dfrac{a}{2} \nabla^2 U(\theta)
\end{array}
\right)
\nonumber\\
&\succeq
\lambda \left(S + \dfrac{1}{2\rho} \mI \right)
=
\lambda
\left( \begin{array}{ll}
\left(b + \dfrac{1}{2\rho}\right) \mI & \dfrac{a}{2} \mI \vspace{4pt} \\
\dfrac{a}{2} \mI & \left(c + \dfrac{1}{2\rho}\right) \mI
\end{array}
\right), \nonumber
\end{align}
for $a=\dfrac{1}{L_G}$, $b=\dfrac{1}{4 L_G}$, $c=\dfrac{2}{L_G}$, $\gamma=2$, $\xi=2L_G$, and $\lambda=\dfrac{\rho}{10}$.
That is equivalent to having:
\begin{align}
\widehat{M}_C =
\left( \begin{array}{ll}
\left( \dfrac{a}{2} \xi - \left(b + \dfrac{1}{2\rho}\right) \lambda \right) \mI
& \left( \dfrac{c+a\gamma}{2} \xi - \dfrac{a}{2}\lambda \right)\mI - \dfrac{b}{2} \nabla^2 U(\theta) \vspace{4pt} \\
\left( \dfrac{c+a\gamma}{2} \xi - \dfrac{a}{2}\lambda \right)\mI - \dfrac{b}{2} \nabla^2 U(\theta)
& \left( \gamma \left( 2c \xi + 1 \right) -\left( c + \dfrac{1}{2\rho} \right) \lambda \right)\mI - \dfrac{a}{2} \nabla^2 U(\theta)
\end{array}
\right) \nonumber
\end{align}
to be positive semidefinite.

Denote $\alpha = \dfrac{a}{2} \xi - \left(b + \dfrac{1}{2\rho}\right) \lambda$,
$\beta = \dfrac{c+a\gamma}{2} \xi - \dfrac{a}{2}\lambda$,
and $\sigma = \gamma \left( 2c \xi + 1 \right) -\left( c + \dfrac{1}{2\rho} \right) \lambda$.
We analyze the eigenvalues of
$\widehat{M}_C = \left( \begin{array}{ll}
\alpha \mI
& \beta \mI - \dfrac{b}{2} \nabla^2 U(\theta) \vspace{4pt} \\
\beta \mI - \dfrac{b}{2} \nabla^2 U(\theta)
& \sigma \mI - \dfrac{a}{2} \nabla^2 U(\theta)
\end{array}
\right)$ and ask when they will all be nonnegative.
We write the characteristic equation for $\widehat{M}$:
\begin{align}
\det\left[ \widehat{M}_C - l \cdot \mI \right]
&=
\det\left[ \left( \begin{array}{ll}
(\alpha - l) \mI
& \beta \mI - \dfrac{b}{2} \nabla^2 U(\theta) \\
\beta \mI - \dfrac{b}{2} \nabla^2 U(\theta)
& (\sigma - l) \mI - \dfrac{a}{2} \nabla^2 U(\theta)
\end{array}
\right) \right]
\nonumber\\
&=
\det\left[
(\alpha - l)(\sigma - l)\mI - \dfrac{a}{2} (\alpha - l) \nabla^2 U(\theta) - \left( \beta \mI - \dfrac{b}{2} \nabla^2 U(\theta) \right)^2
\right]
= 0, \nonumber
\end{align}
since $\beta \mI - \dfrac{b}{2} \nabla^2 U(\theta)$ and $(\sigma - l) \mI - \dfrac{a}{2} \nabla^2 U(\theta)$ commute.
Diagonalizing $\nabla^2 U(\theta) = V^{-1} \Lambda V$, we obtain a set of independent equations based on each eigenvalue $\Lambda_j$ of $\nabla^2 U(\theta)$:
\[
l^2
+ \left( \dfrac{a}{2} \Lambda_j - \alpha - \sigma \right)l
- \left( \dfrac{b^2}{4} \Lambda_j^2 + \left(\dfrac{a}{2}\alpha - b \beta\right) \Lambda_j + \beta^2 - \alpha\sigma \right) = 0.
\]
To guarantee that $l\geq0$, we need that $\forall \Lambda_j \in [-L_G, L_G]$,
\begin{align}
\left\{
\begin{array}{l}
\dfrac{a}{2} \Lambda_j - \alpha - \sigma \leq 0 \vspace{5pt}\\
\dfrac{b^2}{4} \Lambda_j^2 + \left(\dfrac{a}{2}\alpha - b \beta\right) \Lambda_j + \beta^2 - \alpha\sigma \leq 0
\end{array}
\right. . \nonumber
\end{align}
Since the linear function $\dfrac{a}{2} \Lambda_j - \alpha - \sigma$ of $\Lambda_j$ is increasing; the quadratic function $\dfrac{b^2}{4} \Lambda_j^2 + \left(\dfrac{a}{2}\alpha - b \beta\right) \Lambda_j + \beta^2 - \alpha\sigma$ of $\Lambda_j$ is convex, we simply need the inequality to be satisfied at the end points:
\begin{align}
\left\{
\begin{array}{l}
\dfrac{a}{2} L_G - \alpha - \sigma \leq 0 \vspace{5pt}\\
\dfrac{b^2}{4} L_G^2 - \left(\dfrac{a}{2}\alpha - b \beta\right) L_G + \beta^2 - \alpha\sigma \leq 0 \vspace{5pt}\\
\dfrac{b^2}{4} L_G^2 + \left(\dfrac{a}{2}\alpha - b \beta\right) L_G + \beta^2 - \alpha\sigma \leq 0
\end{array}
\right. . \nonumber
\end{align}
We verify these inequalities by plugging in the setting of $a=\dfrac{1}{L_G}$, $b=\dfrac{1}{4 L_G}$, $c=\dfrac{2}{L_G}$, $\gamma=2$, $\xi=2L_G$, and $\lambda=\dfrac{\rho}{10}$ in the definition of $\alpha$, $\beta$, and $\sigma$.
We obtain:
\begin{align}
\left\{
\begin{array}{l}
\dfrac{a}{2} L_G - \alpha - \sigma
= -\dfrac{92}{5}+\dfrac{9 \rho}{40 L_G} \leq 0 \vspace{5pt}\\
\dfrac{b^2}{4} L_G^2 - \left(\dfrac{a}{2}\alpha - b\beta\right)L_G + \beta^2 - \alpha\sigma
= -\dfrac{819}{1600} + \dfrac{191 \rho}{800 L_G} - \dfrac{\rho^2}{400 L_G^2} \leq 0 \vspace{5pt}\\
\dfrac{b^2}{4} L_G^2 + \left(\dfrac{a}{2}\alpha - b\beta\right)L_G + \beta^2 - \alpha\sigma
= -\dfrac{2499}{1600} + \dfrac{191 \rho}{800 L_G} - \dfrac{\rho^2}{400 L_G^2} \leq 0
\end{array}
\right. . \nonumber
\end{align}

Therefore, $M_C \succeq \lambda \left(S + \dfrac{1}{2 \rho}\mI_{2d\times 2d}\right)$ for $a=\dfrac{1}{L_G}$, $b=\dfrac{1}{4 L_G}$, $c=\dfrac{2}{L_G}$, $\gamma=2$, $\xi=2L_G$, and $\lambda = \dfrac{L_G}{10}$.
\end{proof}

\subsection{Supporting Proof for Lemma~\ref{lemma:cont_flow_dL}}
\begin{proof}[Proof of Lemma~\ref{lemma:sum_up_cont}]
First note that $-8 \Ep{p^*}{ \left< \nabla_x (\nabla_x)^* S \nabla_x h, (D(x) + Q(x)) \nabla_x h \right> }$ separates into three terms:
\begin{align}
\lefteqn{ -8 \Ep{p^*}{ \left< \nabla_x (\nabla_x)^* S \nabla_x h, (D(x) + Q(x)) \nabla_x h \right> } } \nonumber\\
&= -4a \Ep{p^*}{ \left<
\nabla_x \big( (\nabla_\theta)^* \nabla_r h + (\nabla_r)^* \nabla_\theta h \big) , (D+Q) \nabla_x h \right> } \label{eq:dt_2_1}\\
&-8b \Ep{p^*}{ \left< \nabla_x (\nabla_\theta)^* \nabla_\theta h, (D+Q) \nabla_x h \right> } \label{eq:dt_2_2}\\
&-8c \Ep{p^*}{ \left< \nabla_x (\nabla_r)^* \nabla_r h, (D+Q) \nabla_x h \right> }. \label{eq:dt_2_3}
\end{align}
We then deal with the three terms one by one.
\begin{enumerate}
\item
For the cross term $-4a\Ep{p^*}{ \left<
\nabla_x \big( (\nabla_\theta)^* \nabla_r h + (\nabla_r)^* \nabla_\theta h \big) , (D+Q) \nabla_x h \right> }$ in Line~\ref{eq:dt_2_1},
\begin{align}
\lefteqn{ - \Ep{p^*}{ \left<
\nabla_x \big( (\nabla_\theta)^* \nabla_r h + (\nabla_r)^* \nabla_\theta h \big) , (D+Q) \nabla_x h \right> } }
\nonumber\\
&= - \Ep{p^*}{ \left<
\left( \begin{array}{l}
\nabla_\theta\\
\nabla_r
\end{array} \right)
\big( (\nabla_\theta)^* \nabla_r h + (\nabla_r)^* \nabla_\theta h \big) ,
(D+Q)
\left( \begin{array}{l}
\nabla_\theta h \\
\nabla_r h
\end{array} \right)
\right> }
\nonumber\\
&= -\gamma \Ep{p^*}{ \left<
\nabla_r
\big( (\nabla_\theta)^* \nabla_r h + (\nabla_r)^* \nabla_\theta h \big),
\nabla_r h
\right> } \label{eq:cross_1}
\\
& - \Ep{p^*}{ \left<
\left( \begin{array}{l}
\nabla_\theta\\
\nabla_r
\end{array} \right)
\big( (\nabla_\theta)^* \nabla_r h + (\nabla_r)^* \nabla_\theta h \big) ,
Q \left( \begin{array}{c}
\nabla_\theta h \\
\nabla_r h
\end{array} \right)
\right> }
. \label{eq:cross_2}
\end{align}
Here, $\nabla_\theta$ commutes with $\nabla_r$ and $(\nabla_r)^*$.

\begin{itemize}
\item
Hence Line~\eqref{eq:cross_1} equals:
\begin{align}
\lefteqn{ - \gamma \Ep{p^*}{ \left< \nabla_r (\nabla_\theta)^* \nabla_r h, \nabla_r h \right>
+ \left< \nabla_r (\nabla_r)^* \nabla_\theta h, \nabla_r h \right> } } \nonumber\\
&= - \gamma \Ep{p^*}{ \left< \nabla_r h, \nabla_\theta (\nabla_r)^* \nabla_r h \right>
+ \left< \nabla_\theta h, \nabla_r (\nabla_r)^* \nabla_r h \right> } \nonumber\\
&= - \gamma \Ep{p^*}{ \left< \nabla_r h, (\nabla_r)^* \nabla_r \nabla_\theta h \right>
+ \left< \nabla_\theta h, \nabla_r (\nabla_r)^* \nabla_r h \right> }\text{\footnotemark} \nonumber\\
&= - \gamma \Ep{p^*}{ \left< \nabla_\theta h, \big( (\nabla_r)^* \nabla_r + \nabla_r (\nabla_r)^* \big) \nabla_r h \right> }. \nonumber
\end{align}
\footnotetext{Here $(\nabla_r)^* \nabla_r \nabla_\theta h$ is a column vector with its elements defined as: $\big((\nabla_r)^* \nabla_r \nabla_\theta h\big)_i = \sum_{j}\left(\dfrac{\partial}{\partial r_j}\right)^*\dfrac{\partial}{\partial r_j}\dfrac{\partial}{\partial \theta_i} h$.}
We make use of the commutator of $\nabla_r$ and $(\nabla_r)^*$, $[\nabla_r, (\nabla_r)^*] \overrightarrow{v} = \nabla_r (\nabla_r)^* \overrightarrow{v}(x) - (\nabla_r)^* \nabla_r \overrightarrow{v}(x) = - \nabla_r \nabla_r^\rT \overrightarrow{v} + \xi \overrightarrow{v} + \nabla_r^\rT \nabla_r \overrightarrow{v}$,
and simplify Line~\eqref{eq:cross_1}:
\begin{align}
\lefteqn{ - \gamma \Ep{p^*}{ \left< \nabla_r (\nabla_\theta)^* \nabla_r h, \nabla_r h \right>
+ \left< \nabla_r (\nabla_r)^* \nabla_\theta h, \nabla_r h \right> } } \nonumber\\
&= - \gamma \Ep{p^*}{ \left< \nabla_\theta h, \big( 2 (\nabla_r)^* \nabla_r + [\nabla_r, (\nabla_r)^*] \big) \nabla_r h \right> } \nonumber\\
&= - \gamma \Ep{p^*}{ \left< \nabla_\theta h, 2 (\nabla_r)^* \nabla_r \nabla_r h + \xi \nabla_r h \right> }
\nonumber\\
&= - 2\gamma \Ep{p^*}{ \left< \nabla_r \nabla_\theta h, \nabla_r \nabla_r h \right>_F }
- \gamma \xi \Ep{p^*}{ \left< \nabla_\theta h, \nabla_r h \right> }, \nonumber
\end{align}
where we have used $\left< \cdot , \cdot \right>_F$ to also denote Frobenius inner product between matrices.

\item
Line~\eqref{eq:cross_2} can be simplified by using the representation of the vector flow in Eq.~\eqref{eq:simple_irr}:
\begin{align}
\lefteqn{ - \Ep{p^*}{ \left<
\left( \begin{array}{l}
\nabla_\theta\\
\nabla_r
\end{array} \right)
\big( (\nabla_\theta)^* \nabla_r h + (\nabla_r)^* \nabla_\theta h \big) ,
Q \left( \begin{array}{c}
\nabla_\theta h \\
\nabla_r h
\end{array} \right)
\right> } } \nonumber\\
&= - \Ep{p^*}{ \left<
\left( \begin{array}{l}
\nabla_\theta\\
\nabla_r
\end{array} \right)
\big( (\nabla_\theta)^* \nabla_r h + (\nabla_r)^* \nabla_\theta h \big) ,
Q \left( \begin{array}{c}
\nabla U(\theta) \\
\xi r
\end{array} \right) \dfrac{h}{2}
\right> } \nonumber\\
&= - \dfrac{1}{2} \Ep{p^*}{ \left<
\left( \begin{array}{l}
\nabla_\theta\\
\nabla_r
\end{array} \right) h,
\left( \begin{array}{l}
\nabla_r\\
\nabla_\theta
\end{array} \right)
\big( \xi r^\rT \nabla_\theta h - \nabla^\rT U(\theta) \nabla_r h \big)
\right> }. \label{eq:above_eq}
\end{align}
Denote $B[h] = \xi r^\rT \nabla_\theta h - \nabla^\rT U(\theta) \nabla_r h$, then $B$ is an anti-symmetric operator: $B^*[h] = -B[h]$.
Then Eq.~\eqref{eq:above_eq} can be further simplified:
\begin{align}
\lefteqn{ - \dfrac{1}{2} \Ep{p^*}{ \left<
\left( \begin{array}{l}
\nabla_\theta\\
\nabla_r
\end{array} \right) h,
\left( \begin{array}{l}
\nabla_r\\
\nabla_\theta
\end{array} \right)
\big( \xi r^\rT \nabla_\theta h - \nabla^\rT U(\theta) \nabla_r h \big)
\right> } } \nonumber\\
&= - \dfrac{1}{2} \Ep{p^*}{ \left<
\left( \begin{array}{l}
\nabla_\theta\\
\nabla_r
\end{array} \right) h,
\left( \begin{array}{l}
\nabla_r\\
\nabla_\theta
\end{array}
\right) B[h] \right> } \nonumber\\
&= - \dfrac{1}{2} \Ep{p^*}{
\left< \nabla_\theta h, \nabla_r B[h] \right>
+
\left< \nabla_r h, \nabla_\theta B[h] \right> }
\nonumber\\
&= - \dfrac{1}{2} \Ep{p^*}{
\left< \nabla_\theta h, \nabla_r B[h] \right>
+
\left< \nabla_r h, B \nabla_\theta [h] \right>
+
\left< \nabla_r h, [\nabla_\theta, B] [h] \right> }
\nonumber\\
&= - \dfrac{1}{2} \Ep{p^*}{
\left< \nabla_\theta h, \nabla_r B[h] \right>
-
\left< B \nabla_r h, \nabla_\theta [h] \right>
+
\left< \nabla_r h, [\nabla_\theta, B] [h] \right> }
\nonumber\\
&= - \dfrac{1}{2} \Ep{p^*}{ \left< \nabla_\theta h, [\nabla_r, B] [h] \right>
+
\left< \nabla_r h, [\nabla_\theta, B] [h] \right> }. \label{eq:above_eq_cont}
\end{align}
Since $[\nabla_r, B] [h] = \xi \nabla_\theta h$ and $[\nabla_\theta, B] [h] = - \nabla^2 U(\theta) \nabla_r h$,
Eq.~\eqref{eq:above_eq_cont} becomes
\begin{align}
\lefteqn{ - \dfrac{1}{2} \Ep{p^*}{ \left< \nabla_\theta h, [\nabla_r, B] [h] \right>
+
\left< \nabla_r h, [\nabla_\theta, B] [h] \right> } } \nonumber\\
&= - \dfrac{1}{2} \Ep{p^*}{
\xi \left< \nabla_\theta h, \nabla_\theta h \right>
- \left< \nabla_r h, \nabla^2 U(\theta) \nabla_r h \right>
}. \nonumber
\end{align}
Therefore, Line~\eqref{eq:cross_2} is
\begin{align}
\lefteqn{ - \Ep{p^*}{ \left<
\left( \begin{array}{l}
\nabla_\theta\\
\nabla_r
\end{array} \right)
\big( (\nabla_\theta)^* \nabla_r h + (\nabla_r)^* \nabla_\theta h \big) ,
Q \left( \begin{array}{c}
\nabla_\theta h \\
\nabla_r h
\end{array} \right)
\right> } } \nonumber\\
&= - \dfrac{1}{2} \Ep{p^*}{
\xi \left< \nabla_\theta h, \nabla_\theta h \right>
- \left< \nabla_r h, \nabla^2 U(\theta) \nabla_r h \right>
}. \nonumber
\end{align}
Summing up Lines~\eqref{eq:cross_1} and \eqref{eq:cross_2},
\begin{align}
\lefteqn{ - \Ep{p^*}{ \left<
\nabla_x \big( (\nabla_\theta)^* \nabla_r h + (\nabla_r)^* \nabla_\theta h \big) , (D+Q) \nabla_x h \right> } }
\nonumber\\
&=
- 2\gamma \Ep{p^*}{ \left< \nabla_\theta \nabla_r h, \nabla_r \nabla_r h \right>_F }
\nonumber\\ \nonumber
&- \gamma\xi \Ep{p^*}{ \left< \nabla_\theta h, \nabla_r h \right> }
- \dfrac{\xi}{2} \Ep{p^*}{ || \nabla_\theta h ||^2 }
+ \dfrac{1}{2} \Ep{p^*}{ \left< \nabla_r h, \nabla^2 U(\theta) \nabla_r h \right> }.
\end{align}
\end{itemize}

\item
For $-8b\Ep{p^*}{ \left< \nabla_x (\nabla_\theta)^* \nabla_\theta h, (D+Q) \nabla_x h \right> }$ in Line~\ref{eq:dt_2_2},
\begin{align}
\lefteqn{ -2\Ep{p^*}{ \left< \nabla_x (\nabla_\theta)^* \nabla_\theta h, (D+Q) \nabla_x h \right> }}
\nonumber\\
&= - 2\Ep{p^*}{ \left<
\left( \begin{array}{l}
\nabla_\theta\\
\nabla_r
\end{array} \right)
(\nabla_\theta)^* \nabla_\theta h ,
(D+Q)
\left( \begin{array}{l}
\nabla_\theta h \\
\nabla_r h
\end{array} \right)
\right> } \nonumber\\
&= - 2\gamma \Ep{p^*}{ \left< \nabla_r (\nabla_\theta)^* \nabla_\theta h, \nabla_r h \right> }
- \Ep{p^*}{ \left< \nabla_\theta h, \nabla_\theta B[h] \right> } \nonumber\\
&= - 2\gamma \Ep{p^*}{ \left< \nabla_r (\nabla_\theta)^* \nabla_\theta h, \nabla_r h \right> }
- \Ep{p^*}{ \left< \nabla_\theta h, B \nabla_\theta h + [\nabla_\theta, B][h] \right> } \nonumber\\
&= - 2\gamma \Ep{p^*}{ \left< \nabla_\theta \nabla_r h, \nabla_\theta \nabla_r h \right>_F }
+ \Ep{p^*}{ \left< \nabla_\theta h, \nabla^2 U(\theta) \nabla_r h \right> }. \nonumber
\end{align}

\item
For $-8c\Ep{p^*}{ \left< \nabla_x (\nabla_r)^* \nabla_r h, (D+Q) \nabla_x h \right> }$ in Line~\ref{eq:dt_2_3},
\begin{align}
\lefteqn{ -2\Ep{p^*}{ \left< \nabla_x (\nabla_r)^* \nabla_r h, (D+Q) \nabla_x h \right> }}
\nonumber\\
&= - 2\Ep{p^*}{ \left<
\left( \begin{array}{l}
\nabla_\theta\\
\nabla_r
\end{array} \right)
(\nabla_r)^* \nabla_r h ,
(D+Q)
\left( \begin{array}{l}
\nabla_\theta h \\
\nabla_r h
\end{array} \right)
\right> } \nonumber\\
&= - 2\gamma \Ep{p^*}{ \left< \nabla_r (\nabla_r)^* \nabla_r h, \nabla_r h \right> }
- \Ep{p^*}{ \left< \nabla_r h, \nabla_r B[h] \right> } \nonumber\\
&= - 2\gamma \Ep{p^*}{ \left< \left( (\nabla_r)^* \nabla_r + [\nabla_r, (\nabla_r)^*] \right) \nabla_r h, \nabla_r h \right> } \nonumber\\
&- \Ep{p^*}{ \left< \nabla_r h, B\nabla_r h + [\nabla_r, B][h] \right> } \nonumber\\
&= - 2\gamma \Ep{p^*}{ \left< \nabla_r \nabla_r h, \nabla_r \nabla_r h \right>_F } \nonumber\\
&- 2\gamma \xi \Ep{p^*}{ \left< \nabla_r h, \nabla_r h \right> }
- \xi \Ep{p^*}{ \left< \nabla_\theta h, \nabla_r h \right> }. \nonumber
\end{align}

\end{enumerate}
Summing everything up,
\begin{align}
\lefteqn{ - 8 \Ep{p^*}{ \left< \nabla_x (\nabla_x)^* S \nabla_x h, (D(x) + Q(x)) \nabla_x h \right> }} \nonumber\\
&= - 8a\gamma \Ep{p^*}{ \left< \nabla_\theta \nabla_r h, \nabla_r \nabla_r h \right>_F } \label{eq:sum_up_1}\\
&- 8b\gamma \Ep{p^*}{ \left< \nabla_\theta \nabla_r h, \nabla_\theta \nabla_r h \right>_F } \label{eq:sum_up_2}\\
&- 8c\gamma \Ep{p^*}{ \left< \nabla_r \nabla_r h, \nabla_r \nabla_r h \right>_F } \label{eq:sum_up_3}\\
&- 2a\xi \Ep{p^*}{ || \nabla_\theta h ||^2 } \nonumber\\
&- 4\Ep{p^*}{ \left< \nabla_r h, \left( 2 c \gamma \xi \mI - \dfrac{a}{2} \nabla^2 U(\theta) \right) \nabla_r h \right> } \nonumber\\
&- 4\Ep{p^*}{ \left< \nabla_\theta h, \left( (c \xi + a \gamma \xi)\mI - b\nabla^2 U(\theta) \right) \nabla_r h \right> }.  \nonumber
\end{align}

For Lines~\eqref{eq:sum_up_1}--\eqref{eq:sum_up_3},
\begin{align}
\lefteqn{ - a \Ep{p^*}{ \left< \nabla_\theta \nabla_r h, \nabla_r \nabla_r h \right>_F } }
\nonumber\\
&- b \Ep{p^*}{ \left< \nabla_\theta \nabla_r h, \nabla_\theta \nabla_r h \right>_F }
\nonumber\\
&- c \Ep{p^*}{ \left< \nabla_r \nabla_r h, \nabla_r \nabla_r h \right>_F }
\nonumber\\
&= - \gamma \Ep{p^*}{ \left< \nabla_x \nabla_r h, S \nabla_x \nabla_r h \right>_F }. \nonumber
\end{align}
Therefore,
\begin{align}
\lefteqn{ - 8 \Ep{p^*}{ \left< \nabla_x (\nabla_x)^* S \nabla_x h, (D(x) + Q(x)) \nabla_x h \right> }} \nonumber\\
&= - 8 \gamma \Ep{p^*}{ \left< \nabla_x \nabla_r h, S \nabla_x \nabla_r h \right>_F }\\
&- 2a\xi \Ep{p^*}{ || \nabla_\theta h ||^2 } \nonumber\\
&- 4\Ep{p^*}{ \left< \nabla_r h, \left( 2 c \gamma \xi \mI - \dfrac{a}{2} \nabla^2 U(\theta) \right) \nabla_r h \right> } \nonumber\\
&- 4\Ep{p^*}{ \left< \nabla_\theta h, \left( (c \xi + a \gamma\xi)\mI - b\nabla^2 U(\theta) \right) \nabla_r h \right> }
\label{eq:sum_up_cont_proof}.
\end{align}
\end{proof}

\section{Discretization Error}
\label{subsection:discrete}
\begin{proof}[Proof of Lemma~\ref{lemma:disc_evolution}]
As in the continuous case, define ${h}=\sqrt{\dfrac{\p_\tau(x_\tau)}{\p^*(x_\tau)}}$, and denote $a=\dfrac{1}{L_G}$, $b=\dfrac{1}{4 L_G}$, $c=\dfrac{2}{L_G}$.
First note that
\begin{align}
\lefteqn{ \int \left< \nabla_r \dfrac{\delta L}{\delta \p_t}, \Ep{x_{kh}\sim\p(x_{kh})}{ \big(\nabla U(\theta_\tau) - \nabla U(\theta_{kh})\big)\p(x_\tau | x_{kh}) } \right> \ \rd x_\tau }
\nonumber\\
&= \int \left< \nabla_r \left( 2 \ln h + 4 \dfrac{\nabla_x^* S \nabla_x h}{h} \right), \Ep{x_{kh}\sim\p(x_{kh})}{ \big(\nabla U(\theta_\tau) - \nabla U(\theta_{kh})\big)\p(x_\tau | x_{kh}) } \right> \ \rd x_\tau. \nonumber
\end{align}
We prove in the following that
\begin{align}
\lefteqn{ \int \left< \nabla_r \left( \dfrac{\nabla_x^* S \nabla_x h}{h} \right), \Ep{x_{kh}\sim\p(x_{kh})}{\big(\nabla U(\theta_\tau) - \nabla U(\theta_{kh})\big) \p(x_\tau | x_{kh})} \right> \ \rd x_\tau }  \label{eq:aux_discrete_error}\\
&= \int \left< \nabla_x\nabla_r\ln h , S \nabla_x \left(\dfrac{\Ep{x_{kh}\sim\p(x_{kh})}{\big(\nabla U(\theta_\tau) - \nabla U(\theta_{kh})\big) \p(x_\tau | x_{kh})}}{\p(x_\tau)}\right) \right>_F  {\p_\tau(x_\tau)} \ \rd x_\tau \nonumber\\
&+ \xi \int \left< \dfrac{a}{2}\nabla_\theta\ln h + c\nabla_r\ln h,  \Ep{x_{kh}\sim\p(x_{kh})}{\left(\nabla U(\theta_\tau) - \nabla U(\theta_{kh})\right) \p(x_\tau | x_{kh})} \right> \ \rd x_\tau. \nonumber
\end{align}

Similar to the continuous case, the term in Line~\eqref{eq:aux_discrete_error} separates into four terms:
\begin{align}
\lefteqn{ \int \left< \nabla_r \left( \dfrac{\nabla_x^* S \nabla_x h}{h} \right), \Ep{x_{kh}\sim\p(x_{kh})}{\big(\nabla U(\theta_\tau) - \nabla U(\theta_{kh})\big) \p(x_\tau | x_{kh})} \right> \ \rd x_\tau } \nonumber\\
&= b \int \left< \nabla_r \left( \dfrac{\nabla_\theta^* \nabla_\theta h}{h} \right), \Ep{x_{kh}\sim\p(x_{kh})}{\big(\nabla U(\theta_\tau) - \nabla U(\theta_{kh})\big) \p(x_\tau | x_{kh})} \right> \ \rd x_\tau \label{eq:discrete_dt_1}
\\
&+ \dfrac{a}{2} \int \left< \nabla_r \left( \dfrac{\nabla_\theta^* \nabla_r h}{h} \right), \Ep{x_{kh}\sim\p(x_{kh})}{\big(\nabla U(\theta_\tau) - \nabla U(\theta_{kh})\big) \p(x_\tau | x_{kh})} \right> \ \rd x_\tau \label{eq:discrete_dt_2}
\\
&+ \dfrac{a}{2} \int \left< \nabla_r \left( \dfrac{\nabla_r^* \nabla_\theta h}{h} \right), \Ep{x_{kh}\sim\p(x_{kh})}{\big(\nabla U(\theta_\tau) - \nabla U(\theta_{kh})\big) \p(x_\tau | x_{kh})} \right> \ \rd x_\tau \label{eq:discrete_dt_3}
\\
&+ c \int \left< \nabla_r \left( \dfrac{\nabla_r^* \nabla_r h}{h} \right), \Ep{x_{kh}\sim\p(x_{kh})}{\big(\nabla U(\theta_\tau) - \nabla U(\theta_{kh})\big) \p(x_\tau | x_{kh})} \right> \ \rd x_\tau. \label{eq:discrete_dt_4}
\end{align}
We first simplify Lines~\eqref{eq:discrete_dt_1} and \eqref{eq:discrete_dt_2} and then deal with Lines~\eqref{eq:discrete_dt_3} and \eqref{eq:discrete_dt_4}.
\begin{enumerate}
\item
For Lines~\eqref{eq:discrete_dt_1} and \eqref{eq:discrete_dt_2}:
\begin{align*}
\lefteqn{ \int \left< \nabla_r \left( \dfrac{\nabla_\theta^* \nabla_{\#} h}{h} \right), \Ep{x_{kh}\sim\p(x_{kh})}{\big(\nabla U(\theta_\tau) - \nabla U(\theta_{kh})\big) \p(x_\tau | x_{kh})} \right> \ \rd x_\tau } \\
&= \int \left< h \nabla_r \nabla_\theta^* \nabla_{\#} h - \nabla_r h \nabla_\theta^* \nabla_{\#} h , \dfrac{\DiscError}{\p(x_\tau)} \right> \p^*(x_\tau)\ \rd x_\tau \\
&= \int \left< \nabla_r\nabla_{\#} h,  \dfrac{\DiscError}{\p(x_\tau)} \nabla_\theta^\rT h \right>_F \p^*(x_\tau)\ \rd x_\tau \\
&- \int \left< \nabla_\theta\nabla_r h,  \dfrac{\DiscError}{\p(x_\tau)} \nabla_{\#}^T h \right>_F \p^*(x_\tau)\ \rd x_\tau \\
&+ \int \left<  h \nabla_r \nabla_{\#} h - \nabla_r h \nabla_{\#}^T h,  \nabla_\theta\left(\dfrac{\DiscError}{\p(x_\tau)}\right) \right>_F \p^*(x_\tau)\ \rd x_\tau \\
&= \int \left< \nabla_r\nabla_{\#} h,  \dfrac{\DiscError}{\p(x_\tau)} \nabla_\theta^\rT h \right>_F \p^*(x_\tau)\ \rd x_\tau \\
&- \int \left< \nabla_\theta\nabla_r h,  \dfrac{\DiscError}{\p(x_\tau)} \nabla_{\#}^T h \right>_F \p^*(x_\tau)\ \rd x_\tau \\
&+ \int \left< \nabla_r \nabla_{\#}\ln h,  \nabla_\theta\left(\dfrac{\DiscError}{\p(x_\tau)}\right) \right>_F \p_\tau(x_\tau)\ \rd x_\tau.
\end{align*}
When $\#=\theta$,
\begin{align*}
\lefteqn{ \int \left< \nabla_r \left( \dfrac{\nabla_\theta^* \nabla_\theta h}{h} \right), \Ep{x_{kh}\sim\p(x_{kh})}{\big(\nabla U(\theta_\tau) - \nabla U(\theta_{kh})\big) \p(x_\tau | x_{kh})} \right> \ \rd x_\tau } \\
&= \int \left< \nabla_r \nabla_\theta\ln h,  \nabla_\theta\left(\dfrac{\DiscError}{\p(x_\tau)}\right) \right>_F \p_\tau(x_\tau)\ \rd x_\tau.
\end{align*}
When $\#=r$,
\begin{align*}
\lefteqn{ \int \left< \nabla_r \left( \dfrac{\nabla_\theta^* \nabla_r h}{h} \right), \Ep{x_{kh}\sim\p(x_{kh})}{\big(\nabla U(\theta_\tau) - \nabla U(\theta_{kh})\big) \p(x_\tau | x_{kh})} \right> \ \rd x_\tau } \\
&= \int \left< \nabla_r^2 h,  \dfrac{\DiscError}{\p(x_\tau)} \nabla_\theta^\rT h \right>_F \p^*(x_\tau)\ \rd x_\tau \\
&- \int \left< \nabla_\theta\nabla_r h,  \dfrac{\DiscError}{\p(x_\tau)} \nabla_r^\rT h \right>_F \p^*(x_\tau)\ \rd x_\tau \\
&+ \int \left< \nabla_r^2 \ln h,  \nabla_\theta\left(\dfrac{\DiscError}{\p(x_\tau)}\right) \right>_F \p_\tau(x_\tau)\ \rd x_\tau.
\end{align*}

\item
For Lines~\eqref{eq:discrete_dt_3} and \eqref{eq:discrete_dt_4}:
\begin{align*}
\lefteqn{ \int \left< \nabla_r \left( \dfrac{\nabla_r^* \nabla_{\#} h}{h} \right), \Ep{x_{kh}\sim\p(x_{kh})}{\big(\nabla U(\theta_\tau) - \nabla U(\theta_{kh})\big) \p(x_\tau | x_{kh})} \right> \ \rd x_\tau } \\
&= \int \left< h \nabla_r \nabla_r^* \nabla_{\#} h - \nabla_r h \nabla_r^* \nabla_{\#} h , \dfrac{\DiscError}{\p(x_\tau)} \right> \p^*(x_\tau)\ \rd x_\tau \\
&= \xi \int \left< \dfrac{\nabla_{\#} h}{h},  \DiscError \right>_F \ \rd x_\tau \\
&+ \int \left< \nabla_r\nabla_{\#} h,  \dfrac{\DiscError}{\p(x_\tau)} \nabla_r^\rT h \right>_F \p^*(x_\tau)\ \rd x_\tau \\
&- \int \left< \nabla_r^2 h,  \dfrac{\DiscError}{\p(x_\tau)} \nabla_{\#}^T h \right>_F \p^*(x_\tau)\ \rd x_\tau \\
&+ \int \left<  h \nabla_r \nabla_{\#} h - \nabla_r h \nabla_{\#}^T h,  \nabla_r\left(\dfrac{\DiscError}{\p(x_\tau)}\right) \right>_F \p^*(x_\tau)\ \rd x_\tau \\
&= \xi \int \left< \nabla_{\#}\ln h,  \DiscError \right>_F \ \rd x_\tau \\
&+ \int \left< \nabla_r\nabla_{\#} h,  \dfrac{\DiscError}{\p(x_\tau)} \nabla_r^\rT h \right>_F \p^*(x_\tau)\ \rd x_\tau \\
&- \int \left< \nabla_r^2 h,  \dfrac{\DiscError}{\p(x_\tau)} \nabla_{\#}^T h \right>_F \p^*(x_\tau)\ \rd x_\tau \\
&+ \int \left< \nabla_r \nabla_{\#}\ln h,  \nabla_r\left(\dfrac{\DiscError}{\p(x_\tau)}\right) \right>_F \p_\tau(x_\tau)\ \rd x_\tau.
\end{align*}
When $\#=\theta$,
\begin{align*}
\lefteqn{ \int \left< \nabla_r \left( \dfrac{\nabla_r^* \nabla_\theta h}{h} \right), \Ep{x_{kh}\sim\p(x_{kh})}{\big(\nabla U(\theta_\tau) - \nabla U(\theta_{kh})\big) \p(x_\tau | x_{kh})} \right> \ \rd x_\tau } \\
&= \xi \int \left< \nabla_\theta\ln h,  \DiscError \right>_F \ \rd x_\tau \\
&+ \int \left< \nabla_r\nabla_\theta h,  \dfrac{\DiscError}{\p(x_\tau)} \nabla_r^\rT h \right>_F \p^*(x_\tau)\ \rd x_\tau \\
&- \int \left< \nabla_r^2 h,  \dfrac{\DiscError}{\p(x_\tau)} \nabla_\theta^\rT h \right>_F \p^*(x_\tau)\ \rd x_\tau \\
&+ \int \left< \nabla_r \nabla_\theta\ln h,  \nabla_r\left(\dfrac{\DiscError}{\p(x_\tau)}\right) \right>_F \p_\tau(x_\tau)\ \rd x_\tau.
\end{align*}
When $\#=r$,
\begin{align*}
\lefteqn{ \int \left< \nabla_r \left( \dfrac{\nabla_r^* \nabla_r h}{h} \right), \Ep{x_{kh}\sim\p(x_{kh})}{\big(\nabla U(\theta_\tau) - \nabla U(\theta_{kh})\big) \p(x_\tau | x_{kh})} \right> \ \rd x_\tau } \\
&= \xi \int \left< \nabla_r\ln h,  \DiscError \right>_F \ \rd x_\tau \\
&+ \int \left< \nabla_r^2\ln h,  \nabla_r\left(\dfrac{\DiscError}{\p(x_\tau)}\right) \right>_F \p_\tau(x_\tau)\ \rd x_\tau.
\end{align*}
\end{enumerate}
Therefore, Lines~\eqref{eq:discrete_dt_1}--\eqref{eq:discrete_dt_4} combines to be:
\begin{align}
\lefteqn{ \int \left< \nabla_r \left( \dfrac{\nabla_x^* S \nabla_x h}{h} \right), \Ep{x_{kh}\sim\p(x_{kh})}{\big(\nabla U(\theta_\tau) - \nabla U(\theta_{kh})\big) \p(x_\tau | x_{kh})} \right> \ \rd x_\tau } \nonumber\\
&= b \int \left< \nabla_r \nabla_\theta\ln h,  \nabla_\theta\left(\dfrac{\DiscError}{\p(x_\tau)}\right) \right>_F \p_\tau(x_\tau)\ \rd x_\tau \nonumber\\
&+ \dfrac{a}{2} \int \left< \nabla_r^2 \ln h,  \nabla_\theta\left(\dfrac{\DiscError}{\p(x_\tau)}\right) \right>_F \p_\tau(x_\tau)\ \rd x_\tau \nonumber\\
&+ \dfrac{a}{2} \int \left< \nabla_r \nabla_\theta\ln h,  \nabla_r\left(\dfrac{\DiscError}{\p(x_\tau)}\right) \right>_F \p_\tau(x_\tau)\ \rd x_\tau \nonumber\\
&+ c \int \left< \nabla_r^2\ln h,  \nabla_r\left(\dfrac{\DiscError}{\p(x_\tau)}\right) \right>_F \p_\tau(x_\tau)\ \rd x_\tau \nonumber\\
&+ \xi \int \left< \dfrac{a}{2} \nabla_\theta\ln h + c \nabla_r\ln h,  \DiscError \right>_F \ \rd x_\tau \nonumber\\
&= \int \left< \nabla_x \nabla_r\ln h, S \nabla_x\left(\dfrac{\DiscError}{\p(x_\tau)}\right) \right>_F \p_\tau(x_\tau)\ \rd x_\tau \nonumber\\
&+ \xi \int \left< \dfrac{a}{2} \nabla_\theta\ln h + c \nabla_r\ln h,  \DiscError \right>_F \ \rd x_\tau. \nonumber
\end{align}

Hence
\begin{align}
\lefteqn{ \int \left< \nabla_r \dfrac{\delta L}{\delta \p_t}, \DiscError \right> \ \rd x_\tau }
\nonumber\\
&= \int \left< \nabla_r \left( 2 \ln h + 4 \dfrac{\nabla_x^* S \nabla_x h}{h} \right), \DiscError \right> \ \rd x_\tau
\nonumber\\
&= \int \left< \nabla_r\ln\dfrac{\p_\tau(x_\tau)}{\p^*(x_\tau)}, \DiscError \right> \ \rd x_\tau \nonumber\\
&+ 2 \int \left< \nabla_x \nabla_r\ln\dfrac{\p_\tau(x_\tau)}{\p^*(x_\tau)}, S \nabla_x\left(\dfrac{\DiscError}{\p(x_\tau)}\right) \right>_F \p_\tau(x_\tau)\ \rd x_\tau \label{eq:difficulty}\\
&+ 2 \xi \int \left< \dfrac{a}{2} \nabla_\theta\ln\dfrac{\p_\tau(x_\tau)}{\p^*(x_\tau)} + c \nabla_r\ln\dfrac{\p_\tau(x_\tau)}{\p^*(x_\tau)},  \DiscError \right>_F \ \rd x_\tau. \nonumber
\end{align}

It can be seen that the expectation in Line~\eqref{eq:difficulty} can be rewritten as $x_{kh}$ conditioning on $x_\tau$:
\begin{align*}
\lefteqn{ \int \left< \nabla_x \nabla_r\ln\dfrac{\p_\tau(x_\tau)}{\p^*(x_\tau)}, S \nabla_x\left(\dfrac{\DiscError}{\p_\tau(x_\tau)}\right) \right>_F \p_\tau(x_\tau)\ \rd x_\tau } \\
&= \int \left< \nabla_x \nabla_r\ln\dfrac{\p_\tau(x_\tau)}{\p^*(x_\tau)}, S \nabla_{x_\tau} \Ep{x_{kh}\sim\p(x_{kh}|x_\tau)}{\nabla U(\theta_\tau)-\nabla U(\theta_{kh})} \right>_F \p_\tau(x_\tau)\ \rd x_\tau.
\end{align*}
Therefore,
\begin{align}
\lefteqn{ \int \left< \nabla_r \dfrac{\delta L}{\delta \p_t}, \DiscError \right> \ \rd x_\tau }
\nonumber\\
&= \int \left< \nabla_r\ln\dfrac{\p_\tau(x_\tau)}{\p^*(x_\tau)}, \DiscError \right> \ \rd x_\tau \nonumber\\
&+ 2 \int \left< \nabla_x \nabla_r\ln\dfrac{\p_\tau(x_\tau)}{\p^*(x_\tau)}, S \nabla_{x_\tau} \Ep{x_{kh}\sim\p(x_{kh}|x_\tau)}{\nabla U(\theta_\tau)-\nabla U(\theta_{kh})} \right>_F \p_\tau(x_\tau)\ \rd x_\tau \nonumber\\
&+ 2 \xi \int \left< \dfrac{a}{2} \nabla_\theta\ln\dfrac{\p_\tau(x_\tau)}{\p^*(x_\tau)} + c \nabla_r\ln\dfrac{\p_\tau(x_\tau)}{\p^*(x_\tau)},  \DiscError \right>_F \ \rd x_\tau \nonumber\\
&= a \xi \int \left< \nabla_\theta\ln\dfrac{\p_\tau(x_\tau)}{\p^*(x_\tau)},  \DiscError \right>_F \ \rd x_\tau \nonumber\\
&+ (2c\xi+1) \int \left< \nabla_r\ln\dfrac{\p_\tau(x_\tau)}{\p^*(x_\tau)}, \DiscError \right> \ \rd x_\tau \nonumber\\
&+ 2 \int \left< \nabla_x \nabla_r\ln\dfrac{\p_\tau(x_\tau)}{\p^*(x_\tau)}, S \nabla_{x_\tau} \Ep{x_{kh}\sim\p(x_{kh}|x_\tau)}{\nabla U(\theta_\tau)-\nabla U(\theta_{kh})} \right>_F \p_\tau(x_\tau)\ \rd x_\tau. \nonumber
\end{align}
\end{proof}

\begin{proof}[Proof of Lemma~\ref{lemma:disc_error_bound}]
We first explicitly calculate $\nabla_{x_\tau} \Ep{x_{kh}\sim\p(x_{kh}|x_\tau)}{\nabla U(\theta_\tau)-\nabla U(\theta_{kh})}$ in the following Lemma~\ref{lemma:diff_exp}.
To obtain the expression, we use synchronous coupling of the trajectories of underdamped Langevin algorithm with infinitesimally different initial conditions.
\begin{lemma}
Denote $\nu = \tau-kh \leq h$ and
\[
\eta = \dfrac{1}{\gamma}\left(
\dfrac{e^{\gamma\xi\step}\left(1-e^{-\gamma\xi\step}\right)^2}{\gamma\xi}
-\left(\step-\dfrac{1-e^{-\gamma\xi\step}}{\gamma\xi}\right)
\right)
\sim \mathcal{O}\left(\xi\nu^2\right).
\]
Then for $\nu\leq\dfrac{1}{8L_G}$ (and $\gamma=2$, and $\xi=2L_G$),
\begin{align}
\lefteqn{ \nabla_{x_\tau} \Ep{x_{kh}\sim\p(x_{kh}|x_\tau)}{\nabla U(\theta_\tau)-\nabla U(\theta_{kh})} } \nonumber\\
&= \mathbb{E}_{x_{kh}\sim\p(x_{kh}|x_\tau)}\left(
\begin{array}{c}
\left(\nabla^2 U(\theta_\tau) - \nabla^2 U(\theta_{kh}) \right)
+ \nabla^2 U(\theta_{kh}) \left(\left(\mI+\eta\nabla^2 U(\theta_{kh})\right)^{-1}-\mI\right) \\
- \dfrac{e^{\gamma\xi\step}-1}{\gamma} \nabla^2 U(\theta_{kh}) \left(\mI+\eta\nabla^2 U(\theta_{kh})\right)^{-1}
\end{array}
\right). \label{eq:diff_exp}
\end{align}
\label{lemma:diff_exp}
\end{lemma}
Taking Lemma~\ref{lemma:diff_exp} as given, we can separate Term~\eqref{eq:disc_error_diff} into two:
\begin{subequations}
\begin{align}
\lefteqn{ \int \left< \nabla_x \nabla_r\ln\dfrac{\p_\tau(x_\tau)}{\p^*(x_\tau)}, S \nabla_{x_\tau} \Ep{x_{kh}\sim\p(x_{kh}|x_\tau)}{\nabla U(\theta_\tau)-\nabla U(\theta_{kh})} \right>_F \p_\tau(x_\tau)\ \rd x_\tau } \nonumber\\
&= \int \int \left<
S \nabla_x \nabla_r\ln\dfrac{\p_\tau(x_\tau)}{\p^*(x_\tau)},
\left(\begin{array}{c}
\nabla^2 U(\theta_\tau) - \nabla^2 U(\theta_{kh}) \\
0
\end{array}\right)
\right>_F \p(x_{kh}|x_\tau) \p_\tau(x_\tau)\ \rd x_{kh} \rd x_\tau
\label{eq:disc_diff_separate_1}\\
&+
\int \int \left<
S \nabla_x \nabla_r\ln\dfrac{\p_\tau(x_\tau)}{\p^*(x_\tau)},
\left(\begin{array}{c}
\nabla^2 U(\theta_{kh}) \left(\left(\mI+\eta\nabla^2 U(\theta_{kh})\right)^{-1}-\mI\right) \\
- \dfrac{e^{\gamma\xi\step}-1}{\gamma} \nabla^2 U(\theta_{kh}) \left(\mI+\eta\nabla^2 U(\theta_{kh})\right)^{-1}
\end{array}\right)
\right>_F \nonumber\\
& \; \quad \qquad \cdot \p(x_{kh}|x_\tau) \p_\tau(x_\tau)\ \rd x_{kh} \rd x_\tau.
\label{eq:disc_diff_separate_2}
\end{align}
\end{subequations}
We then make use of the properties of Frobenius inner product to upper bound Terms~\eqref{eq:disc_diff_separate_1} and~\eqref{eq:disc_diff_separate_2} by the Frobenius norms:
\begin{align*}
\lefteqn{ \left< S \nabla_x \nabla_r\ln\dfrac{\p_\tau(x_\tau)}{\p^*(x_\tau)}, A_{2d \times d} \right>_F } \\
&= \left< \sqrt{S} \nabla_x \nabla_r\ln\dfrac{\p_\tau(x_\tau)}{\p^*(x_\tau)}, \sqrt{S} A_{2d \times d} \right>_F \\
&\leq \alpha \left< \nabla_x \nabla_r\ln\dfrac{\p_\tau(x_\tau)}{\p^*(x_\tau)}, S \nabla_x \nabla_r\ln\dfrac{\p_\tau(x_\tau)}{\p^*(x_\tau)} \right>_F
+ \dfrac{1}{4\alpha} \left< A_{2d \times d}, S A_{2d \times d} \right>_F.
\end{align*}
As a result, we obtain that for Term~\eqref{eq:disc_diff_separate_1},
\begin{align*}
\lefteqn{ \left<
S \nabla_x \nabla_r\ln\dfrac{\p_\tau(x_\tau)}{\p^*(x_\tau)},
\left(\begin{array}{c}
\nabla^2 U(\theta_\tau) - \nabla^2 U(\theta_{kh}) \\
0
\end{array}\right)
\right>_F } \\
&\leq
\dfrac{\gamma}{2} \left< \nabla_x \nabla_r\ln\dfrac{\p_\tau(x_\tau)}{\p^*(x_\tau)}, S \nabla_x \nabla_r\ln\dfrac{\p_\tau(x_\tau)}{\p^*(x_\tau)} \right>_F
+ \dfrac{b}{2\gamma} \lrn{\nabla^2 U(\theta_\tau) - \nabla^2 U(\theta_{kh})}_F^2;
\end{align*}
and for Term~\eqref{eq:disc_diff_separate_2},
\begin{align}
\lefteqn{ \left<
S \nabla_x \nabla_r\ln\dfrac{\p_\tau(x_\tau)}{\p^*(x_\tau)},
\left(\begin{array}{c}
\nabla^2 U(\theta_{kh}) \left(\left(\mI+\eta\nabla^2 U(\theta_{kh})\right)^{-1}-\mI\right) \\
- \dfrac{e^{\gamma\xi\step}-1}{\gamma} \nabla^2 U(\theta_{kh}) \left(\mI+\eta\nabla^2 U(\theta_{kh})\right)^{-1}
\end{array}\right)
\right>_F } \nonumber\\
&\leq
\dfrac{\gamma}{2} \left< \nabla_x \nabla_r\ln\dfrac{\p_\tau(x_\tau)}{\p^*(x_\tau)}, S \nabla_x \nabla_r\ln\dfrac{\p_\tau(x_\tau)}{\p^*(x_\tau)} \right>_F \nonumber\\
&+ \dfrac{1}{2\gamma} \Bigg<\left(\begin{array}{c}
\nabla^2 U(\theta_{kh}) \left(\left(\mI+\eta\nabla^2 U(\theta_{kh})\right)^{-1}-\mI\right) \\
- \dfrac{e^{\gamma\xi\step}-1}{\gamma} \nabla^2 U(\theta_{kh}) \left(\mI+\eta\nabla^2 U(\theta_{kh})\right)^{-1}
\end{array}\right), \nonumber\\
&\qquad\qquad S \left(\begin{array}{c}
\nabla^2 U(\theta_{kh}) \left(\left(\mI+\eta\nabla^2 U(\theta_{kh})\right)^{-1}-\mI\right) \\
- \dfrac{e^{\gamma\xi\step}-1}{\gamma} \nabla^2 U(\theta_{kh}) \left(\mI+\eta\nabla^2 U(\theta_{kh})\right)^{-1}
\end{array}\right)\Bigg>_F \nonumber\\
&\leq
\dfrac{\gamma}{2} \left< \nabla_x \nabla_r\ln\dfrac{\p_\tau(x_\tau)}{\p^*(x_\tau)}, S \nabla_x \nabla_r\ln\dfrac{\p_\tau(x_\tau)}{\p^*(x_\tau)} \right>_F \nonumber\\
&+ \dfrac{(b+c)d}{2\gamma} \lrn{
\left(\begin{array}{c}
\nabla^2 U(\theta_{kh}) \left(\left(\mI+\eta\nabla^2 U(\theta_{kh})\right)^{-1}-\mI\right) \\
- \dfrac{e^{\gamma\xi\step}-1}{\gamma} \nabla^2 U(\theta_{kh}) \left(\mI+\eta\nabla^2 U(\theta_{kh})\right)^{-1}
\end{array}\right)}_2^2.
\label{eq:disc_diff_separate_2_intermediate}
\end{align}

To obtain the final bound, we simplify Eq.~\eqref{eq:disc_diff_separate_2_intermediate} by demonstrating the following fact.
\begin{fact}
For $0\leq\nu\leq\min\left\{\dfrac{1}{\gamma\xi}, \dfrac{1}{\sqrt{2e L_G \xi}}\right\}$,
\[
\lrn{\left(\begin{array}{c}
\nabla^2 U(\theta_{kh}) \left(\left(\mI+\eta\nabla^2 U(\theta_{kh})\right)^{-1}-\mI\right) \\
- \dfrac{e^{\gamma\xi\step}-1}{\gamma} \nabla^2 U(\theta_{kh}) \left(\mI+\eta\nabla^2 U(\theta_{kh})\right)^{-1}
\end{array}\right)}_2
\leq 4e \max\{L_G^2\xi\nu^2, L_G\xi\nu\}.
\]
\label{fact:disc_bound}
\end{fact}
Since $\nu\leq\dfrac{1}{8L_G}\leq\min\left\{\dfrac{1}{\gamma\xi}, \dfrac{1}{\sqrt{2e L_G \xi}}\right\}$, and $\lrn{\nabla^2 U(\theta_\tau) - \nabla^2 U(\theta_{kh})}_F\leq L_H \lrn{\theta_\tau-\theta_{kh}}$, we plug the above inequalities into Terms~\eqref{eq:disc_diff_separate_1} and~\eqref{eq:disc_diff_separate_2} and arrive at our conclusion:
\begin{align*}
\lefteqn{ \int \left< \nabla_x \nabla_r\ln\dfrac{\p_\tau(x_\tau)}{\p^*(x_\tau)}, S \nabla_{x_\tau} \Ep{x_{kh}\sim\p(x_{kh}|x_\tau)}{\nabla U(\theta_\tau)-\nabla U(\theta_{kh})} \right>_F \p_\tau(x_\tau)\ \rd x_\tau } \\
&\leq
\gamma \Ep{\p_\tau(x_\tau)}{\left< \nabla_x \nabla_r\ln\dfrac{\p_\tau(x_\tau)}{\p^*(x_\tau)}, S \nabla_x \nabla_r\ln\dfrac{\p_\tau(x_\tau)}{\p^*(x_\tau)} \right>_F} \\
&+ \dfrac{2e(b+c)d}{\gamma}\max\{L_G^4\xi^2\nu^4, L_G^2\xi^2\nu^2\} + \dfrac{b L_H^2}{2\gamma} \Ep{\p(x_{kh}|x_\tau)\p_\tau(x_\tau)}{\lrn{\theta_\tau-\theta_{kh}}^2}.
\end{align*}

\end{proof}

\begin{proof}[Proof of Lemma~\ref{lemma:diff_exp}]
We study the following term with an arbitrary vector $v\in\mathbb{R}^{2d}$ (and denote $\hat{x}_n = \left(\hat{\theta}_n, \hat{r}_n\right) \in\mathbb{R}^{2d}$):
\begin{align*}
\lefteqn{ v^\rT \nabla_{x_\tau} \Ep{x_{kh}\sim\p(x_{kh}|x_\tau)}{\nabla U(\theta_\tau)-\nabla U(\theta_{kh})} } \\
&= \lim\limits_{h\rightarrow0} \dfrac{1}{h} \Ep{\substack{x_{kh}\sim\p(x_{kh}|x_\tau)\\\hat{x}_n\sim\p(\hat{x}_n|x_\tau+hv)}}{\big(\nabla U(\theta_\tau+hv)-\nabla U(\hat{\theta}_n)\big)-\big(\nabla U(\theta_\tau)-\nabla U(\theta_{kh})\big)} \\
&= \lim_{h\rightarrow0} \dfrac{1}{h} \Ep{\left(x_{kh},\hat{x}_n\right)\sim\Gamma\left(\p(x_{kh}|x_\tau),\p(\hat{x}_n|x_\tau+hv)\right)}{\big(\nabla U(\theta_\tau+hv)-\nabla U(\theta_\tau)\big)-\big(\nabla U(\hat{\theta}_n)-\nabla U(\theta_{kh})\big)},
\end{align*}
where $\Gamma\left(\p(x_{kh}|x_\tau),\p(\hat{x}_n|x_\tau+hv)\right)$ is any joint distribution of $x_{kh}$ and $\hat{x}_n$ with marginal distributions being $\p(x_{kh}|x_\tau)$ and $\p(\hat{x}_n|x_\tau+hv)$ -- any coupling between the two random variables.

Recall from~\eqref{eq:sim_updates} that the relation between $x_\tau$ and $x_{kh}$ is:
\begin{align}
\left\{
\begin{array}{l}
\theta_\tau = \theta_{kh} + \dfrac{1-e^{-\gamma\xi\step}}{\gamma} r_{kh} - \dfrac{1}{\gamma}\left(\step-\dfrac{1-e^{-\gamma\xi\step}}{\gamma\xi}\right)\nabla U(\theta_{kh}) + W_\theta \\
r_\tau = r_{kh} - \left(1-e^{-\gamma\xi\step}\right)r_{kh} - \dfrac{1-e^{-\gamma\xi\step}}{\gamma\xi}\nabla U(\theta_{kh}) + W_r
\end{array}
\right., \label{eq:reverse_updates}
\end{align}
where $W_x^\rT=\left(W_\theta^\rT, W_r^\rT\right)$ is the Gaussian random variable.
It can be proven that for step size $\nu\leq h \leq\dfrac{1}{8L_G}$, $x_{kh}$ is uniquely determined given $x_\tau$ and $W_x$.
Here we take the parallel coupling between $x_{kh}$ and $\hat{x}_n$.
Namely, we take:
\begin{align}
\left\{
\begin{array}{l}
\theta_\tau + hv_\theta = \hat{\theta}_n + \dfrac{1-e^{-\gamma\xi\step}}{\gamma} \hat{r}_n - \dfrac{1}{\gamma}\left(\step-\dfrac{1-e^{-\gamma\xi\step}}{\gamma\xi}\right)\nabla U(\hat{\theta}_n) + W_\theta \\
r_\tau + hv_r = \hat{r}_n - \left(1-e^{-\gamma\xi\step}\right)\hat{r}_n - \dfrac{1-e^{-\gamma\xi\step}}{\gamma\xi}\nabla U(\hat{\theta}_n) + W_r
\end{array}
\right., \nonumber
\end{align}
where the Gaussian random variable $W_x$ takes the same value as that in Eq.~\eqref{eq:reverse_updates}.
Then we get that for any pair of $\left(x_{kh},\hat{x}_n\right)$ following this joint law,
\begin{align*}
\hat{\theta}_n-\theta_{kh}
= hv_\theta + h \Delta(\bar{\theta}),
\end{align*}
where we define
\[
\Delta(\bar{\theta})
= \left(\left(\mI+\eta\nabla^2 U(\bar{\theta})\right)^{-1}-\mI\right) v_\theta
- \dfrac{e^{\gamma\xi\step}-1}{\gamma}\left(\mI+\eta\nabla^2 U(\bar{\theta})\right)^{-1} v_r,
\]
$\bar{\theta}$ a convex combination of $\theta_{kh}$ and $\hat{\theta}_n$, and
\[
\eta = \dfrac{1}{\gamma}\left(
\dfrac{e^{\gamma\xi\step}\left(1-e^{-\gamma\xi\step}\right)^2}{\gamma\xi}
-\left(\step-\dfrac{1-e^{-\gamma\xi\step}}{\gamma\xi}\right)
\right)
\sim \mathcal{O}(\xi \nu^2).
\]
Therefore,
\begin{align*}
\lefteqn{ \Ep{\left(x_{kh},\hat{x}_n\right)\sim\Gamma\left(\p(x_{kh}|x_\tau),\p(\hat{x}_n|x_\tau+hv)\right)}{\big(\nabla U(\theta_\tau+hv)-\nabla U(\theta_\tau)\big)-\big(\nabla U(\hat{\theta}_n)-\nabla U(\theta_{kh})\big)} } \\
&= \Ep{\left(x_{kh},\hat{x}_n\right)\sim\Gamma\left(\p(x_{kh}|x_\tau),\p(\hat{x}_n|x_\tau+hv)\right)}{\nabla^2 U(\tilde{\theta})hv_\theta-\nabla^2 U(\bar{\theta})\left(\hat{\theta}_n-\theta_{kh}\right)} \\
&= \Ep{\Gamma}{\left(\nabla^2 U(\tilde{\theta}) - \nabla^2 U(\bar{\theta}) \right)hv_\theta
+ h \nabla^2 U(\bar{\theta})\Delta(\bar{\theta})},
\end{align*}
where $\tilde{\theta}$ is a convex combination of $\theta_\tau$ and $\theta_\tau+hv$.
Taking the limit $h\rightarrow0$, we have:
\begin{align*}
\lefteqn{ v^\rT \nabla_{x_\tau} \Ep{x_{kh}\sim\p(x_{kh}|x_\tau)}{\nabla U(\theta_\tau)-\nabla U(\theta_{kh})} } \\
&= \lim\limits_{h\rightarrow0} \dfrac{1}{h} \Ep{\substack{x_{kh}\sim\p(x_{kh}|x_\tau)\\\hat{x}_n\sim\p(\hat{x}_n|x_\tau+hv)}}{\big(\nabla U(\theta_\tau+hv)-\nabla U(\hat{\theta}_n)\big)-\big(\nabla U(\theta_\tau)-\nabla U(\theta_{kh})\big)} \\
&= \Ep{x_{kh}\sim\p(x_{kh}|x_\tau)}{\left(\nabla^2 U(\theta_\tau) - \nabla^2 U(\theta_{kh}) \right) v_\theta
+ \nabla^2 U(\theta_{kh}) \Delta(\theta_{kh})}.
\end{align*}
Therefore,
\begin{align}
\lefteqn{ \nabla_{x_\tau} \Ep{x_{kh}\sim\p(x_{kh}|x_\tau)}{\nabla U(\theta_\tau)-\nabla U(\theta_{kh})} } \nonumber\\
&= \mathbb{E}_{x_{kh}\sim\p(x_{kh}|x_\tau)}\left(
\begin{array}{c}
\left(\nabla^2 U(\theta_\tau) - \nabla^2 U(\theta_{kh}) \right)
+ \nabla^2 U(\theta_{kh}) \left(\left(\mI+\eta\nabla^2 U(\theta_{kh})\right)^{-1}-\mI\right) \\
- \dfrac{e^{\gamma\xi\step}-1}{\gamma} \nabla^2 U(\theta_{kh}) \left(\mI+\eta\nabla^2 U(\theta_{kh})\right)^{-1}
\end{array}
\right). \nonumber
\end{align}
\end{proof}

\section{Overall Convergence of the Underdamped Langevin Algorithm}
\label{Append:overall_cvg}

\begin{proof}[Proof of Lemma~\ref{lemma:eigenvalues}]
We aim to prove that
\begin{align}
M &= \left( \begin{array}{ll}
\dfrac{31}{64} a \xi \cdot \mI
& \dfrac{c+a\gamma}{2} \xi \cdot \mI - \dfrac{b}{2} \nabla^2 U(\theta) \vspace{4pt} \\
\dfrac{c+a\gamma}{2} \xi \cdot \mI - \dfrac{b}{2} \nabla^2 U(\theta)
& \dfrac{31}{32} \gamma \left( 2c \xi + 1 \right)\mI - \dfrac{a}{2} \nabla^2 U(\theta)
\end{array}
\right)
\nonumber\\
&\succeq
\lambda \left(S + \dfrac{1}{2\rho} \mI \right)
=
\lambda
\left( \begin{array}{ll}
\left(b + \dfrac{1}{2\rho}\right) \mI & \dfrac{a}{2} \mI \vspace{4pt} \\
\dfrac{a}{2} \mI & \left(c + \dfrac{1}{2\rho}\right) \mI
\end{array}
\right), \nonumber
\end{align}
for $a=\dfrac{1}{L_G}$, $b=\dfrac{1}{4 L_G}$, $c=\dfrac{2}{L_G}$, $\gamma=2$, $\xi=2L_G$, and $\lambda=\dfrac{\rho}{30}$.
That is equivalent to having:
\begin{align}
\widehat{M} =
\left( \begin{array}{ll}
\left( \dfrac{31}{64} a\xi - \left(b + \dfrac{1}{2\rho}\right) \lambda \right) \mI
& \left( \dfrac{c+a\gamma}{2} \xi - \dfrac{a}{2}\lambda \right)\mI - \dfrac{b}{2} \nabla^2 U(\theta) \vspace{4pt} \\
\left( \dfrac{c+a\gamma}{2} \xi - \dfrac{a}{2}\lambda \right)\mI - \dfrac{b}{2} \nabla^2 U(\theta)
& \left( \dfrac{31}{32} \gamma \left( 2c \xi + 1 \right) -\left( c + \dfrac{1}{2\rho} \right) \lambda \right)\mI - \dfrac{a}{2} \nabla^2 U(\theta)
\end{array}
\right) \nonumber
\end{align}
to be positive semidefinite.

Denote $\alpha = \dfrac{31}{64} a\xi - \left(b + \dfrac{1}{2\rho}\right) \lambda$,
$\beta = \dfrac{c+a\gamma}{2} \xi - \dfrac{a}{2}\lambda$,
and $\sigma = \dfrac{31}{32} \gamma \left( 2c \xi + 1 \right) -\left( c + \dfrac{1}{2\rho} \right) \lambda$.
Then we analyze the eigenvalues of
$\widehat{M} = \left( \begin{array}{ll}
\alpha \mI
& \beta \mI - \dfrac{b}{2} \nabla^2 U(\theta) \vspace{4pt} \\
\beta \mI - \dfrac{b}{2} \nabla^2 U(\theta)
& \sigma \mI - \dfrac{a}{2} \nabla^2 U(\theta)
\end{array}
\right)$ and ask when they will all be nonnegative.
We write the characteristic equation for $\widehat{M}$:
\begin{align}
\det\left[ \widehat{M} - l \cdot \mI \right]
&=
\det\left[ \left( \begin{array}{ll}
(\alpha - l) \mI
& \beta \mI - \dfrac{b}{2} \nabla^2 U(\theta) \\
\beta \mI - \dfrac{b}{2} \nabla^2 U(\theta)
& (\sigma - l) \mI - \dfrac{a}{2} \nabla^2 U(\theta)
\end{array}
\right) \right]
\nonumber\\
&=
\det\left[
(\alpha - l)(\sigma - l)\mI - \dfrac{a}{2} (\alpha - l) \nabla^2 U(\theta) - \left( \beta \mI - \dfrac{b}{2} \nabla^2 U(\theta) \right)^2
\right]
= 0, \nonumber
\end{align}
since $\beta \mI - \dfrac{b}{2} \nabla^2 U(\theta)$ and $(\sigma - l) \mI - \dfrac{a}{2} \nabla^2 U(\theta)$ commute.
Diagonalizing $\nabla^2 U(\theta) = V^{-1} \Lambda V$, we obtain a set of independent equations based on each eigenvalue $\Lambda_j$ of $\nabla^2 U(\theta)$:
\[
l^2
+ \left( \dfrac{a}{2} \Lambda_j - \alpha - \sigma \right)l
- \left( \dfrac{b^2}{4} \Lambda_j^2 + \left(\dfrac{a}{2}\alpha - b \beta\right) \Lambda_j + \beta^2 - \alpha\sigma \right) = 0.
\]
To guarantee that $l\geq0$, we need that $\forall \Lambda_j \in [-L_G, L_G]$,
\begin{align}
\left\{
\begin{array}{l}
\dfrac{a}{2} \Lambda_j - \alpha - \sigma \leq 0 \vspace{5pt} \\
\dfrac{b^2}{4} \Lambda_j^2 + \left(\dfrac{a}{2}\alpha - b \beta\right) \Lambda_j + \beta^2 - \alpha\sigma \leq 0
\end{array}
\right. . \nonumber
\end{align}
Since the linear function $\dfrac{a}{2} \Lambda_j - \alpha - \sigma$ of $\Lambda_j$ is increasing; the quadratic function $\dfrac{b^2}{4} \Lambda_j^2 + \left(\dfrac{a}{2}\alpha - b \beta\right) \Lambda_j + \beta^2 - \alpha\sigma$ of $\Lambda_j$ is convex, we simply need the inequality to satisfy at the end points:
\begin{align}
\left\{
\begin{array}{l}
\dfrac{a}{2} L_G - \alpha - \sigma \leq 0 \vspace{5pt}\\
\dfrac{b^2}{4} L_G^2 - \left(\dfrac{a}{2}\alpha - b \beta\right) L_G + \beta^2 - \alpha\sigma \leq 0 \vspace{5pt}\\
\dfrac{b^2}{4} L_G^2 + \left(\dfrac{a}{2}\alpha - b \beta\right) L_G + \beta^2 - \alpha\sigma \leq 0
\end{array}
\right. . \nonumber
\end{align}
We verify these inequalities by plugging in the setting of $a=\dfrac{1}{L_G}$, $b=\dfrac{1}{4 L_G}$, $c=\dfrac{2}{L_G}$, $\gamma=2$, $\xi=2L_G$, and $\lambda=\dfrac{\rho}{30}$, in the definition of $\alpha$, $\beta$, and $\sigma$.
Then for $L_G \geq 2\rho$, we obtain that
\begin{align}
\left\{
\begin{array}{l}
\dfrac{a}{2} L_G - \alpha - \sigma
= - \dfrac{8579}{480} + \dfrac{3 \rho}{40 L_G} \leq 0 \vspace{5pt}\\
\dfrac{b^2}{4} L_G^2 - \left(\dfrac{a}{2}\alpha - b\beta\right)L_G + \beta^2 - \alpha\sigma
= -\dfrac{5357}{115200}+\dfrac{241 \rho}{3200 L_G}-\dfrac{\rho^2}{3600 L_G^2} \leq 0 \vspace{5pt}\\
\dfrac{b^2}{4} L_G^2 + \left(\dfrac{a}{2}\alpha - b\beta\right)L_G + \beta^2 - \alpha\sigma
= -\dfrac{126077}{115200} + \dfrac{241 \rho}{3200 L_G} - \dfrac{\rho^2}{3600 L_G^2} \leq 0
\end{array}
\right. . \nonumber
\end{align}

Therefore, $M\succeq
\lambda \left( S + \dfrac{1}{2\rho}\mI_{2d\times2d} \right)$ when we take
$a=\dfrac{1}{L_G}$, $b=\dfrac{1}{4 L_G}$, $c=\dfrac{2}{L_G}$, $\gamma=2$, and $\xi=2L_G$, where the contraction rate $\lambda$ is $\lambda=\dfrac{\rho}{30}$.

\end{proof}

\begin{proof}[Proof of Lemma~\ref{lemma:bounded_variance_main}]
For the expectation of ${\lrn{\theta_\tau-\theta_{kh}}^2}$ taken over the joint distribution of $(x_\tau,x_{kh})$, we use the definition of $x_\tau$ in our Equation~\eqref{eq:underdamp_diff_disc} to expand it (by way of Jensen's inequality):
\begin{align}
\Ep{\p(x_{kh}, x_\tau)}{\lrn{\theta_\tau-\theta_{kh}}^2}
&= \xi\E{\lrn{\int_{kh}^\tau r_s \rd s}^2} \nonumber\\
&\leq \xi h \int_{kh}^\tau \E{\lrn{r_s}^2} \rd s \nonumber\\
&\leq \xi h^2 \sup_{s\in[kh,(k+1)h]} \Ep{r_s\sim\p_s}{\lrn{r_s}^2} \nonumber\\
&= 2 L_G h^2 \sup_{s\in[kh,(k+1)h]} \Ep{r_s\sim\p_s}{\lrn{r_s}^2}.
\label{eq:differential_ineq_pre}
\end{align}


In the following Lemma~\ref{lemma:bounded_variance}, we uniformly upper bound $\E{\lrn{r_s}^2}$ by $\mathcal{O}\left( \dfrac{d}{\rho}\right)$.
\begin{lemma}
Assume that function $U$ satisfies Assumption~\ref{A1}--\ref{A3}, where $\rho$ denotes the minimum of the log-Sobolev constant and $1$.
If we take $\gamma=2$, $\xi=2L_G$, and
\[
h = \dfrac{1}{56} \dfrac{1}{\sqrt{L_G}}
\min\left\{ \dfrac{1}{24} \dfrac{\rho}{L_G}, \dfrac{ \sqrt{L_G} \rho }{L_H} \right\}
\cdot \min\left\{ \left(\widetilde{C_N} + 2\right)^{-1/2} \sqrt{\dfrac{\epsilon}{d}}, \sqrt{\dfrac{\epsilon}{C_M}} \right\},
\]
where $\epsilon \leq d \dfrac{L_G}{\rho}$.
Then for $r_s$ following Equation~\eqref{eq:underdamp_diff_disc}, $\forall s \geq 0$,
\[
\E{\lrn{x_s}^2}\leq \left(12\widetilde{C_N} + 13 \right) \dfrac{d}{\rho} + 12\dfrac{C_M}{\rho}
= \mathcal{O}\left(\dfrac{d}{\rho}\right).
\]
\label{lemma:bounded_variance}
\end{lemma}
We defer the proof of Lemma~\ref{lemma:bounded_variance} to Sec.~\ref{sec:bounded_variance}.

Taking Lemma~\ref{lemma:bounded_variance} as given, we can find that $\Ep{r_s\sim\p_s}{\lrn{r_s}^2}$ in Eq.~\eqref{eq:differential_ineq_pre} is upper bounded as:
\[
\sup_{s\in[kh,(k+1)h]} \Ep{r_s\sim\p_s}{\lrn{r_s}^2}
\leq \sup_{s\in[kh,(k+1)h]} \Ep{x_s\sim\p_s}{\lrn{x_s}^2}
\leq \left(12\widetilde{C_N} + 13 \right) \dfrac{d}{\rho} + 12\dfrac{C_M}{\rho},
\]
resulting in the final bound for $\Ep{\p(x_{kh}, x_\tau)}{\lrn{\theta_\tau-\theta_{kh}}^2}$ to be:
\[
\Ep{\p(x_{kh}, x_\tau)}{\lrn{\theta_\tau-\theta_{kh}}^2}
\leq \left( \left( 24 \widetilde{C_N} + 26 \right) \dfrac{L_G}{\rho} \cdot d + 24 C_M \dfrac{L_G}{\rho} \right) h^2
= \mathcal{O} \left( \dfrac{L_G}{\rho} d \cdot h^2 \right).
\]

\end{proof}

\begin{lemma}
Let $\p_0(x) = \p_0(\theta) \p_0(r)$, where
\[\displaystyle
\p_0(\theta) = \left(\dfrac{L_G}{2\pi}\right)^{d/2} \exp\left(-\dfrac{L_G}{2}\lrn{\theta}^2\right),
\]
and
\[\displaystyle
\p_0(r) = \left(\dfrac{\xi}{2\pi}\right)^{d/2} \exp\left(-\dfrac{\xi}{2}\lrn{r}^2\right).
\]
For $\p^*(x) \propto \left(-U(\theta) - \dfrac{\xi}{2}\lrn{r}^2\right)$, if $U(\theta)$ follows Assumptions~\ref{A1}--\ref{A3},
then we can define $\widetilde{C_N} = C_N + \dfrac{1}{2}\ln\dfrac{L_G}{2\pi}$ and obtain that
\begin{align}
    \KL{\p_0}{\p^*} = \int \p_0(x)\ln\left(\dfrac{\p_0(x)}{\p^*(x)}\right)\rd x
    \leq \widetilde{C_N}\cdot d + C_M,
    \label{eq:Lemma_initial_dist_a}
\end{align}
and
\begin{align}
    \mathcal{L}[\p_0]
    &= \KL{\p_0}{\p^*} + \Ep{\p_0}{\left< \nabla_x \ln\dfrac{\p_0}{\p^*}, S \nabla_x \ln\dfrac{\p_0}{\p^*} \right>} \nonumber\\
    &\leq \left( \widetilde{C_N} + 1 \right) d + C_M.
    \label{eq:Lemma_initial_dist_b}
\end{align}
With the setting of $\xi=2L_G$, we can also obtain that
\begin{align}
    \Ep{x\sim\p^*}{\lrn{x}^2}
    \leq \left(4\dfrac{\widetilde{C_N}}{\rho} + \dfrac{5}{2} \dfrac{1}{L_G}\right)\cdot d + 4\dfrac{C_M}{\rho}.
    \label{eq:Lemma_initial_dist_c}
\end{align}
\label{lemma:initial_dist}
\end{lemma}

\begin{proof}[Proof of Lemma~\ref{lemma:initial_dist}]
We want to bound $\KL{\p_0}{\p^*} = \displaystyle\int \p_0(x)\ln\left(\dfrac{\p_0(x)}{\p^*(x)}\right)\rd x = \displaystyle\int \p_0(\theta)\ln\left(\dfrac{\p_0(\theta)}{\p^*(\theta)}\right)\rd \theta$, where $\p^*(\theta)\propto e^{-U(\theta)}$ and $\p_0(\theta) = \left(\displaystyle\dfrac{L_G}{2\pi}\right)^{d/2} \exp\left(-\dfrac{L_G}{2}\lrn{\theta}^2\right)$.
First note that
\[
\p^*(\theta) = \exp\left(-U(\theta)\right) \bigg/ {\int \exp\left(-U(\theta)\right) \rd \theta}.
\]
By Assumptions~\ref{A2} and~\ref{A3}, $U(\theta)\leq\dfrac{L_G}{2}\|\theta\|^2$, $\forall \theta\in\mathbb{R}^d$.
We also know that: $\ln{\int \exp\left(-U(\theta)\right) \rd \theta} \leq C_N \cdot d + C_M$.

Therefore,
\begin{align}
-\ln p^*(\theta) &= U(\theta) + \ln{\int \exp\left(-U(\theta)\right) \rd \theta} \label{eq:int_change}\\
&\leq \dfrac{L_G}{2}\|\theta\|^2 + C_N \cdot d + C_M. \nonumber
\end{align}
Hence
\begin{align*}
- \int \p_0(\theta) \ln p^*(\theta) \rd \theta
\leq \dfrac{d}{2} + C_N \cdot d + C_M.
\end{align*}
We can also calculate that
\begin{align*}
\int \p_0(\theta) \ln p_0(\theta) \rd \theta
= - \dfrac{d}{2} -\dfrac{d}{2}\ln\dfrac{2\pi}{L_G}.
\end{align*}
Therefore,
\begin{align*}
\KL{\p_0}{\p^*} &= \int \p_0(\theta) \ln p_0(\theta) \rd \theta - \int \p_0(\theta) \ln p^*(\theta) \rd \theta \\
&\leq \left(C_N + \dfrac{1}{2}\ln\dfrac{L_G}{2\pi}\right)\cdot d + C_M \\
&= \widetilde{C_N}\cdot d + C_M.
\end{align*}

For $\Ep{\p_0}{\left< \nabla_x \ln\dfrac{\p_0}{\p^*}, S \nabla_x \ln\dfrac{\p_0}{\p^*} \right>}$, since $U$ is $L_G$-Lipschitz smooth, $\lrn{\nabla_\theta \ln p^*(x)}^2 \leq L_G^2 \lrn{\theta}^2$, and thus
\begin{align*}
\lefteqn{ \Ep{\p_0}{\left< \nabla_x \ln\dfrac{\p_0}{\p^*}, S \nabla_x \ln\dfrac{\p_0}{\p^*} \right>} } \\
&= \dfrac{1}{4L_G} \Ep{\p_0}{\lrn{\nabla_\theta \ln\dfrac{\p_0}{\p^*}}^2} \\
&\leq \dfrac{1}{2L_G} \Ep{\p_0}{\lrn{\nabla_\theta \ln\p_0}^2 + \lrn{\nabla_\theta \ln\p^*}^2} \\
&\leq L_G \Ep{\p_0}{\lrn{\theta}^2} \\
&= d.
\end{align*}
Consequently,
\begin{align}
    \mathcal{L}[\p_0]
    &= \KL{\p_0}{\p^*} + \Ep{\p_0}{\left< \nabla_x \ln\dfrac{\p_0}{\p^*}, S \nabla_x \ln\dfrac{\p_0}{\p^*} \right>} \nonumber\\
    &\leq \left( \widetilde{C_N} + 1 \right) d + C_M.
\end{align}

For $\Ep{x^*\sim\p^*}{\lrn{x^*}^2}$, we bound it using $W_2(\p^*,\p_0)$. We choose an auxiliary random variable $\theta_0$ following the law of $\p_0(\theta)$ and couples optimally with $\theta^*\sim\p^*(\theta)$: $(\theta^*,\theta_0) \sim \gamma\in \Gamma_{opt} (\p^*,\p_0)$. We then have
\begin{align*}
\Ep{x^*\sim\p^*}{\lrn{x^*}^2}
&= \Ep{r^*\sim\p^*(r)}{\lrn{r^*}^2} + \Ep{\theta^*\sim\p^*(\theta)}{\lrn{\theta^*}^2} \\
&= \dfrac{d}{\xi} + \Ep{(\theta^*,\theta_0)\sim\gamma}{\lrn{\theta_0 + (\theta^*-\theta_0)}^2} \\
&\leq \dfrac{d}{\xi} + 2\Ep{\theta_0\sim\p_0}{\lrn{\theta_0}^2} + 2\Ep{(\theta^*,\theta_0)\sim\gamma}{\lrn{\theta^*-\theta_0}^2} \\
&= \dfrac{d}{\xi} + \dfrac{2d}{L_G} + 2 W_2^2(\p^*,\p_0).
\end{align*}
We further expand this inequality by using the extended Talagrand inequality, Eq.~\eqref{eq:Talangrand_ineq}, which applies to the joint density function $\p^*(\theta,r)\propto \exp\left(-U(\theta)-\dfrac{\xi}{2}\lrn{r}^2\right)$ with log-Sobolev constant greater than or equal to $\rho$ and Lipschitz smoothness of $U+\dfrac{\xi}{2}\lrn{r}^2$ less than or equal to $4L_G$:
\begin{align}
W_2^2(\p_s,\p^*)\leq\dfrac{2}{\rho} \KL{\p_s}{\p^*}. \nonumber
\end{align}
Therefore, for $\xi = 2 L_G$,
\begin{align*}
\Ep{x^*\sim\p^*}{\lrn{x^*}^2}
&\leq \dfrac{d}{\xi} + \dfrac{2d}{L_G} + \dfrac{4}{\rho} \KL{\p_0}{\p^*} \\
&\leq \left(\dfrac{4\widetilde{C_N}}{\rho} + \dfrac{1}{\xi} + \dfrac{2}{L_G}\right)\cdot d + \dfrac{4C_M}{\rho} \\
&= \left(4\dfrac{\widetilde{C_N}}{\rho} + \dfrac{5}{2} \dfrac{1}{L_G}\right)\cdot d + 4\dfrac{C_M}{\rho}.
\end{align*}

It is worth noting that the choice of the initial condition $\p_0$ can be flexible.
For example, if we choose $x_0\sim\mathcal{N}\left(0,\mI\right)$, then $\KL{\p_0}{\p^*}\leq \left(C_N + \dfrac{L_G}{2} - \dfrac{1}{2} - \dfrac{1}{2}\ln(2\pi)\right)\cdot d + C_M$
(resulting in merely an extra $\ln L_G$ term in the overall computation complexity).
\end{proof}

\subsection{Supporting Proof for Lemma~\ref{lemma:bounded_variance_main} 
}
\label{sec:bounded_variance}
\begin{proof}[Proof of Lemma~\ref{lemma:bounded_variance}]
In what follows, we will prove that:
\begin{enumerate}
\item
$\E{\lrn{x_0}^2}\leq \left(12\widetilde{C_N} + 13 \right) \dfrac{d}{\rho} + 12\dfrac{C_M}{\rho}$. \label{pf:L6_goal_1}
\item
If $\forall s \leq kh$, $\E{\lrn{x_s}^2}\leq \left(12\widetilde{C_N} + 13 \right) \dfrac{d}{\rho} + 12\dfrac{C_M}{\rho}$, then $\forall s\in[kh, (k+1)h]$,
\[
\E{\lrn{x_s}^2}\leq \left(12\widetilde{C_N} + 13 \right) \dfrac{d}{\rho} + 12\dfrac{C_M}{\rho}.
\]
\label{pf:L6_goal_2}
\end{enumerate}
By induction, this will prove Lemma~\ref{lemma:bounded_variance}.

For claim~\ref{pf:L6_goal_1}, we can calculate that $\Ep{x_0\sim\p_0}{\lrn{x_0}^2} = \dfrac{3}{2} \cdot \dfrac{d}{L_G} \leq \left(12\widetilde{C_N} + 13 \right) \dfrac{d}{\rho} + 12\dfrac{C_M}{\rho}$.

We prove claim~\ref{pf:L6_goal_2} in a two step procedure: we first prove in the following Lemma~\ref{lemma:bounded_kinetic} that if $\E{\lrn{x_{kh}}^2}$ is bounded, then $\E{\lrn{x_s}^2}$ remains bounded for $s\in[kh, (k+1)h]$.
We then provide a specific bound of it.
\begin{lemma}
Assume the step size $h \leq \dfrac{1}{8 L_G}$ and let $\gamma=2$ and $\xi=2L_G$.
Then $\forall s\in[kh, (k+1)h]$, $\E{\lrn{x_s}^2}\leq 2\E{\lrn{x_{kh}}^2} + \dfrac{d}{L_G}$.
\label{lemma:bounded_kinetic}
\end{lemma}
It can be verified that for $\epsilon \leq 2 d$ and $\rho \leq 1$, $h$ is indeed smaller than $\dfrac{1}{8 L_G}$.
Thus Lemma~\ref{lemma:bounded_kinetic}, in conjunction with the induction hypothesis, gives us a rough bound that $\forall s\in[kh, (k+1)h]$,
\begin{align}
\E{\lrn{x_s}^2}
\leq 2\E{\lrn{x_{kh}}^2} + \dfrac{d}{L_G}
&\leq \left(24\widetilde{C_N} + 26 \right) \dfrac{d}{\rho} + 24\dfrac{C_M}{\rho} + \dfrac{d}{L_G} \nonumber\\
&\leq \left(24\widetilde{C_N} + 27 \right) \dfrac{d}{\rho} + 24\dfrac{C_M}{\rho}.
\label{eq:variance_bound_induction}
\end{align}

Then to accurately bound $\E{\lrn{x_s}^2}$, we use $\Ep{x^*\sim\p^*}{\lrn{x^*}^2}$ as an anchor point and bound the Wasserstein-$2$ distance between $p_s$ and $p^*$.
To this end, we choose an auxiliary random variable $x^*$ following the law of $\p^*$ and couples optimally with $\p(x_s)$: $(x_s,x^*)\sim\zeta\in\Gamma_{opt}(\p(x_s), \p^*(x^*))$.
Then using Young's inequality and Eq.~\eqref{eq:Lemma_initial_dist_c} in Lemma~\ref{lemma:initial_dist},
\begin{align*}
\E{\lrn{x_s}^2}
&= \Ep{(x_s,x^*)\sim\zeta}{\lrn{x^* + (x_s-x^*)}^2} \\
&\leq 2\Ep{\p^*}{\lrn{x^*}^2} + 2\Ep{(x_s,x^*)\sim\zeta}{\lrn{x_s-x^*}^2} \\
&\leq \left(8\dfrac{\widetilde{C_N}}{\rho} + 5\dfrac{1}{L_G}\right)\cdot d + 8\dfrac{C_M}{\rho} + 2W_2^2(\p_s,\p^*).
\end{align*}
Applying the extended Talagrand inequality, Eq.~\eqref{eq:Talangrand_ineq}, we obtain that
\begin{align}
\E{\lrn{x_s}^2}
\leq \left(8\dfrac{\widetilde{C_N}}{\rho} + 5\dfrac{1}{L_G}\right)\cdot d + 8\dfrac{C_M}{\rho} + \dfrac{4}{\rho} \KL{\p_s}{\p^*}. \label{eq:variance_bound_each_step}
\end{align}

On the other hand, we can use dissipation of the Lyapunov functional to bound the growth of the KL-divergence, and in turn the growth of $\E{\lrn{x_s}^2}$ in Eq.~\eqref{eq:variance_bound_each_step}.
This is the thesis of the following Lemma~\ref{lemma:aux_dL}.
\begin{lemma}
Let $x_s$ follow the underdamped Langevin algorithm~\ref{alg:main} with parameters $\xi=2L_G$, $\gamma=2$, and the step size $h = (k+1)h - kh$ given in Eq.~\eqref{eq:h_def}.
Also let $p_s$ be the probability distribution of $x_s$.
Assume that Eq.~\eqref{eq:variance_bound_induction} (given by the induction hypothesis in conjunction with Lemma~\ref{lemma:bounded_kinetic}) holds for any $s \in [kh, (k+1)h]$.
Then for $\epsilon \leq 2 d$ and $\rho \leq 1$, $\forall s \in [kh, (k+1)h]$,
\begin{align}
\dfrac{\rd \mathcal{L}[\p_s]}{\rd s}
\leq - \dfrac{\rho}{30} \cdot \left( \mathcal{L}[\p_s]
- \dfrac{\epsilon}{2} \right). \label{eq:Lyapunov_bound_each_step}
\end{align}
\label{lemma:aux_dL}
\end{lemma}
Applying Gr\"onwall's Lemma in Eq.~\eqref{eq:Lyapunov_bound_each_step}, we obtain that the objective functional $\mathcal{L}$ will not increase by more than $\epsilon/2$ throughout the progress of the algorithm:
\[
\mathcal{L}[\p_s] - \dfrac{\epsilon}{2}
\leq e^{-\frac{\rho}{30} (s-kh)} \left( \mathcal{L}[\p_{kh}] - \dfrac{\epsilon}{2} \right)
\leq e^{-\frac{\rho}{30} kh - \frac{\rho}{30} (s-kh)} \left( \mathcal{L}[\p_{0}] - \dfrac{\epsilon}{2} \right)
\leq \mathcal{L}[\p_{0}],
\]
where $\mathcal{L}[\p_s] = \KL{\p_s}{\p^*} +  \Ep{\p_s}{\left< \nabla_x \ln\dfrac{\p_s}{\p^*}, S \nabla_x \ln\dfrac{\p_s}{\p^*} \right>}$.
Therefore, we can bound $\KL{\p_s}{\p^*}$ using initial conditions
\begin{align*}
\KL{\p_s}{\p^*}
\leq \mathcal{L}[\p_s]
&\leq \mathcal{L}[\p_0] + \dfrac{\epsilon}{2}.
\end{align*}

From Lemma~\ref{lemma:initial_dist}, we know that $\mathcal{L}[\p_0] \leq \left( \widetilde{C_N} + 1 \right) d + C_M$.
%
Therefore, for $\epsilon\leq 2 d$,
\begin{align}
\KL{\p_s}{\p^*}
&\leq \left( \widetilde{C_N} + 1 \right) d + C_M + \dfrac{\epsilon}{2} \nonumber\\
&\leq \left( \widetilde{C_N} + 2 \right) d + C_M.
\label{eq:KL_bounded_each_step}
\end{align}
Plugging Eq.~\eqref{eq:KL_bounded_each_step} into Eq.~\eqref{eq:variance_bound_each_step}, we obtain our final result that
\begin{align*}
\E{\lrn{x_s}^2}
&\leq \left(8\dfrac{\widetilde{C_N}}{\rho} + 5\dfrac{1}{L_G}\right) d + 8\dfrac{C_M}{\rho} + \dfrac{4}{\rho} \KL{\p_s}{\p^*} \\
&\leq \left(12\dfrac{\widetilde{C_N}}{\rho} + 8\dfrac{1}{\rho} + 5\dfrac{1}{L_G}\right) d + 12\dfrac{C_M}{\rho}\\
&\leq \left(12\widetilde{C_N} + 13 \right) \dfrac{d}{\rho} + 12\dfrac{C_M}{\rho}, 
\end{align*}
since $\rho\leq L_G$.
\end{proof}


\begin{proof}[Proof of Lemma~\ref{lemma:bounded_kinetic}]
We begin from the discretized dynamics of underdamped Langevin diffusion Eq.~\eqref{eq:simple_irr_discrete} to calculate that $\forall s\in[kh,(k+1)h]$,
\begin{align}
\dfrac{\rd}{\rd s} \E{\lrn{x_s}^2}
&= \dfrac{\rd}{\rd s} \E{\lrn{\theta_s}^2 + \lrn{r_s}^2} \nonumber\\
&= 2\E{\left<
\left(\begin{array}{c}
\theta_s \\ r_s
\end{array}\right),
\left(\begin{array}{c}
\xi r_s \\ -\nabla U(\theta_{kh}) - \gamma\xi r_s - \gamma\nabla_r \ln \p_s
\end{array}\right)
\right>} \nonumber\\
&\leq 2\E{ \xi \lrw{\theta_s, r_s} - \lrw{r_s, \theta_{kh}} - \gamma\xi\lrn{r_s}^2 } - 2\gamma \int_{\mathbb{R}^d} \lrw{r_s, \nabla_r \ln \p_s} \p_s \rd x_s \nonumber\\
&\leq 2\E{ \xi\lrn{\theta_s}\lrn{r_s} + L_G\lrn{\theta_{kh}}\lrn{r_s} - \gamma\xi\lrn{r_s}^2 } + 2\gamma d \nonumber\\
&\leq 2L_G \E{\lrn{\theta_s}^2 + \lrn{r_s}^2} + 2L_G \E{\lrn{\theta_{kh}}^2 + \lrn{r_{kh}}^2} + 2\gamma d,
\label{eq:Lemma_11_tmp}
\end{align}
where the last step follows from plugging in the setting of  $\gamma=2$ and $\xi=2L_G$ and using Young's inequality.
Multiplying $e^{-2L_G s} > 0$ on both ends of Eq.~\eqref{eq:Lemma_11_tmp}, we obtain that $\forall s$,
\begin{align}
\dfrac{\rd}{\rd s} \left( e^{-2L_G s} \E{\lrn{x_s}^2} \right)
\leq e^{-2L_G s} \left( 2L_G \E{\lrn{x_{kh}}^2} + 2\gamma d \right).
\end{align}
Applying the fundamental theorem of calculus and multiplying $e^{2L_G \tau} > 0$ on both sides, we have that
\begin{align*}
\E{\lrn{x_\tau}^2}
&\leq e^{2L_G \tau} \int_{kh}^{\tau} e^{-2L_G s}\left(2L_G \E{ \lrn{x_{kh}}^2 } + 2\gamma d\right)\rd s + e^{2L_G (\tau-kh)} \E{\lrn{x_{kh}}^2} \\
&= \dfrac{1}{2L_G} \left( e^{2L_G (\tau-kh)} - 1 \right) \left(2L_G \E{ \lrn{x_{kh}}^2 } + 2\gamma d\right) + e^{2L_G (\tau-kh)} \E{\lrn{x_{kh}}^2}.
\end{align*}

It can then be checked that when $\tau - kh\leq h \leq \dfrac{1}{8 L_G}$, the factor $\left( e^{2L_G (\tau-kh)} - 1 \right) \leq \dfrac{1}{2}$, and that
\[
\E{\lrn{x_\tau}^2}
\leq 2\E{\lrn{x_{kh}}^2}
+ \dfrac{d}{L_G},
\quad \forall \tau\in[kh, (k+1)h].
\]

\end{proof}

\begin{proof}[Proof of Lemma~\ref{lemma:aux_dL}]
Applying the result of Eq.~\eqref{eq:differential_ineq_pre} that:
\[
\Ep{\p(x_{kh}, x_\tau)}{\lrn{\theta_\tau-\theta_{kh}}^2}
\leq 2 L_G h^2 \sup_{s\in[kh,(k+1)h]} \Ep{r_s\sim\p_s}{\lrn{r_s}^2}
\]
to Eq.~\eqref{eq:dL_overall_cont}--\eqref{eq:dL_overall_error_2}, we obtain that for $\xi=2L_G$, $\gamma=2$, and $\forall \tau\in[kh,(k+1)h]$,
\begin{align}
\dfrac{\rd \mathcal{L}(\p_\tau)}{\rd \tau}
&\leq - \dfrac{\rho}{30} \mathcal{L}(\p_\tau) \nonumber\\
&+ \left( 68 L_G^2 + \dfrac{1}{8} \dfrac{L_H^2}{L_G}\right) \Ep{\p(x_{kh}, x_\tau)}{\lrn{\theta_\tau-\theta_{kh}}^2}
+ 18 e L_G d \max\left\{ L_G^4 (\tau-kh)^4, L_G^2 (\tau-kh)^2 \right\} \nonumber\\
&\leq - \dfrac{\rho}{30}
\bigg( \mathcal{L}(\p_\tau)
- 60 \dfrac{L_G}{\rho} \left( 68 L_G^2 + \dfrac{1}{8} \dfrac{L_H^2}{L_G}\right) h^2 \sup_{s\in[kh,(k+1)h]} \Ep{r_s\sim\p_s}{\lrn{r_s}^2} \nonumber\\
&\qquad\qquad - 540 e \dfrac{L_G}{\rho} d \max\left\{ L_G^4 h^4, L_G^2 h^2 \right\} \bigg) \nonumber\\
&\leq - \dfrac{\rho}{30}
\bigg( \mathcal{L}(\p_\tau)
- 60 \dfrac{L_G}{\rho}
\max\left\{ 136 L_G^2, \dfrac{1}{4} \dfrac{L_H^2}{L_G}
\right\} h^2 \sup_{s\in[kh,(k+1)h]} \Ep{r_s\sim\p_s}{\lrn{r_s}^2} \nonumber\\
&\qquad\qquad - 1500 \dfrac{L_G}{\rho} d \max\left\{ L_G^4 h^4, L_G^2 h^2 \right\} \bigg).
\label{eq:differential_ineq}
\end{align}


Using the definition of $h = \dfrac{1}{56} \dfrac{1}{\sqrt{L_G}}
\min\left\{ \dfrac{1}{24} \dfrac{\rho}{L_G}, \dfrac{ \sqrt{L_G} \rho }{L_H} \right\}
\cdot \min\left\{ \left(\widetilde{C_N} + 2\right)^{-1/2} \sqrt{\dfrac{\epsilon}{d}}, \sqrt{\dfrac{\epsilon}{C_M}} \right\}$ in Eq.~\eqref{eq:h_def}, we know that
\begin{align*}
    L_G^2 h^2 \leq \dfrac{1}{6000} \dfrac{1}{\widetilde{C_N} + 2} \cdot \dfrac{\rho^2}{L_G} \dfrac{\epsilon}{d}.
\end{align*}
Plugging this setting into the last term of Eq.~\eqref{eq:differential_ineq}, we obtain that for $\epsilon \leq 2 d$ and $\rho \leq 1$,
\[
1500 \dfrac{L_G}{\rho} d \max\left\{ L_G^4 h^4, L_G^2 h^2 \right\}
\leq \dfrac{\epsilon}{4}.
\]
We can similarly combine this setting of the step size $h$ with the premise of this Lemma, Eq.~\eqref{eq:variance_bound_induction}, that
$\sup_{s\in[kh,(k+1)h]} \Ep{r_s\sim\p_s}{\lrn{r_s}^2} \leq \left(24\widetilde{C_N} + 27 \right) \dfrac{d}{\rho} + 24\dfrac{C_M}{\rho}$, and obtain:
\begin{align*}
\lefteqn{60 \dfrac{L_G}{\rho}
\max\left\{ 136 L_G^2, \dfrac{1}{4} \dfrac{L_H^2}{L_G}
\right\} h^2 \cdot
\sup_{s\in[kh,(k+1)h]} \Ep{r_s\sim\p_s}{\lrn{r_s}^2}} \\
&\leq 60 \dfrac{L_G}{\rho}
\max\left\{ 144 L_G^2, \dfrac{1}{4} \dfrac{L_H^2}{L_G}
\right\} \cdot
\left( \left(24\widetilde{C_N} + 27 \right) \dfrac{d}{\rho} + 24\dfrac{C_M}{\rho} \right) h^2 \\
&\leq 28^2 L_G
\max\left\{ 24^2 \dfrac{L_G^2}{\rho^2}, \dfrac{L_H^2}{L_G\rho^2}
\right\} \cdot
\max\left\{ \left( \widetilde{C_N} + 2 \right) d, C_M \right\} h^2
\leq \dfrac{\epsilon}{4}.
\end{align*}

Consequently, the time derivative of the Lyapunov functional $\mathcal{L}$ is bounded as:
\begin{align}
\dfrac{\rd \mathcal{L}[\p_s]}{\rd s}
&\leq - \rho \cdot \left( \mathcal{L}[\p_s]
- \dfrac{\epsilon}{2} \right).
\end{align}
\end{proof}

\section{Proofs for Auxiliary Facts}
\begin{proof}[Proof of Fact~\ref{fact:normalization}]
By Assumptions~\ref{B1} and~\ref{B3}, $U(\theta)\leq\dfrac{L_G}{2}\|\theta\|^2$, $\forall \theta\in\mathbb{R}^d$.
We also prove in the following that
\begin{itemize}
\item
$U(\theta)\geq\dfrac{m}{4}\|\theta\|^2$, $\forall \theta\in\mathbb{R}^d\setminus\ball\left(0,\dfrac{8L_G}{m}R\right)$;
\item
$U(\theta)\geq-\dfrac{L_G}{2}\|\theta\|^2$, $\forall \theta\in\ball\left(0,\dfrac{8L_G}{m}R\right)$.
\end{itemize}

The latter case follows directly from Assumptions~\ref{B1} and~\ref{B3}.
For the former case where $\|\theta\|\geq \dfrac{8L_G}{m}R$,
define $\vartheta=\dfrac{R}{\|\theta\|} \theta$.
Since $\|\vartheta\| = R$,
\[
\left<\nabla U(\vartheta), \vartheta\right> \geq -L_GR^2.
\]
Because any convex combination of $\theta$ and $\vartheta$ belongs to the set $\mathbb{R}^d\setminus\ball(0,R)$, where $U$ is $m$-strongly convex,
\begin{align*}
U(\theta) - U(\vartheta)
&\geq
\left<\nabla U(\vartheta), \theta-\vartheta\right> + \dfrac{m}{2}\|\theta-\vartheta\|^2 \\
&=
\left(\dfrac{\|\theta\|}{R}-1\right)\left<\nabla U(\vartheta), \vartheta\right> + \dfrac{m}{2} \left(\dfrac{\|\theta\|}{R}-1\right)^2 \\
&\geq - \left(\dfrac{\|\theta\|}{R}-1\right) L_GR^2 + \dfrac{m}{2} \left(\dfrac{\|\theta\|}{R}-1\right)^2 \\
&\geq \dfrac{m}{4} \|\theta\|^2 + L_GR^2,
\end{align*}
since $\|\theta\|\geq \dfrac{8L_G}{m}R$.
Again, using Assumptions~\ref{B1} and~\ref{B3}, $U(\vartheta)\geq-\dfrac{L_G}{2}R^2$, which leads to the result that
$U(\theta)\geq \dfrac{m}{4} \|\theta\|^2$.

Therefore, $U(\theta)\geq \dfrac{m}{4} \|\theta\|^2 - 32 \dfrac{L_G^2}{m^2} L_G R^2$ and
\begin{align}
\ln{\int \exp\left(-U(\theta)\right) \rd \theta} 
&\leq \ln \int \exp\left(-\dfrac{m}{4}\|\theta\|^2 + 32 \dfrac{L_G^2}{m^2} L_G R^2\right) \rd \theta \nonumber\\
&= \dfrac{d}{2}\ln\dfrac{4\pi}{m} + 32 \dfrac{L_G^2}{m^2} L_G R^2. \nonumber
\end{align}
Hence $C_N \leq \dfrac{1}{2} \ln\dfrac{4\pi}{m}$ and $C_M \leq 32 \dfrac{L_G^2}{m^2} L_G R^2$.
\end{proof}

\begin{proof}[Proof of Fact~\ref{fact:disc_bound}]
We begin with the definition of
\[
\eta = \dfrac{1}{\gamma}\left(
\dfrac{e^{\gamma\xi\step}\left(1-e^{-\gamma\xi\step}\right)^2}{\gamma\xi}
-\left(\step-\dfrac{1-e^{-\gamma\xi\step}}{\gamma\xi}\right)
\right),
\]
and provide bound for it when $0 \leq \step \leq \min\left\{\dfrac{1}{\gamma\xi}, \dfrac{1}{\sqrt{2e L_G \xi}}\right\}$.

First note that for $0\leq\step\leq\dfrac{1}{\gamma\xi}$,
\[
1-\gamma\xi\nu \leq e^{-\gamma\xi\step} \leq 1+\gamma\xi\step.
\]
Then we obtain that
\begin{align*}
\eta &= \dfrac{1}{\gamma}\left(
\dfrac{e^{\gamma\xi\step}}{\gamma\xi}\left(\gamma\xi\step\right)^2
-\step+\dfrac{-\gamma\xi\step}{\gamma\xi}\right) \\
&= \gamma\step^2 e^{\gamma\xi\step} \\
&\leq e\xi\step^2.
\end{align*}

We then prove Fact~\ref{fact:disc_bound} by separating the following term:
\begin{align*}
\lefteqn{ \lrn{\left(\begin{array}{c}
\nabla^2 U(\theta_{kh}) \left(\left(\mI+\eta\nabla^2 U(\theta_{kh})\right)^{-1}-\mI\right) \\
- \dfrac{e^{\gamma\xi\step}-1}{\gamma} \nabla^2 U(\theta_{kh}) \left(\mI+\eta\nabla^2 U(\theta_{kh})\right)^{-1}
\end{array}\right)}_2 } \\
&\leq 2\max \bigg\{ \lrn{\nabla^2 U(\theta_{kh}) \left(\left(\mI+\eta\nabla^2 U(\theta_{kh})\right)^{-1}-\mI\right)}_2, \\
&\quad\qquad\qquad \lrn{\dfrac{1-e^{\gamma\xi\step}}{\gamma} \nabla^2 U(\theta_{kh}) \left(\mI+\eta\nabla^2 U(\theta_{kh})\right)^{-1}}_2 \bigg\}.
\end{align*}

Since $\lrn{\eta\nabla^2 U(\theta_{kh})} \leq e\xi\nu^2\lrn{\nabla^2 U(\theta_{kh})} \leq e L_G \xi \nu^2 < 1$ for $\nu \leq \min\left\{\dfrac{1}{\gamma\xi}, \dfrac{1}{\sqrt{2e L_G \xi}}\right\}$,
$\left(\mI+\eta\nabla^2 U(\theta_{kh})\right)^{-1}$ admits the following series expansion:
\[
\left(\mI+\eta\nabla^2 U(\theta_{kh})\right)^{-1}
= \sum_{n=0}^\infty \left(-\eta \nabla^2 U(\theta_{kh})\right)^{n}.
\]
Consequently,
\[
\lrn{\left(\mI+\eta\nabla^2 U(\theta_{kh})\right)^{-1}}_2
\leq \sum_{n=0}^\infty \left(\eta L_G\right)^n
= \dfrac{1}{1 - \eta L_G}
\leq 2,
\]
and
\[
\lrn{\left(\mI+\eta\nabla^2 U(\theta_{kh})\right)^{-1} - \mI}_2
\leq \sum_{n=1}^\infty \left(\eta L_G\right)^n
= \dfrac{\eta L_G}{1 - \eta L_G}
\leq 2\eta L_G
= 2e L_G \xi \nu^2.
\]

Therefore, for the first term,
\begin{align*}
\lefteqn{ \lrn{\nabla^2 U(\theta_{kh}) \left(\left(\mI+\eta\nabla^2 U(\theta_{kh})\right)^{-1}-\mI\right)}_2 } \\
&\leq \lrn{\nabla^2 U(\theta_{kh})}_2 \lrn{\left(\mI+\eta\nabla^2 U(\theta_{kh})\right)^{-1}-\mI}_2 \\
&\leq 2e L_G^2 \xi \nu^2.
\end{align*}
For the second term,
\begin{align*}
\lefteqn{ \lrn{\dfrac{1-e^{\gamma\xi\step}}{\gamma} \nabla^2 U(\theta_{kh}) \left(\mI+\eta\nabla^2 U(\theta_{kh})\right)^{-1}}_2 } \\
&\leq \xi \nu \lrn{\nabla^2 U(\theta_{kh})} \lrn{\left(\mI+\eta\nabla^2 U(\theta_{kh})\right)^{-1}}_2 \\
&\leq 2L_G \xi\nu.
\end{align*}
Therefore,
\[
\lrn{\left(\begin{array}{c}
\nabla^2 U(\theta_{kh}) \left(\left(\mI+\eta\nabla^2 U(\theta_{kh})\right)^{-1}-\mI\right) \\
- \dfrac{e^{\gamma\xi\step}-1}{\gamma} \nabla^2 U(\theta_{kh}) \left(\mI+\eta\nabla^2 U(\theta_{kh})\right)^{-1}
\end{array}\right)}_2
\leq 4e \max\{L_G^2 \xi \nu^2, L_G \xi\nu\}.
\]
\end{proof}

\end{document}